\documentclass{article} %
\usepackage{fullpage}
\usepackage{authblk}
\usepackage[sort&compress,numbers]{natbib}

\usepackage{color}
\usepackage{xcolor}
\usepackage{hyperref}
\hypersetup{
    colorlinks,
    linkcolor={red!50!black},
    citecolor={blue!50!black},
    urlcolor={blue!80!black}
}
\usepackage{enumitem}
\usepackage{mathrsfs}
\usepackage{mathtools}
\usepackage{dsfont}
\usepackage[font=small,labelfont=bf]{caption}
\usepackage{subcaption}
\usepackage{graphicx}

\usepackage{amsmath,amsthm,amssymb,colonequals,etoolbox}
\usepackage{thmtools}
\usepackage{url}
\usepackage{cleveref}
\usepackage{microtype}
\usepackage{thm-restate}
\usepackage{array}
\usepackage{algorithm}
\usepackage{algpseudocode}
\algnewcommand{\algorithmicinput}{\textbf{Input:}}
\algnewcommand{\algorithmicoutput}{\textbf{Output:}}
\algnewcommand{\Input}[1]{\Statex \algorithmicinput\ #1}
\algnewcommand{\Output}[1]{\Statex \algorithmicoutput\ #1}
\usepackage[absolute,overlay]{textpos}
\usepackage{wrapfig}

\newcommand{\MyParagraph}[1]{\paragraph{#1}}

\allowdisplaybreaks
\renewcommand{\geq}{\geqslant}

\renewcommand{\leq}{\leqslant}

\newtheorem{thm}{Theorem}[section]
\newtheorem{mydef}[thm]{Definition}
\newtheorem{myprop}[thm]{Proposition}
\newtheorem{mylemma}[thm]{Lemma}

\newtheorem{mythm}[thm]{Theorem}

\DeclareMathOperator*{\rank}{rank}
\DeclareMathOperator*{\diag}{diag}
\DeclareMathOperator*{\Tr}{\mathrm{tr}}

\newcommand{\calU}{\mathcal{U}}

\newcommand{\R}{\ensuremath{\mathbb{R}}}

\newcommand{\N}{\ensuremath{\mathbb{N}}}

\newcommand{\norm}[1]{\lVert #1 \rVert}
\newcommand{\bignorm}[1]{\left\lVert #1 \right\rVert}

\newcommand{\ip}[2]{\ensuremath{\langle #1, #2 \rangle}}

\newcommand{\E}{\mathbb{E}}
\newcommand{\abs}[1]{\ensuremath{| #1 |}}
\newcommand{\bigabs}[1]{\ensuremath{\left| #1 \right|}}
\newcommand{\floor}[1]{\lfloor #1 \rfloor}

\newcommand{\ind}{\mathds{1}}

\newcommand{\T}{\mathsf{T}}

\newcommand{\calD}{\mathcal{D}}

\newcommand{\calI}{\mathcal{I}}

\newcommand{\calF}{\mathcal{F}}

\newcommand{\cvectwo}[2]{\begin{bmatrix} #1 \\ #2 \end{bmatrix}}

\numberwithin{equation}{section}

\newcommand{\opnorm}[1]{\norm{#1}_{\mathrm{op}}}

\newcommand{\rmd}{\mathrm{d}}

\newcommand{\sfN}{\mathsf{N}}

\DeclarePairedDelimiterX{\infdivx}[2]{(}{)}{%
  #1\;\delimsize\|\;#2%
}

\newcommand{\SBSDE}{\mathrm{S}\textrm{-}\mathrm{BSDE}}
\newcommand{\FPINNs}{\mathrm{FS\textrm{-}PINNs}}
\newcommand{\SFPINNs}{\mathrm{S\textrm{-}FS\textrm{-}PINNs}}

\newcommand{\sbullet}{\scalebox{0.55}{$\bullet$}}

\newcommand{\Xs}{X^{\sbullet}}
\newcommand{\Ys}{Y^{\sbullet,\theta}}
\newcommand{\Zs}{Z^{\sbullet,\theta}}
\newcommand{\hs}{h^{\sbullet}}

\newcommand{\Xhats}{\hat{X}^{\sbullet}}
\newcommand{\Yhats}{\hat{Y}^{\sbullet,\theta}}
\newcommand{\Zhats}{\hat{Z}^{\sbullet,\theta}}

\newcommand{\Zbars}{\bar{Z}^{\sbullet,\theta}}

\newcommand{\mkmat}[1]{\ensuremath{\begin{bmatrix}#1\end{bmatrix}}}

\title{Integration Matters for Learning PDEs with Backward SDEs}

\author[1]{Sungje Park}
\author[1]{Stephen Tu}
\affil[1]{Department of Electrical and Computer Engineering, University of Southern California}

\date{May 5, 2025, Revised: \today}

\begin{document}

\maketitle

\begin{abstract}

Backward stochastic differential equation (BSDE)-based deep learning methods provide an alternative to Physics-Informed Neural Networks (PINNs) for solving high-dimensional partial differential equations (PDEs), offering potential algorithmic advantages in settings such as stochastic optimal control, where the PDEs of interest are tied to an underlying dynamical system.
However, standard BSDE-based solvers have empirically been shown to underperform relative to PINNs in the literature.
In this paper, we identify the root cause of this performance gap as a discretization bias introduced by the standard Euler-Maruyama (EM) integration scheme applied to one-step self-consistency BSDE losses, which shifts the optimization landscape off target. We find that this bias cannot be satisfactorily addressed through finer step-sizes or multi-step self-consistency losses. 
To properly handle this issue, we propose a Stratonovich-based BSDE formulation, which we implement with stochastic Heun integration.
We show that our proposed approach completely eliminates the bias issues faced by EM integration.
Furthermore, our empirical results show that our Heun-based BSDE method consistently outperforms EM-based variants and achieves competitive results with PINNs across multiple high-dimensional benchmarks. Our findings highlight the critical role of integration schemes in BSDE-based PDE solvers, an algorithmic detail that has received little attention thus far in the literature.

\end{abstract}

\section{Introduction}
\label{sec:intro}

Numerical solutions to 
partial differential equations (PDEs) are foundational to modeling problems across a diverse set of fields in science and engineering.
However, due to the {curse of dimensionality}
of traditional numerical methods,
application of classic solvers to high dimensional PDEs is computationally intractable.
In recent years, motivated by the success of deep learning methods, both Physics-Informed Neural Networks (PINNs)~\cite{raissi2017physicsI,raissi2017physicsII} and backward stochastic differential equation (BSDE)  methods~\cite{han2017deep,raissi2024forward,nusken2023interpolating} have emerged as %
promising alternatives to classic techniques.

Despite the widespread popularity of PINNs methods, in this paper we focus on the use of BSDE-based methods for
solving high-dimensional PDEs.
The key difference between PINNs and BSDE methods is that while PINNs 
minimize the PDE residual directly on randomly sampled collocation points, 
BSDE methods reformulate PDEs as forward-backward SDEs (FBSDEs) and simulate the resulting stochastic processes to minimize the discrepancy between predicted and terminal conditions at the end of the forward SDE trajectory~\cite{han2017deep},
or across an intermediate time-horizon via \emph{self-consistency} losses~\cite{raissi2024forward,nusken2023interpolating}.
BSDE methods are especially well-suited for high-dimensional problems where there is underlying dynamics---such as in stochastic optimal control or quantitative finance---as the crux of these methods involving sampling over stochastic trajectories rather than over bounded spatial domains.
Furthermore, BSDE methods offers a significant advantage in problems where the governing equations of the PDE are unknown and can only be accessed through simulation~\cite{wang2022modelfree},
as in model-free optimal control (cf.~\Cref{sec:appendix:model_free} for more details). 
In contrast, PINNs methods require explicit knowledge of the PDE equations, which may be either be impractical to obtain for various tasks or require a separate model learning step within the training pipeline.

Surprisingly, despite the aforementioned benefits of BSDE methods compared with PINNs, 
a thorough comparison between the two techniques remains largely absent from the literature. One notable exception is recent work by~\cite{nusken2023interpolating} which finds
that on several benchmark problems, PDE solutions found by BSDE-based approaches significantly underperform the corresponding PINNs solutions.
To address this gap, they propose a hybrid \emph{interpolating} loss between the PINNs and BSDE losses.
While promising, their result has two key disadvantages. First, the underlying cause of the performance gap between BSDE and PINNs methods is not elucidated.
Second, their method introduces a new hyper-parameter (the horizon-length controlling the level of interpolation) 
which must be tuned for optimal performance,
adding complexity to the already delicate training process~\cite{wang2023expert}.

In this work, we identify the key source of the performance gap between BSDE and PINNs methods
as the standard Euler-Maruyama (EM) scheme used in BSDE methods for stochastic integration. 
Although simple to implement, we show that the EM scheme introduces a significant discretization bias in \emph{one-step} BSDE losses, resulting in a discrepancy between the optimization objective and the true solution. 
We furthermore show that the EM discretization bias can only be made arbitrarily small by using
\emph{multi-step} BSDE losses, 
which we show both theoretically and empirically comes at a significant cost in performance.
Our analysis reveals the interpolating loss of \cite{nusken2023interpolating} as a method to find (via hyper-parameter tuning) the best suitable horizon length
(i.e., number of self-consistency steps).

As an alternative, we propose interpreting both the forward and backward SDEs as Stratonovich SDEs---as opposed to It{\^{o}} SDEs---and utilizing the stochastic Heun integration scheme for numerical integration.
We prove that the use of the stochastic Heun method completely eliminates the non-vanishing bias issues which occur in the EM formulation 
for one-step BSDE losses. 
This removes all performance tradeoffs in the horizon-length, allowing us to utilize single-step self-consistency losses. The result is a practical BSDE-based algorithm that is competitive with PINNs methods without the need for interpolating losses.
Surprisingly, prior to our work
the role of stochastic integration has received little attention in the BSDE literature; we hope that our results inspire further algorithmic and implementation level improvements for BSDE solvers.

\section{Related Work}
\label{sec:related_work}

In recent years, PINNs~\cite{raissi2017physicsI,raissi2017physicsII,raissi2019pinns,yu2018deep,sirignano2018DGM} has emerged as a popular method for solving high dimensional PDEs.
PINNs methods parameterize the PDE solution as a neural network and directly minimize the PDE residual as a loss function, provides a mesh-free method
that can easily incorporate complex boundary conditions and empirical data.
However, the PINNs approach remains an incomplete solution and still 
suffers from notable issues including
optimization challenges~\cite{krishnapriyan2021pinnfailures,rathore2024challenges,wang2022pinnsNTK,wang2024pinnsdifficulty,chuang2023predictive},
despite a concerted effort to remedy these difficulties~\cite{krishnapriyan2021pinnfailures,wang2023expert,wu2023adaptive,wu2024ropinn,rathore2024challenges,wang2022l2loss,gao2023failureinformed,mojgani2022lagrangian}.
Hence, the application of PINNs as a general purpose solver for complex high-dimensional PDEs remains an active area of research.

On the other hand, a complementary line of work proposes methods based on BSDEs to solve high-dimensional PDEs~\cite{han2017deep,han2018bsde,hure2020deep,beck2019machine,raissi2024forward,nusken2023interpolating,takanashi2022controlvariate,andersson2023deepfbsde}.
These approaches reformulate PDEs as forward-backward SDEs to derive a trajectory-based loss.
While the original deep BSDE methods~\cite{han2017deep,han2018bsde,hure2020deep}
learn separate neural networks to predict both the value and gradient at each discrete time-step, follow up work~\cite{raissi2024forward,nusken2023interpolating} uses \emph{self-consistency}, i.e., the residual of stochastic integration along BSDE trajectories, to form a loss. In this work, we exclusively focus on self-consistency BSDE losses, as they generalize the original method while allowing for a single network to parameterize the PDE solution for all space and time, similarly to PINNs.
Similar self-consistency losses have also been recently utilized to learn solutions to
Fokker-Planck PDEs~\cite{shen2022selfconsistency,boffi2023probability,li2023probabilityflow}.
Further discussions on related works and the paper's relationship to recent BSDE-based methods can be found in~\Cref{sec:appendix:related_works}.

The main purpose of our work is to understand the performance differences between PINNs and BSDE methods on high-dimensional PDEs. The most relevant work to ours is \cite{nusken2023interpolating}, which to the best of our knowledge is the only work in the literature that directly compares PINNs and BSDE methods in a head-to-head evaluation. As discussed in \Cref{sec:intro}, one of our main contributions shows that the gap in performance between PINNs and BSDE methods observed in \cite{nusken2023interpolating} is due to the choice of stochastic integration. Previous work~\cite{chassagneux2022deep} studying 
stochastic Runge-Kutta discretizations for BSDEs methods
considers only the original BSDE 
losses instead of self-consistency methods, and hence does not
uncover the issues identified in our work.

\section{Background and Problem Setup}
\label{sec:background}

We consider learning approximate solutions to the following
non-linear boundary value PDEs
\begin{align}
    R[u](x, t) &:= \partial_t u(x, t) + \frac{1}{2} \Tr(H(x, t) \cdot \nabla^2 u(x, t)) + \ip{f(x, t)}{\nabla u(x, t)} - h[u](x, t) = 0, \label{eq:main_PDE} 
\end{align}
over domain $x \in \Omega \subseteq \R^d, \,\, t \in \calI := [0, T]$
with boundary conditions
(a) $u(x, T) = \phi(x) \,\,\forall x \in \Omega$
and (b)
$u(x, t) = \phi_{b}(x, t) \,\,\forall x \in \partial \Omega, \, t \in \calI$.
Here, $u : \Omega \times \calI \mapsto \R$
is a candidate PDE solution,
$f : \Omega \times \calI \mapsto \R^d$ is a vector-field, 
$h[u] := h_0(x, t, u(x, t), \nabla u(x, t))$ 
for some
$h_0 : \Omega \times \calI \times \R \times \R^d \mapsto \R$ captures the non-linear terms,
$H(x, t) = g(x, t) g(x, t)^\T \in \R^{d \times d}$ for $g(x, t) \in \R^{d \times d}$
is a positive definite matrix-valued function,
$\phi : \Omega \times \calI$ and
$\phi_b : \partial \Omega \times \calI$
are boundary conditions, 
and both $\nabla$ and $\nabla^2$ denote spatial gradients and Hessians, respectively.
For expositional clarity, 
we will assume that $\Omega=\R^d$ and drop the second boundary condition (b), noting that all subsequent arguments can be extended in a straightforward manner for bounded domains $\Omega$.

\MyParagraph{Physics-Informed Neural Networks (PINNs).} Under the assumption of knowledge of the operator $R[u]$ and the boundary condition $\phi$, 
the standard PINNs methodology~\cite{raissi2017physicsI,raissi2017physicsII,raissi2019pinns,yu2018deep,sirignano2018DGM} for solving \eqref{eq:main_PDE} works by parameterizing the solution $u(x, t)$ in a function class $\calU := \{ u_\theta(x, t) \mid \theta \in \Theta \}$ (e.g., $\theta$ represents the weights of a neural network), and minimizing the PINNs loss over
$\calU$:\footnote{We leave consideration of the PINNs loss with non-square losses (e.g., \cite{wang2022l2loss}) to future work.}
\begin{align}
    L_{\mathrm{PINNs}}(\theta; \lambda) := \E_{(x, t) \sim \mu}[ (R[u_\theta](x, t))^2 ] + \lambda \cdot \E_{x \sim \mu'}[ (u_\theta(x, T) - \phi(x))^2 ], \label{eq:PINNs}
\end{align}
where $\mu$ is a measure over $\Omega \times \calI$ and
$\mu'$ is a measure over $\Omega$. The choice of measures $\mu, \mu'$, in addition to the relative weight $\lambda$ are hyper-parameters which must be carefully selected by the user.
To simplify exposition further, we will assume that each $u_\theta(x, t) \in \calU$ satisfies
$u(\cdot, T) = \phi$ (e.g., as in \cite{lagaris1998boundaryconditions,singh2024exact}), 
and hence the PINNs loss simplifies further to
$L_{\mathrm{PINNs}}(\theta) = \E_{(x, t) \sim \mu}[ (R[u_\theta](x, t))^2 ]$.

\MyParagraph{Backward SDEs and self-consistency losses.} While the PINNs loss has received much attention in the literature, a separate line of work has advocated for an
alternative approach to solving PDEs based on backward SDEs~\cite{han2017deep,han2018bsde,hure2020deep,beck2019machine,raissi2024forward}. 
The key idea is that 
given the 
\emph{forward (It{\^{o}}) SDE}:
\begin{align}
    \rmd X_t = f(X_t, t) \rmd t + g(X_t, t) \rmd B_t, \quad X_0 = x_0, \label{eq:forward_SDE}
\end{align}
where $(B_t)_{t \geq 0}$ is standard Brownian motion in $\R^d$, the corresponding \emph{backward (It{\^{o}}) SDE}:
\begin{align}
    \rmd Y_t = h(X_t, t, Y_t, Z_t) \rmd t + Z_t^\T g(X_t, t) \rmd B_t, \quad Y_T = \phi(X_T), \label{eq:backward_SDE}
\end{align}
is solved by setting $Y_t = u(X_t, t)$ and $Z_t = \nabla u(X_t, t)$, where $u$ is a solution to the PDE \eqref{eq:main_PDE};
this equivalence is readily shown with It{\^{o}}'s lemma.
The relationship between the forward and backward SDE has motivated several different types of BSDE loss functions for solving \eqref{eq:main_PDE}.
In this work, we focus on BSDE losses based on 
\emph{self-consistency}~\cite{raissi2024forward,nusken2023interpolating}, which uses the residual of stochastic integration along the BSDE trajectories as supervision.
Self-consistency losses are more practical than other BSDE variants as only one network is required and the weights can be shared across time (unlike e.g., the original BSDE losses~\cite{han2017deep,han2018bsde} which
learn $N$ models to predict both $Y_t$ and $Z_t$ at $N$ discretization points, and require retraining for every new initial condition $x_0$).
Specifically, we consider the following $H$-horizon (for $N = T/H \in \N_+$ and $t_n = nH$) self-consistency BSDE loss:\footnote{
We note that 
the most general form of self-consistency losses are due to \cite{nusken2023interpolating} and take on the form
$L_{\mathrm{BSDE}}(\theta;\rho) = \E_{\substack{x_0 \sim \mu_0,\\ (t_s, t_f) \sim \rho, B_t}} \frac{1}{\Delta_t^2} \left(  u_\theta(X_{t_f}, t_f) - u_\theta(X_{t_s}, t_s) - S_\theta(t_s, t_f) \right)^2$
involving a time-pair distribution $\rho$ over $\calI^2$, where $\Delta_t := t_f - t_s$. We choose to present \eqref{eq:BSDE} as it more closely aligns
with the discrete losses.
}
\begin{align}
    L_{\mathrm{BSDE},H}(\theta) := \E_{x_0, B_t} \frac{1}{NH^2} \sum_{n=0}^{N-1} \left( u_\theta(X_{t_{n+1}}, t_{n+1}) - u_\theta(X_{t_n}, t_n) - S_\theta(t_n, t_{n+1}) \right)^2, \label{eq:BSDE}
\end{align}
where $S_\theta(t_0, t_1) := \int_{t_0}^{t_1} h_\theta(X_t, t) \rmd t - \int_{t_0}^{t_1} \nabla u_\theta(X_t, t)^\T g(X_t, t) \rmd B_t$ 
with $h_\theta(x, t) := h[u_\theta](x, t)$, and 
$x_0\sim\mu_0$ is drawn from a distribution $\mu_0$ over initial conditions
for the forward SDE \eqref{eq:forward_SDE}.

\MyParagraph{Euler-Maruyama integration.}
Unlike the PINNs loss \eqref{eq:PINNs}, the BSDE loss \eqref{eq:BSDE} 
must be discretized with an appropriate stochastic integrator. 
The standard choice is to use the Euler-Maruyama (EM) method, selecting 
$N \in \N_+$ to define a step-size $\tau := T/N$ and time-points $t_n := n \tau$,
and integrating the forward and backward SDEs as follows:
\begin{align}
    \hat{X}_{n+1} &= \hat{X}_n + \tau f(\hat{X}_n, t_n) + \sqrt{\tau} g(\hat{X}_n, t_n) w_n, \quad w_n \sim \sfN(0, I_d), \quad \hat{X}_0 = x_0, \nonumber \\
    \hat{Y}^\theta_{n+1} &= \hat{Y}^\theta_n + \tau h_\theta(\hat{X}_n, t_n) + \sqrt{\tau} \nabla u_\theta(\hat{X}_n, t_n)^\T g(\hat{X}_n, t_n) w_n, \quad \hat{Y}^\theta_0 = u_\theta(x_0, 0). \label{eq:backwards_SDE_EM}
\end{align}
With this discretization, the $k$-step EM-BSDE loss (for $N/k \in \N_+$) for
$L_{\mathrm{BSDE}}(\theta)$ is:
\begin{align}
        L_{\mathrm{EM}_k,\tau}(\theta) &:= \mathop{\E}_{x_0, w_n} \frac{k}{N\tau^2} \sum_{n=0}^{\frac{N}{k}-1} \left( u_\theta(\hat{X}_{(n+1)k}, t_{(n+1)k}) - u_\theta(\hat{X}_{nk}, t_{nk}) - (\hat{Y}^\theta_{(n+1)k} - \hat{Y}^\theta_{nk}) \right)^2. \label{eq:BSDE_EM} 
\end{align}
In the \emph{one-step} (or \emph{single-step}) setting $k=1$, \eqref{eq:BSDE_EM}
reduces to the self-consistent BSDE loss from \cite{raissi2024forward},
also discussed in~\cite{chan2019machine,kapllani2024bsde};
we use the shorthand $L_{\mathrm{EM},\tau}(\theta)$ to denote this setting.
We refer to the $k > 1$ case generally as \emph{multi-step},
which is a form of interpolating loss from~\cite{nusken2023interpolating}.
Another notable case is when $k=N$, which recovers the \emph{full-horizon}
losses used in the original BSDE works~\cite{han2017deep,han2018bsde}.

\section{Analysis of One-Step Self-Consistency Losses}
\label{sec:one_step_BSDE}

In this section we conduct an analysis of both EM and Heun stochastic integration applied to the one-step self-consistency BSDE loss.

\MyParagraph{The H{\"{o}}lder space $C^{k,1}$.}
Let $f : M \mapsto M'$, where both $M, M'$ are subsets of Euclidean space (with possibly different dimension). We say that $f$ is $C^{k,1}(M, M')$ ($C^{k,1}$ when $M,M'$ are clear) if $f$ is both bounded and $k$-times continuously differentiable on $M$, and 
all $j$-th derivatives of $f$ for $j \in \{0, \dots, k\}$ are Lipschitz continuous.
The H{\"{o}}lder norm $\norm{f}_{C^{k,1}(M, M')}$ is the smallest bound possible on $\norm{f}$ and all the Lipschitz constants for $D^j f$, $j \in \{0, \dots, k\}$.

\subsection{Analysis of Euler-Maruyama for BSDE}
\label{sec:results:euler_BSDE}

We first illustrate the bias when using EM to integrate the single-step consistency loss.
To do this, we define the point-wise EM loss at resolution $\tau$
for a fixed $(x, t) \in \R^d \times \calI$ as:
\begin{align}
    \ell_{\mathrm{EM},\tau}(\theta, x, t) &:= \E_{w}\left( u_\theta(\hat{x}_{t+\tau}, t+\tau) - u_\theta(x, t) - \tau h_\theta(x, t) - \sqrt{\tau}\ip{\nabla u_\theta(x, t)}{g(x, t) w} \right)^2, \label{eq:one_step_EM_BSDE} \\
    \hat{x}_{t+\tau} &:= x + \tau f(x, t) + \sqrt{\tau} g(x, t) w, \quad w \sim \sfN(0, I_d). \nonumber
\end{align}
The point-wise EM loss \eqref{eq:one_step_EM_BSDE} is related to the one-step EM-BSDE loss via
$L_{\mathrm{EM},\tau}(\theta) = \frac{1}{N \tau^2} \sum_{n=0}^{N-1} \E_{\hat{X}_n}[ \ell_{\mathrm{EM,\tau}}(\theta, \hat{X}_n, t_n) ]$.
Our first result shows that the dominant error term (in $\tau$) of the loss \eqref{eq:one_step_EM_BSDE} suffers from an additive bias term that is introduced as a result of the EM integration.

\begin{restatable}{mylemma}{EMthm}
\label{lemma:euler_maruyama_bsde}
Suppose that $f, g$ are bounded and $u_\theta$ is $C^{2,1}$.
We have that
\begin{align}
     \tau^{-2} \cdot  \ell_{\mathrm{EM},\tau}(\theta, x, t) = (R[u_\theta](x, t))^2 + \frac{1}{2} \Tr\left[ (H(x, t) \cdot \nabla^2 u_\theta (x, t))^2 \right] + O(\tau^{1/2}),
\end{align}
where the $O(\cdot)$ hides factors depending on $d$, the bounds on $f, g$, and
$\norm{ u_\theta }_{C^{2,1}}$.
\end{restatable}
\Cref{lemma:euler_maruyama_bsde} can further be used, in conjunction
with standard results on the order $1/2$ strong convergence of 
EM integration~\cite{kloeden1992numericalsolutions},
to show the following statement regarding the full loss $L_{\mathrm{EM},\tau}(\theta)$.
\begin{restatable}{mythm}{EMfinalform}
\label{thm:EM_final_form}
Suppose that $f,g,h_\theta \in C^{0, 1}$, $u_\theta \in C^{2,1}$, and $\tau \leq 1$.
We have that:
\begin{align}
     L_{\mathrm{EM},\tau}(\theta) = \frac{1}{T} \int_0^T \E\left[ (R[u_\theta](X_t, t))^2 + \frac{1}{2} \Tr\left[ (H(X_t, t) \cdot \nabla^2 u_\theta (X_t, t))^2 \right] \right] \rmd t + O(\tau^{1/2}), \label{eq:EM_tau_one_step_limit}
\end{align}
where the $O(\cdot)$ hides constants that depend on $d$, $T$, and the H{\"{o}}lder norms
of $f$, $g$, $h_\theta$, and $u_\theta$.
\end{restatable}
\Cref{thm:EM_final_form}
implies that even if the function class $\calU$ is expressive enough to contain a PDE solution $u_{\theta_\star}$ to \eqref{eq:main_PDE} satisfying $R[u_{\theta_\star}] = 0$, 
in general $L_{\mathrm{EM},\tau}(\theta_\star) > \inf_{u_\theta \in \calU} L_{\mathrm{EM},\tau}(\theta)$, and hence
optimizing $L_{\mathrm{EM},\tau}(\theta)$ can lead to sub-optimal solutions \emph{even in the limit of infinite simulated data.}
Furthermore, this bias \emph{cannot} be resolved by simply reducing the step-size $\tau$, since the PDE residual term and the bias term are both the same order (cf.~\eqref{eq:EM_tau_one_step_limit}).
We illustrate this with a one-dimensional example in \Cref{fig:loss_EM_tau}.
Proofs for both \Cref{lemma:euler_maruyama_bsde} and \Cref{thm:EM_final_form}
are given in \Cref{sec:appendix:proofs:euler_maruyama_bsde}.

\subsection{Stratonovich BSDEs and Stochastic Heun Integration}
\label{sec:results:strat_BSDE}

Our next step is to derive a new BSDE loss based on Heun integration. The starting point is to interpret the forward SDE as a Stratonovich SDE (in contrast to \eqref{eq:forward_SDE} which is an It{\^{o}} SDE):
\begin{align}
    \rmd \Xs_t = f(\Xs_t, t) \rmd t + g(\Xs_t, t) \circ \rmd B_t, \quad \Xs_0 = x_0. \label{eq:S_forward_SDE}
\end{align}
For any $u$ that satisfies 
\eqref{eq:main_PDE}, we have 
$\rmd u(\Xs_t, t) 
    = \hs[u](\Xs_t, t) \rmd t + \nabla u(\Xs_t, t)^\T g(\Xs_t, t) \circ \rmd B_t$
with 
$\hs[u](x, t) := h[u](x, t) - \frac{1}{2}\Tr(H(x, t) \nabla^2 u(x, t))$
by the Stratonovich chain rule,
which motivates the following $H$-horizon self-consistency Stratonovich BSDE loss:
\begin{align}
    L_{\SBSDE,H}(\theta) &:= \E_{x_0,B_t} \frac{1}{NH^2} \sum_{n=0}^{N-1} \left(  u_\theta(\Xs_{t_{n+1}}, t_{n+1}) - u_\theta(\Xs_{t_n}, t_n) - S^{\sbullet}_\theta(t_n, t_{n+1}) \right)^2,\label{eq:S_BSDE}
\end{align}
with
$S^{\sbullet}_\theta(t_0, t_1) := \int_{t_0}^{t_1} \hs_\theta(\Xs_t, t) \rmd t + \int_{t_0}^{t_1} \nabla u_\theta(\Xs_t, t)^\T g(\Xs_t, t) \circ \rmd B_t$
where $\hs_\theta(x, t) := \hs[u_\theta](x, t)$.
As \eqref{eq:S_BSDE} 
utilizes Stratonovich integration,
the Euler-Maruyama scheme cannot be used for integration, as it converges to the It{\^{o}} solution.
Hence, we will consider the stochastic Heun integrator~\cite{rumelin1982numericalintegration,kloeden1992numericalsolutions} which has the favorable property of converging to the Stratonovich solution.
We proceed first by defining the augmented forward and backward SDE process $\Zs_t := (\Xs_t, \Ys_t)$:
\begin{align}
    \rmd \cvectwo{ \Xs_t }{ \Ys_t } &= \cvectwo{ f(\Xs_t, t) }{ \hs_\theta(\Xs_t, t) } \rmd t + \cvectwo{ g(\Xs_t, t) }{ \nabla u_\theta(\Xs_t, t)^\T g(\Xs_t, t) } \circ \rmd B_t, \quad \cvectwo{X_0}{\Ys_0} = \cvectwo{x_0}{u_\theta(x_0, 0)},  \label{eq:forward_backwards_strat} \\
    &=: F_\theta(\Zs_t, t) \rmd t + G_\theta(\Zs_t, t) \circ \rmd B_t. \nonumber
\end{align}

\begin{figure}[t]
    \centering
    \begin{subfigure}[t]{0.48\textwidth}
        \centering
        \includegraphics[width=0.95\textwidth]{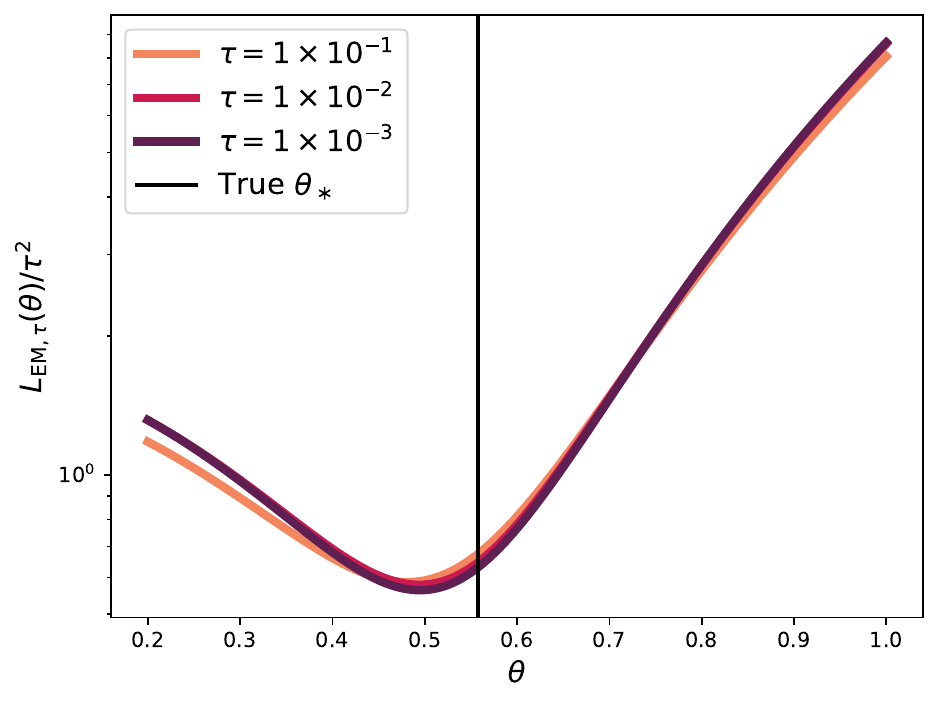}
        \caption{Plot of $L_{\mathrm{EM},\tau}(\theta)$
        at $\tau \in \{10^{-1}, 10^{-2}, 10^{-3}\}$ levels of discretization.}
        \label{fig:loss_EM_tau}
    \end{subfigure}
    \hfill
    \begin{subfigure}[t]{0.48\textwidth}
        \centering
        \includegraphics[width=0.95\textwidth]{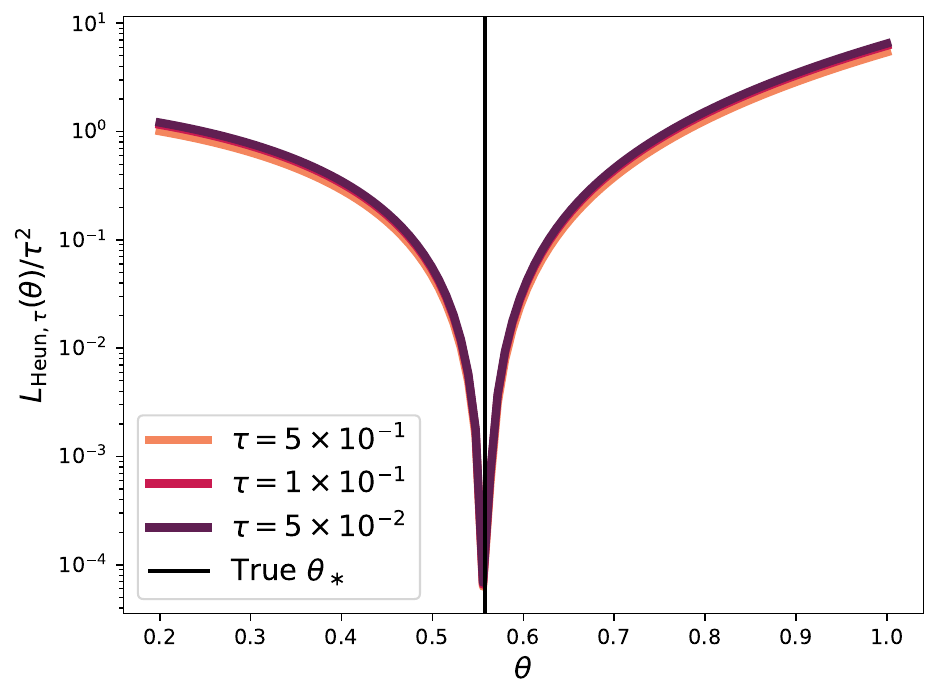}
        \caption{Plot of $L_{\mathrm{Heun},\tau}(\theta)$
        at $\tau \in \{5 \times 10^{-1}, 10^{-1}, 5 \times 10^{-2}\}$ levels of discretization.}
        \label{fig:loss_Heun_tau}
    \end{subfigure}
\caption{A plot of both $L_{\mathrm{EM},\tau}(\theta)$ and $L_{\mathrm{Heun},\tau}(\theta)$ at various levels of discretization. The PDE is a one dimensional Linear Quadratic Regulator HJB equation, where $\theta$ parameterizes a quadratic value function.}
\label{fig:one_step}
\end{figure}

The augmented SDE is discretized as follows using the stochastic Heun scheme:
\begin{align}
    \Zbars_{n+1} &= \Zhats_n + \tau F_\theta(\Zhats_n, t_n) + \sqrt{\tau} G_\theta(\Zhats_n, t_n) w_n, \quad w_n \sim \sfN(0, I_d), \label{eq:heun_scheme} \\
    \Zhats_{n+1} &= \Zhats_n + \frac{\tau}{2}\left(  F_\theta(\Zhats_n, t_n) + F_\theta(\Zbars_{n+1}, t_{n+1}) \right) + \frac{\sqrt{\tau}}{2} \left( G_\theta(\Zhats_n, t_n) + G_\theta(\Zbars_{n+1}, t_{n+1}) \right) w_n, \nonumber 
\end{align}
with $\Zhats_0 = (x_0, u_\theta(x_0, 0))$.
This gives rise to the $k$-step Heun-BSDE loss defined as:
\begin{align}
    L_{\mathrm{Heun}_k,\tau}(\theta) &:= \beta_k \mathop{\E}_{x_0,w_n} \sum_{n=0}^{\frac{N}{k}-1} \left( u_\theta(\Xhats_{(n+1)k}, t_{(n+1)k}) - u_\theta(\Xhats_{nk}, t_{nk}) - (\Yhats_{(n+1)k} - \Yhats_{nk}) \right)^2, \label{eq:BSDE_Heun}
\end{align}
with $\beta_k := \frac{k}{N\tau^2}$.
For the one-step $k=1$ case, we use the shorthand $L_{\mathrm{Heun},\tau}(\theta)$.
We now show that the one-step Heun-BSDE loss $L_{\mathrm{Heun},\tau}(\theta)$
avoids the undesirable bias term which appears 
for the one-step EM-BSDE loss $L_{\mathrm{EM},\tau}(\theta)$ (cf.~\Cref{sec:results:euler_BSDE}).
To do this, analogous to \eqref{eq:one_step_EM_BSDE}, 
we define the point-wise Heun loss at resolution $\tau$ for a fixed $(x, t)$ as:
\begin{align*}
    \ell_{\mathrm{Heun},\tau}(\theta, x, t) &:= \E_{w}( u_\theta(\hat{x}_{t+\tau}, t + \tau) - \hat{y}^\theta_{t+\tau})^2, \\
    \bar{z}^\theta_{t+\tau} &= z^\theta_t + \tau F_\theta(z^\theta_t, t) + \sqrt{\tau} G_\theta(z^\theta_t, t) w, \quad z^\theta_t = (x, u_\theta(x, t)), \\
    \hat{z}^\theta_{t+\tau} &= z^\theta_t + \frac{\tau}{2}( F_\theta(z^\theta_t, t) + F_\theta(\bar{z}^\theta_{t+\tau}, t+\tau)) + \frac{\sqrt{\tau}}{2}( G_\theta(z^\theta_t, t) + G_\theta(\bar{z}^\theta_{t+\tau}, t+\tau)) w,
\end{align*}
noting that $\hat{z}^\theta_{t+\tau} = (\hat{x}_{t+\tau}, \hat{y}^\theta_{t+\tau})$.
Similar to $L_{\mathrm{EM},\tau}(\theta)$,
we have the following identity
$L_{\mathrm{Heun},\tau}(\theta) = \frac{1}{N \tau^2}\sum_{n=0}^{N-1} \E_{\Xhats_n}[ \ell_{\mathrm{Heun},\tau}(\theta, \Xhats_n, t_n) ]$.
Our next result illustrates that the point-wise Heun loss avoids the issues
identified with the point-wise EM loss in \Cref{lemma:euler_maruyama_bsde}.

\begin{restatable}{mylemma}{Heunthm}
\label{lemma:heun_bsde}
Suppose that 
$f$, $g$, and $h_\theta$ are all in $C^{1,1}$, 
and $u_\theta$ is in $C^{3,1}$.
We have that
\begin{align}
    \tau^{-2} \cdot \ell_{\mathrm{Heun},\tau}(\theta, x, t) = (R[u_\theta](x, t))^2 + O(\tau^{1/2}),
\end{align}
where the $O(\cdot)$ hides factors depending on $d$ and the H{\"{o}}lder norms of $f$, $g$, $h_\theta$, and $u_\theta$.
\end{restatable}

Furthermore, analogously to 
\Cref{thm:EM_final_form}, we can utilize \Cref{lemma:heun_bsde} in conjunction with the order $1/2$ strong convergence of stochastic Heun (cf.~\Cref{sec:appendix:SDE_convergence})
to the Stratonovich solution
to show the following relationship for the full loss
$L_{\mathrm{Heun},\tau}(\theta)$.
\begin{restatable}{mythm}{Heunfinalform}
\label{thm:heun_final_form}
Suppose that $f$, $g$, and $h_\theta$ are all in $C^{1,1}$, $u_\theta \in C^{3,1}$, and $\tau \leq 1$.
We have that
\begin{align}
    L_{\mathrm{Heun},\tau}(\theta) = \frac{1}{T} \int_0^T \E \left[(R[u_\theta](\Xs_t, t))^2\right] \rmd t + O(\tau^{1/2}), \label{eq:Heun_tau_one_step_limit}
\end{align}
where the $O(\cdot)$ hides factors depending on $d$, $T$, and the H{\"{o}}lder norms of $f$, $g$, $h_\theta$, and $u_\theta$.
\end{restatable}

Therefore, unlike the situation with EM integration in \eqref{eq:EM_tau_one_step_limit}, 
any additional bias terms only enter through a $O(\tau^{1/2})$ 
term which is of higher order than the leading PDE residue term (cf.~\eqref{eq:Heun_tau_one_step_limit}).
In \Cref{fig:loss_Heun_tau}, we show the plot of $L_{\mathrm{Heun},\tau}(\theta)$ on the same HJB PDE problem as in \Cref{fig:loss_EM_tau}, and show that
the bias issue in $L_{\mathrm{EM},\tau}(\theta)$
is now resolved.
The proofs of both
\Cref{lemma:heun_bsde}
and \Cref{thm:heun_final_form}
are given in \Cref{sec:appendix:proofs:heun_bsde}.

\section{Trade-offs for Multi-Step BSDE Losses}
\label{sec:long_horizon_BSDE_analysis}

In \Cref{sec:one_step_BSDE}, we conducted a thorough analysis of the one-step self-consistency losses 
$L_{\mathrm{EM},\tau}(\theta)$ and $L_{\mathrm{Heun},\tau}(\theta)$.
We now consider the other extreme: the full-horizon ($k=N$)
losses $L_{\mathrm{EM}_N,\tau}(\theta)$ and $L_{\mathrm{Heun}_N,\tau}(\theta)$.
The intermediate multi-step regime~\cite{nusken2023interpolating}, where $1 < k < N$, 
serves as an extension to the cited method and is studied experimentally in \Cref{sec:experiments}.
Due to space constraints, we defer the precise theorem statements arising from our analysis, in addition to the proofs, to \Cref{sec:appendix:long_horizon_BSDE_analysis}.

\MyParagraph{BSDE loss and Euler-Maruyama discretization.}
We start with the $H$-horizon BSDE loss \eqref{eq:BSDE}.
Using Jensen's inequality, we show 
(\Cref{prop:BSDE_T_vs_tau})
the relationship 
$L_{\mathrm{BSDE},T}(\theta) \leq L_{\mathrm{BSDE},\tau}(\theta) + O(\tau^{1/2})$.
Thus, at the SDE level, the benefits of using the full-horizon loss $L_{\mathrm{BSDE},T}(\theta)$ 
over the $\tau$-horizon loss $L_{\mathrm{BSDE},\tau}(\theta)$ are not clear, 
given that (a) the full-horizon loss is dominated by the latter $\tau$-horizon loss (up to an order $\tau^{1/2}$ term), and 
(b) $L_{\mathrm{BSDE},\tau}(\theta)$ does indeed vanish for an optimal $\theta_\star$. 

The situation 
becomes more complex when factoring in EM discretization.
Using the order $1/2$ strong convergence of EM, we show (\Cref{prop:EM_N_tau_vs_BSDE_T}) that
$L_{\mathrm{EM}_N,\tau}(\theta) = L_{\mathrm{BSDE},T}(\theta) + O(\tau^{1/2})$.
On the other hand, 
from \Cref{thm:EM_final_form} we
also show (\Cref{prop:EM_tau_vs_BSDE_tau}) that
$L_{\mathrm{EM},\tau}(\theta) = L_{\mathrm{BSDE},\tau}(\theta) + \mathrm{Bias}(\theta) + O(\tau^{1/2})$ 
with the bias term $\mathrm{Bias}(\theta) := \frac{1}{2T} \int_0^T \E \Tr((H(X_t, t) \cdot \nabla^2 u_\theta(X_t, t))^2) \rmd t$ not vanishing as $\tau \to 0$.
Hence, the loss $L_{\mathrm{EM}_N,\tau}(\theta)$ presents an advantage over $L_{\mathrm{EM},\tau}(\theta)$ for sufficiently small discretization sizes in terms of bias. However, the inequality 
$L_{\mathrm{BSDE},T}(\theta) \leq L_{\mathrm{BSDE},\tau}(\theta) + O(\tau^{1/2})$
still holds,
meaning that while $L_{\mathrm{EM}_N,\tau}(\theta)$ does not
suffer from the bias issues identified in $L_{\mathrm{EM},\tau}(\theta)$,
the trade-off is that the
loss $L_{\mathrm{BSDE},T}(\theta)$ it approximates without bias is
nearly dominated by
another loss $L_{\mathrm{BSDE},\tau}(\theta)$;
this is precisely the loss that $L_{\mathrm{EM},\tau}(\theta)$ attempts to approximate,  
but it does so in a way that introduces an irreducible bias term $\mathrm{Bias}(\theta)$. Thus, neither of the EM-BSDE losses for  $k=1$ nor $k=N$ provides a completely satisfactory solution.
In \Cref{sec:experiments:skipped_self_regularization}, we illustrate 
these issues empirically.
Furthermore, in light of this analysis, we can interpret the interpolating loss of \cite{nusken2023interpolating} as attempting to resolve this trade-off by finding the best intermediate multi-step $k \in \{1, \dots, N\}$.

\MyParagraph{Stratonovich BSDE and Heun discretization.}
In the setting of the Stratonovich BSDE and the Heun-BSDE loss, we first show (\Cref{prop:SBSDE_T_vs_tau}) that 
$L_{\SBSDE,T}(\theta) \leq L_{\SBSDE,\tau}(\theta) + O(\tau^{1/2})$ holds at the SDE level,
analogous to the relationship between $L_{\mathrm{BSDE},T}(\theta)$ and $L_{\mathrm{BSDE},\tau}(\theta)$.
Next, we use the order $1/2$ strong convergence of Heun to show (\Cref{prop:heun_N_tau_vs_SBSDE_T}) that
$L_{\mathrm{Heun}_N,\tau}(\theta) = L_{\SBSDE,T}(\theta) + O(\tau^{1/2})$;
again analogous to the relationship between
$L_{\mathrm{EM}_N,\tau}(\theta)$ and $L_{\mathrm{BSDE},T}(\theta)$.
However, unlike the one-step EM case, using \Cref{thm:heun_final_form}
we show (\Cref{prop:heun_tau_vs_SBSDE_tau}) that $L_{\mathrm{Heun},\tau}(\theta) = L_{\SBSDE,\tau}(\theta) + O(\tau^{1/2})$,
from which we conclude $L_{\mathrm{Heun}_N,\tau}(\theta) \leq L_{\mathrm{Heun},\tau}(\theta) + O(\tau^{1/2})$.
Thus---unlike the EM setting---the relationship between the full-horizon and one-step case 
at both the SDE level and Heun-BSDE level is the same, suggesting questionable benefits
of $L_{\mathrm{Heun}_N,\tau}(\theta)$ over $L_{\mathrm{Heun},\tau}(\theta)$. In \Cref{sec:experiments:skipped_self_regularization}, we show that this conclusion is indeed reflected in practice.

\section{Experiments}
\label{sec:experiments}

In this section, we compare the proposed Heun-based BSDE method against both standard PINNs, a variant of PINNs which uses the forward SDE to sample collocation points,
and standard EM-based BSDE solvers
on various high-dimensional PDE problems.
Specifically, we compare the methods:%
\begin{enumerate}[label=(\alph*), itemsep=0pt, topsep=0pt, parsep=4pt, partopsep=0pt, leftmargin=*]
    \item \textbf{PINNs:} The standard PINNs loss $L_{\mathrm{PINNs}}(\theta)$ from \eqref{eq:PINNs} is minimized.
    Since we consider unbounded domains,
    the collocation measure $\mu$ over $(x, t)$ is chosen
    by fitting a normal distribution over the spatial dimensions of the forward SDE trajectories prior to training.
    \item \textbf{FS-PINNs:} 
    The standard PINNs loss \eqref{eq:PINNs} is again minimized,
    where the measure $\mu$ over space-time
    is chosen by directly sampling trajectories from the 
    forward SDE \eqref{eq:forward_SDE}. %
    \item \textbf{EM-BSDE:} 
    The self-consistency loss discretized with the standard Euler-Maruyama (EM) scheme, i.e., $L_{\mathrm{EM}_k,\tau}(\theta)$ from \eqref{eq:BSDE_EM} as described in \Cref{sec:background}.
    \item \textbf{EM-BSDE (NR):}
    A variant of EM-BSDE where we use the BSDE to propagate $Y_t$ instead of setting it directly to $u_\theta(X_t, t)$~\cite[cf.][]{raissi2024forward,nusken2023interpolating}. 
    We refer to this variant as the \emph{no-reset} (NR) variant.
    Specifically, the backward SDEs in
    \eqref{eq:backwards_SDE_EM} is integrated as (starting from $\hat{Y}^\theta_0 = u_\theta(x_0, 0)$):
    \begin{align}
    \hat{Y}^\theta_{n+1} &= \hat{Y}^\theta_n + \tau h_0(\hat{X}_n, t_n, \hat{Y}^\theta_n, \nabla u_\theta(\hat{X}_n, t_n)) + \sqrt{\tau} \nabla u_\theta(\hat{X}_n, t_n)^\T g(\hat{X}_n, t_n) w_n, \label{eq:backwards_SDE_EM_NR}
    \end{align}
    with $w_n \sim \sfN(0, I_d)$.
    The loss $L_{\mathrm{EM}_k,\tau}(\theta)$ remains the same
    except replacing \eqref{eq:backwards_SDE_EM} with \eqref{eq:backwards_SDE_EM_NR}.
    \item \textbf{Heun-BSDE (Ours):} 
    The self-consistency loss discretized with stochastic Heun integration, i.e., $L_{\mathrm{Heun}_k,\tau}(\theta)$ from \eqref{eq:BSDE_Heun} as described in \Cref{sec:results:strat_BSDE}.
\end{enumerate}%
We evaluate these methods on three PDE benchmark problems:
(i) a Hamilton-Jacobi-Bellman (HJB) equation~\cite{raissi2024forward},
(ii) a Black-Scholes-Barenblatt (BSB) equation~\cite{raissi2024forward}, and
(iii) a fully-coupled FBSDE from Bender \& Zhang (BZ)~\cite{bender2008coupledfbsdes}; the PDEs are detailed in \Cref{sec:appendix:experiments:PDEs}.
In addition, we evaluate the methods on an optimal control pendulum swing-up problem to demonstrate application to a non-linear control problem (see~\Cref{sec:appendix:experiments:pendulum}).
To evaluate model performance, the analytical solution (available for all PDEs under consideration) is %
compared with the model output along 5 forward SDE trajectories, using the \emph{relative $L_2$ error} (RL2) metric:
\begin{align}
    \mathrm{RL2} := \sqrt{\frac{\sum_{i=0}^{N}\left(u_{\mathrm{ref}}(X_{t_i},t_i) -u_{\mathrm{pred}}(X_{t_i}, t_i) \right)^2}{\sum_{i=0}^{N} u_{\mathrm{ref}}^2(X_{t_i}, t_i)}}. \label{eq:RL2}
\end{align}
Unless otherwise noted, we set $T=1$ and $N=50$ (i.e., $\tau=0.02$). Model architectures and training details are described  in~\Cref{sec:appendix:experiments}.
Additionally, the code to reproduce our experiments is available at: \url{https://github.com/sungje-park/heunbsde}.
For what follows, we report two main sets of results on
(i) one-step self-consistency losses
(\Cref{sec:experiments:one_step}) and
(ii) multi-step self-consistency losses
(\Cref{sec:experiments:skipped_self_regularization}).

\MyParagraph{Efficient sub-sampling BSDE implementation.}
In our experimental results, we consider both a full FSDE rollout algorithm,
in addition to a batched, sub-sampled FSDE rollout variation of FS-PINNs, EM-BSDE, and Heun-BSDE,
which we find performs similarly to the original algorithm while providing significant speed improvements. 
We sketch the details of the sub-sampled FSDE variation in~\Cref{alg:bsdebatchedsimple};
further details on the two algorithms can be found in~\Cref{sec:appendix:experiments:algorithms}.

\begin{algorithm}[hbt]
    \caption{Batched, Sub-sampling BSDE Algorithm (Simplified)}
    \label{alg:bsdebatchedsimple}
    \begin{algorithmic}[1]
        \small
        \Input Neural network $\hat u_\theta(x,t)$, parameters $\theta$,
        terminal function $\phi$, time step $\Delta t$, trajectory length $N$, 
        evaluation batch size $B$.
        \State Sample initial state: $(x[0],t[0])=(x_0,0)$, with $x_0 \sim \mu$
        \State Sample Brownian noise: $\xi[0:N-1]\sim \mathsf{N}(0,I_d)$
        \State Evaluate network at initial state: $(u,u_x)=(\hat u_\theta(x[0],t[0]),\nabla_x\hat u_\theta(x[0],t[0]))$
        \State \verb|/* Forward SDE rollout */|
        \For{$i=0,\dots,N-1$} %
            \State \verb|/* Use either EM or Heun integration */|
            \State Propagate forward state (with $(u, u_x)$ if coupled): $x[i+1]=x[i]+\Delta x$ %
            \State Propagate time: $t[i+1]=t[i]+\Delta t$
            \State Evaluate network at new state: $(u,u_x)=(\hat u_\theta(x[i+1],t[i+1]),\nabla_x\hat u_\theta(x[i+1],t[i+1]))$
        \EndFor
            \State Stop gradient: $x[0:N]=\mathrm{SG}(x[0:N])$
            \State Random sub-sampling: $(x_i,x_{i+1},t_i,t_{i+1})=\text{perm}(x_i,x_{i+1},t_i,t_{i+1})[0:B-1]$ %
            \State \verb|/* Use either EM or Heun integration */|
            \State Propagate backward SDE at batched points: $y_{i+1}=u_i+\Delta y$ %
            \State \verb|/* Use PINNs loss instead for FS-PINNs */|
            \State Compute self-consistency loss: $\mathcal{L}_{\mathrm{sr}}=\frac NB \sum_{i=0}^{B-1}\left(\hat u_{i+1}-y_{i+1}\right)^2$ %
            \State Compute terminal loss: $\mathcal{L}_{\phi}=(u+ \phi)^2+ \norm{u_x-\nabla_x \phi}^2$
    \end{algorithmic}
\end{algorithm}

 \begin{figure}[!b]
    \centering
    \begin{subfigure}[t]{0.48\textwidth}
        \centering
        \includegraphics[width=\textwidth]{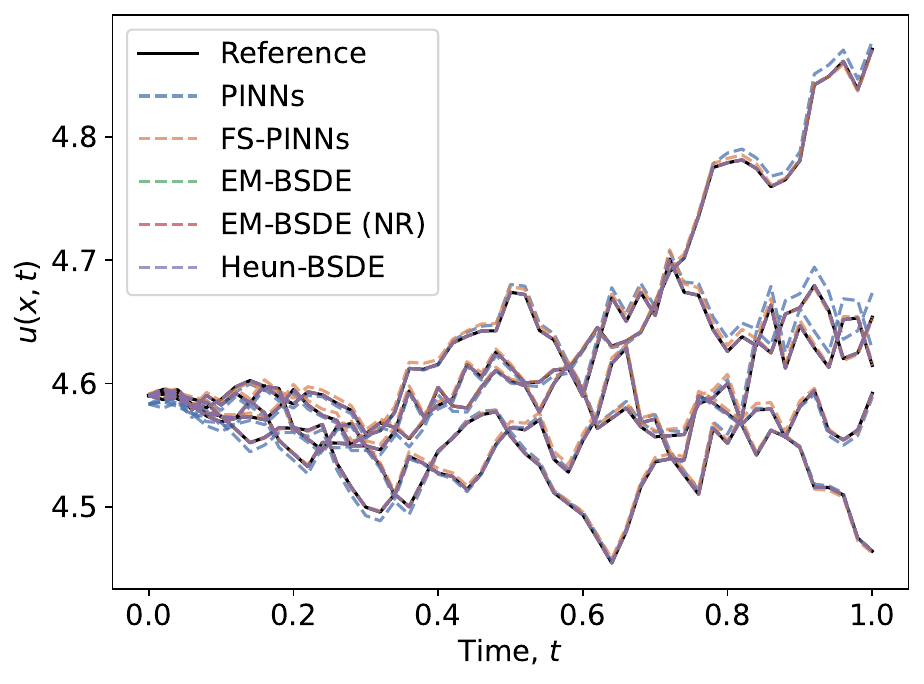}
        \caption{Plot of the learned solution of different models on the 100D HJB problem.}
        \label{fig:hjbsol}
    \end{subfigure}
    \hfill
    \begin{subfigure}[t]{0.48\textwidth}
        \centering
        \includegraphics[width=\textwidth]{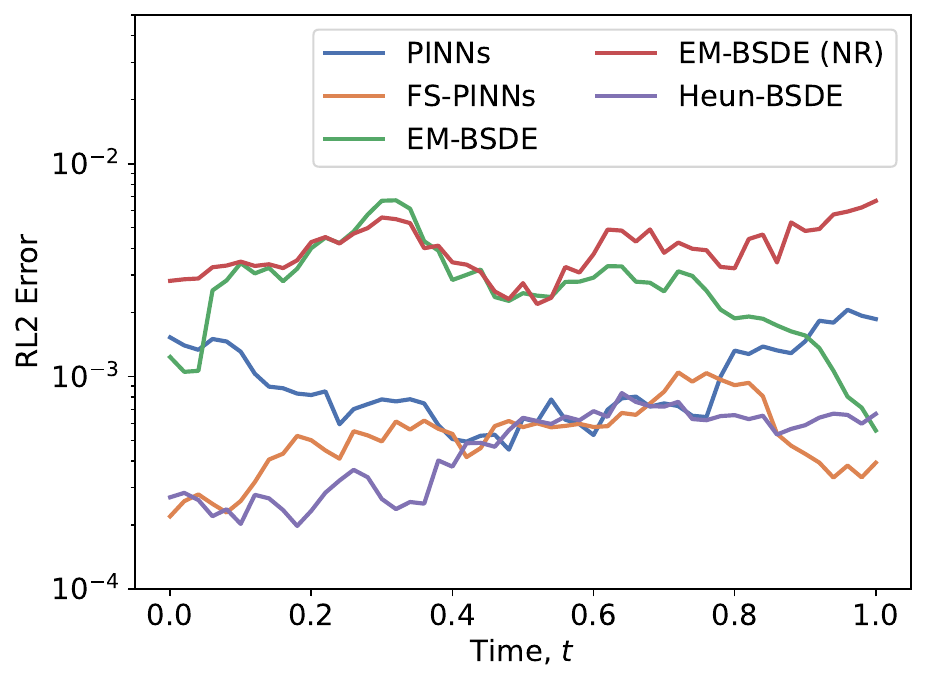}
        \caption{Plot of the RL2 errors across time $t \in [0, 1]$ for the 100D HJB case.}
        \label{fig:hjberr}
    \end{subfigure}
    \caption{
    A plot of the 100D HJB reference and learned solutions for each model and the associated RL2 errors.}
    \label{fig:hjb}
\end{figure}

\subsection{One-Step Self-Consistency Losses}
\label{sec:experiments:one_step}

\begin{table}[!t]
    \centering
    \begin{tabular}{ wc{.1\linewidth}|wc{.14\linewidth}|wc{.14\linewidth}|wc{.15\linewidth}|wc{.15\linewidth}|wc{.14\linewidth}}
        Cases & PINNs & FS-PINNs & EM-BSDE (NR) & EM-BSDE & Heun-BSDE \\ \hline\hline
        \multicolumn{6}{c}{Full Algorithm (cf. \Cref{alg:bsdeoriginal}) Results} \\ \hline\hline
        100D HJB & $.1260\pm.0107$ & $.0737 \pm.0110$ & $.5214\pm.0452$ & $.3626\pm.0113$ & $\mathbf{.0493 \pm .0109}$ \\
        100D BSB & $1.5066\pm.2349$ & $\mathbf{.0497 \pm .0031}$ & $.1855 \pm .0078$ & $.3735 \pm .0470$ & $.0535 \pm .0113$ \\
        100D BZ & - & - & - & $\mathbf{3.1259 \pm .1807}$ & $3.5619 \pm .2716$ \\
        10D BZ & $3.8566\pm.0310$ & $.0351 \pm .0041$ & $.1309 \pm .0311$ & $.1903 \pm .0066$ & $\mathbf{.0228 \pm .0016}$ \\ \hline\hline 
        
        \multicolumn{6}{c}{Batched Algorithm (cf. \Cref{alg:bsdebatched}) Results} \\ \hline\hline
        100D HJB & $.1362 \pm .0276$ & $.1828 \pm .0774$ & $.5214 \pm .0452$ & $.3831 \pm .0084$ & $\mathbf{.0573 \pm .0106}$ \\
        100D BSB & $3.0488 \pm 1.5625$ & $.0851 \pm .0027$ & $.1855 \pm .0078$ & $.3668 \pm .0244$ & $\mathbf{.0472 \pm .0076}$ \\
        100D BZ & - & $5.4502 \pm .1351$ & - & $5.7330 \pm .2342$ & $\mathbf{1.7973 \pm .1108}$ \\
        10D BZ & $3.8495 \pm .1562$ & $.0270 \pm .0017$ & $.1309 \pm .0311$ & $.1933 \pm .0022$ & $\mathbf{.0236 \pm .0031}$\\
    \end{tabular}
    \caption{
    Summary of RL2 errors averaged over three different initialization random seeds, $\pm$ one standard deviation. Settings that failed to converge to a satisfactory solution are denoted with -.
    The first set of results correspond to the full algorithm (see~\Cref{sec:appendix:experiments:algorithms} for a detailed description), whereas the second set of results correspond
    to the batched algorithm (cf.~\Cref{alg:bsdebatchedsimple}).
    }
    \label{tab:resultsummary}
\end{table}

For our first set of results, we compare the PINNs baselines and the one-step ($k=1$)
EM-BSDE baseline with our one-step Heun-BSDE method.
We solve each PDE 
instantiated with a state space of 100 dimensions
using three different initialization seeds for training.
The results are reported in \Cref{tab:resultsummary}, which shows that for nearly all the cases, the Heun-BSDE method outperforms EM-BSDE methods (i.e., lower RL2 error) as predicted by our analysis in \Cref{sec:one_step_BSDE}.
Furthermore, \Cref{fig:hjb} illustrates the performance of all methods across time $t \in [0, 1]$ on the 100D HJB case, which also highlights the low RL2 error of the Heun-BSDE method.
The one exception to the trend is the 100D BZ case, where all methods failed to converge to a high-quality solution. 
We hypothesize due to the high dimensionality of the problem involving fully-coupled SDEs, the optimization landscape for all methods is too complex to recover high-fidelity solutions.
To evaluate this hypothesis,
we further reduce the dimensionality of the BZ problem to 10D,
which restores the relative performance of all methods (cf.~\Cref{tab:resultsummary}, last row).
Another key observation from \Cref{tab:resultsummary} is that FS-PINNs and Heun-BSDE perform similarly across all cases, showing that
parity between the BSDE and PINNs 
is restored through Heun integration.
Finally, we note that the performance of PINNs is quite poor in comparison to FS-PINNs, which illustrates the relative impact of the %
sampling distribution $\mu$ for PINNs methods (cf.~\cite{nabian2021sampling,daw23sampling,zhang2025sampling}).

To show that the gap between EM and Heun performance cannot be closed with finer discretization meshes,
we re-run the 10D BSB example at varying discretization sizes. The results are reported in 
\Cref{fig:discresults}, which show that 
the EM-BSDE methods only experience minimal improvement
with smaller discretization size compared with the Heun-BSDE method. 
\Cref{fig:discresults} corroborates the findings in \Cref{sec:one_step_BSDE}, which shows that the one-step EM-BSDE loss
contains a bias term of the same order as the residual error 
which is not present in the one-step Heun-BSDE loss.

\begin{wrapfigure}[9]{r}{0.4\textwidth}
    \vspace{-15pt}
    \begin{minipage}{\linewidth}
    \centering
        \begin{tabular}{c|c|c}
            Method & Full & Batched \\ \hline \hline
            PINNs & 1x & 1x \\
            FS-PINNs & 2.64x & 1.14x \\
            EM-BSDE & 2.83x & 0.34x \\
            Heun-BSDE & 36.37x & 2.03x
        \end{tabular}
        \captionof{table}{A table of average training time overhead relative to PINNs for both the full and batched algorithm runs.}
        \label{tab:runtimes_average}
    \end{minipage}
\end{wrapfigure}

\begin{figure}[b]
    \begin{minipage}[t]{.48\linewidth}
        \centering
        \includegraphics[width=0.98\linewidth]{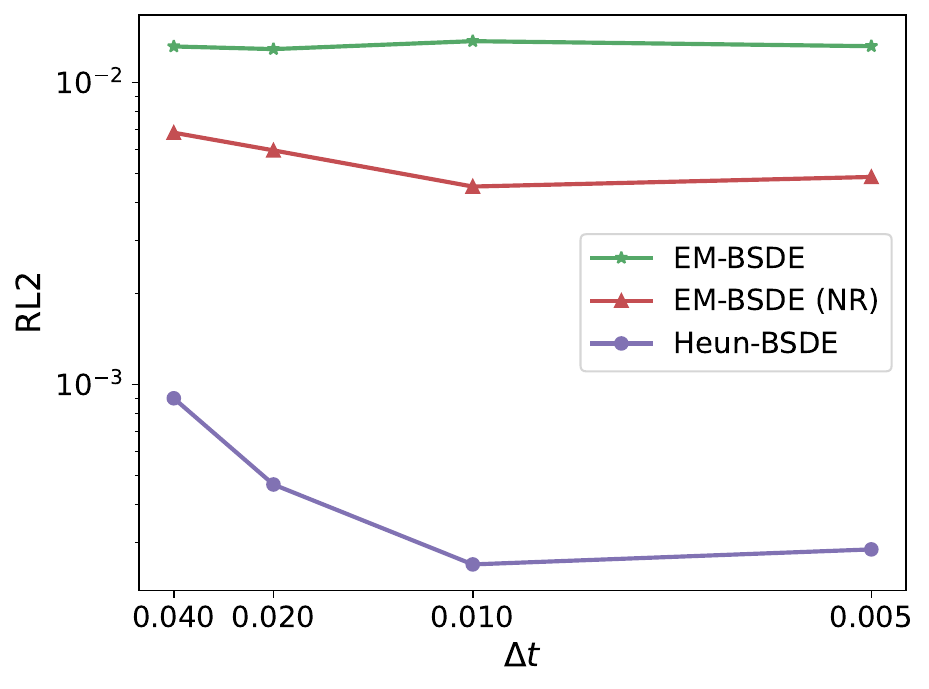}
        \captionof{figure}{RL2 performance for 10D BSB at
        discretization step-sizes $\tau=N^{-1}$ for $N \in \{25, 50, 100, 200\}$.}
        \label{fig:discresults}
    \end{minipage}
    \hfill
    \begin{minipage}[t]{.48\linewidth}
        \centering
        \includegraphics[width=\linewidth]{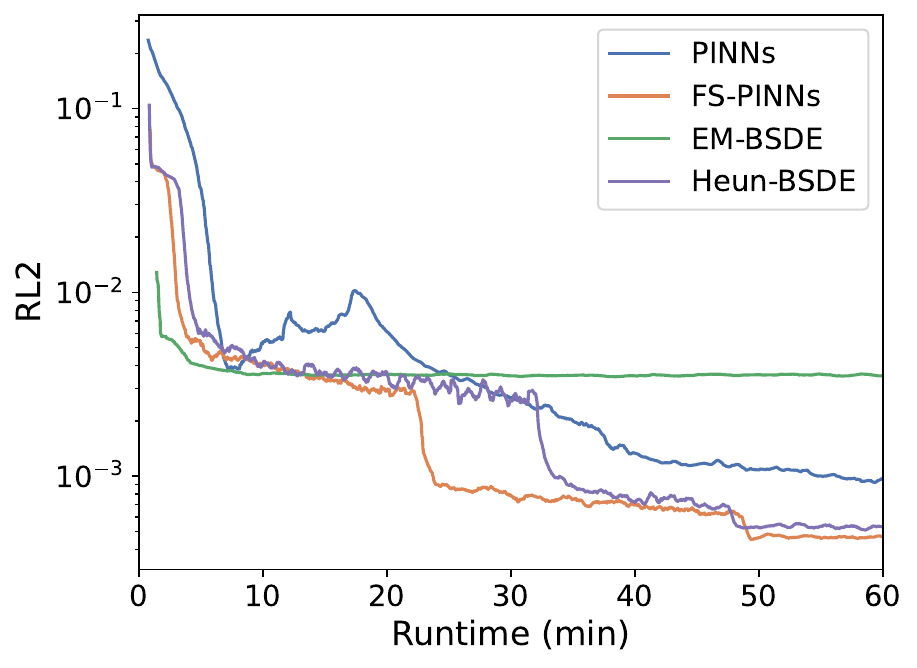}
        \captionof{figure}{A plot of the RL2 performance versus runtime for the 100D HJB problem.}
        \label{fig:runtime}
    \end{minipage}
\end{figure}
\MyParagraph{Computational considerations.}
Although Heun-BSDE outperforms EM-BSDE, it comes at a computational cost.
In \Cref{tab:runtimes_average}, 
we report the average training time for each method on a single NVIDIA A100 GPU; on average Heun-BSDE is approximately 6x slower than the EM-BSDE method for the batched algorithm. There are two major factors to this overhead. First, for the specific elliptic/parabolic PDEs we consider in \eqref{eq:main_PDE}, EM-BSDE does not require the computationally expensive Hessian computation $\nabla^2 u(x, t)$, which PINNs, FS-PINNs, and Heun-BSDE all do require. However, this does not necessarily hold true for all PDEs (e.g.,~\Cref{sec:appendix:experiments:pendulum}). Second, the Heun integration requires approximately double the compute of the EM integration---this is clearly reflected in the overhead between FS-PINNs and Heun-BSDE.

In~\Cref{fig:runtime}, we show the runtime-normalized RL2 performance demonstrating that while EM-BSDE shows strong convergence at first, its performance does not improve with more compute. Conversely, both FS-PINNs and Heun-BSDE achieve similar RL2 performance at equal runtimes.

\subsection{Multi-Step Self-Consistency Losses}
\label{sec:experiments:skipped_self_regularization}

We next consider multi-step self-consistency BSDE losses~\cite{nusken2023interpolating} in order to evaluate the mathematical analysis
conducted in \Cref{sec:long_horizon_BSDE_analysis}.
Specifically, we evaluate both the multi-step $L_{\mathrm{EM}_k,\tau}(\theta)$ (cf.~\eqref{eq:BSDE_EM}) and
$L_{\mathrm{Heun}_k,\tau}(\theta)$ (cf.~\eqref{eq:BSDE_Heun})
for varying values of skip-length $k$.
For multi-step losses, we also 
need to determine where both EM-BSDE and EM-BSDE (NR) will ``reset'' the value of $Y_t$ to $u_\theta(X_t, t)$.
Note that there are many degrees of freedom here in the multi-step formulation, so we simply pick one choice as a representative choice. 
For EM-BSDE, we set $\hat{Y}^\theta_{nk} = u_\theta(\hat{X}_{nk}, t_{nk})$, and use \eqref{eq:backwards_SDE_EM_NR} to integrate between $t_{nk}$ and $t_{(n+1)k}$.
On the other hand, for EM-BSDE (NR), we directly use
the value of $\hat{Y}^\theta_{nk}$ from \eqref{eq:backwards_SDE_EM_NR}.
We also vary EM-BSDE and EM-BSDE (NR) with $N$, the number of discretization steps for the interval $[0, 1]$, varying between $N \in \{50, 200\}$ as well.
We conduct this experiment on the 10D BSB setting,
with the results reported in \Cref{fig:skipped_reg}.

 \begin{figure}[!ht]
    \centering
    \begin{subfigure}[t]{0.48\textwidth}
        \centering
        \includegraphics[width=\textwidth]{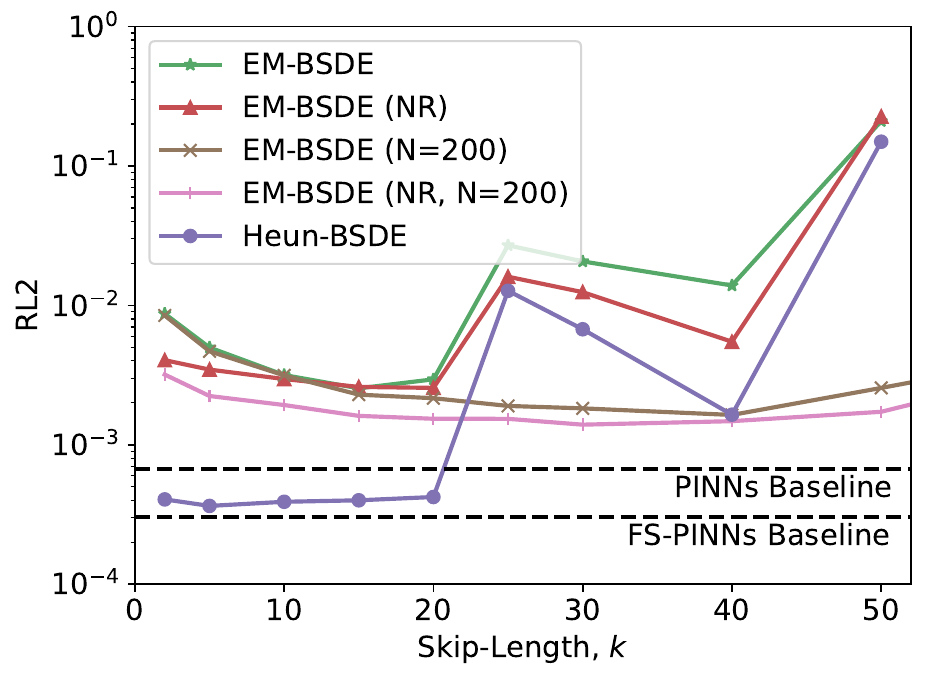}
        \caption{Plot of RL2 performance 
        as a function of the skip-length $k$, with the number of discretization steps also varying in $N \in \{50, 200\}$.}
        \label{fig:skipped_reg50}
    \end{subfigure}
    \hfill
    \begin{subfigure}[t]{0.48\textwidth}
        \centering
        \includegraphics[width=\textwidth]{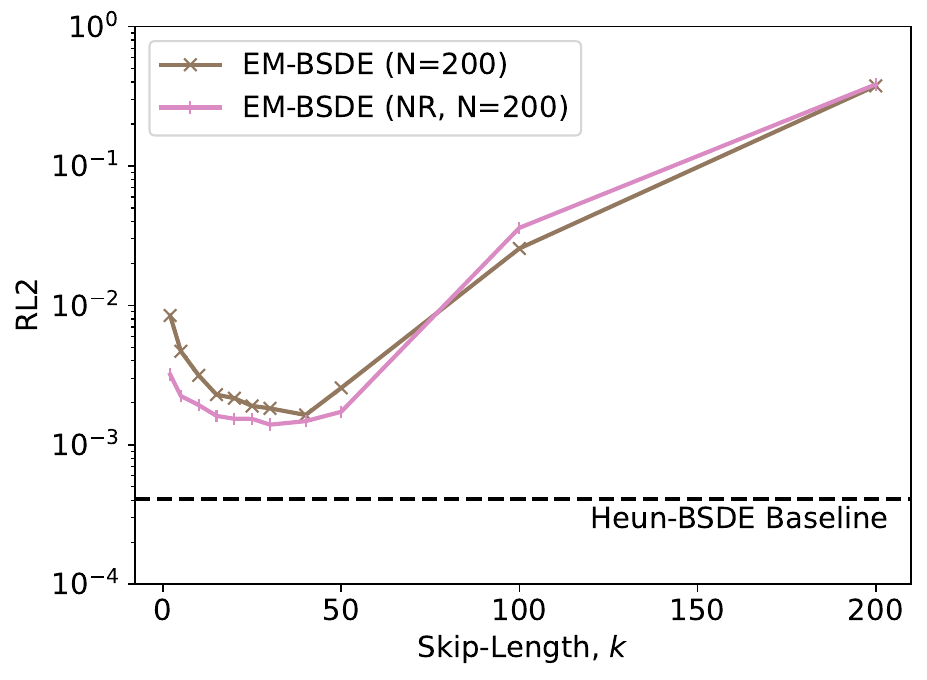}
        \caption{Plot of RL2 performance as a function of the skip-length $k$ for EM-BSDE and EM-BSDE (NR), with number of discretization steps set to $N=200$.}
        \label{fig:skipped_reg200}
    \end{subfigure}
    \caption{A plot of RL2 performance
    of each model on the 10D BSB case 
    at various skip lengths. 
    }
    \label{fig:skipped_reg}
\end{figure}
\Cref{fig:skipped_reg} shows that while the Heun-BSDE performance decreases as the skip-length $k$ increases, the performance of both EM-BSDE methods initially \emph{improves} with skip-length $k$ before then degrading, demonstrating the trade-off between minimizing the bias term and decreasing quality of the self-consistency loss identified in \Cref{sec:long_horizon_BSDE_analysis}.
Furthermore, this trade-off is present for EM-BSDE at both $N=50$ and $N=200$, illustrating again that these issues for EM are not mitigated with finer discretization
step-sizes.
Finally, although the EM-BSDE method overall improves its performance by tuning the skip-length $k$ (consistent with the findings from \cite{nusken2023interpolating}), the Heun-BSDE model at $k=1$ still outperforms the best multi-step EM-BSDE model by a significant margin.

\section{Conclusion and Future Work}
\label{sec:conclusion}

We conducted a systematic study of 
discretization strategies for BSDE-based loss formulations used to solve high-dimensional PDEs. By comparing the commonly used Euler-Maruyama scheme with stochastic Heun integration, we demonstrated that the choice of discretization can significantly impact the accuracy of BSDE-based methods. Our theoretical analysis showed that EM discretization introduces a non-trivial bias to the single-step self-consistency BSDE loss which does not vanish as the step-size decreases.
On the other hand, we show that this bias issue is not present when utilizing Heun discretization. Finally, the empirical results confirmed that the Heun scheme consistently outperforms EM in solution accuracy and performs competitively with PINNs.

Our work underscores the importance of stochastic integrator choice in BSDE solvers and suggest that higher-order schemes---though more computationally intensive---can offer substantial gains in performance. In future works, we aim to reduce Heun-BSDE's computational costs through methods such as Hutchinson trace estimation~\cite{hu2024hutchinson},
reversible Heun~\cite{kidger2021neuralsdes},
and adaptive time-stepping. 
Furthermore, while this current work focuses on understanding and restoring performance parity between BSDE and PINNs methods, future work will utilize the advantages of Heun-BSDE to solve problems such as high-dimensional stochastic control problems in model-free settings (cf.~\cite{wang2022modelfree}).

\subsubsection*{Acknowledgments}
The authors acknowledge the support of the 
USC Viterbi Summer Undergraduate Research Experience (SURE) program, which supported S. Park during
the beginning stages of this research.

\bibliographystyle{unsrt}
\bibliography{paper}

\clearpage
\appendix

\section{Limitations}
\label{sec:limitations}

We provide a concise, bulleted list of the limitations present in our work.
\begin{enumerate}[label=(\alph*), itemsep=0pt, topsep=0pt, parsep=2pt, partopsep=0pt, leftmargin=*]
    \item \textit{Computational overhead of Heun-BSDE:} As discussed in \Cref{sec:experiments:one_step} (cf.~\Cref{tab:runtimes_average}), the computational overhead of the Heun-BSDE method compared with the existing EM-BSDE methods is non-trivial.
    Furthermore, we also found the Heun integrator to be more susceptible to floating point imprecision (cf.~\Cref{tab:f32vsf64}), and hence
    we run our main experiments in \texttt{float64}, which further adds to the computation time.
    We mitigate these issues using batched computation and random sub-sampling (see~\Cref{sec:appendix:experiments:algorithms}) which helped significantly reduce computation time, but there still remains a computational penalty.
    While we believe more sophisticated techniques (e.g., randomized trace estimation and more numerically stable Heun integrators discussed in \Cref{sec:experiments:one_step}) can help to reduce 
    the computational overhead, we have not yet verified this hypothesis experimentally in our current work.

    \item \textit{Performance relative to PINNs:} Our results in \Cref{sec:experiments:one_step} show that the proposed Heun-BSDE method restores parity with the FS-PINNs method in terms of the RL2 error. While this is a significant improvement over EM-BSDE, Heun-BSDE does not yet provide significant (e.g., orders of magnitude) performance improvement over the best PINNs method (cf.~\Cref{tab:resultsummary});
    further work is needed to determine whether or not such an improvement is possible.
    Hence, the current advantage of Heun-BSDE over PINNs is with its model-free capabilities (cf.~\Cref{sec:appendix:model_free}).
    \item \textit{Limitations of theoretical analysis:} 
    While our theoretical analysis in \Cref{sec:one_step_BSDE} and \Cref{sec:long_horizon_BSDE_analysis} is fairly predictive of practice (cf.~\Cref{sec:experiments}), our analysis is not without its own set of limitations.
    One limitation is that we only analyze the two extremes of horizon length: the one-step 
    case (\Cref{sec:one_step_BSDE}) and the full-horizon $N$-step case (\Cref{sec:long_horizon_BSDE_analysis});
    the intermediate regimes (i.e., skip-lengths $k$ satisfying $1 < k < N$) are studied empirically
    (cf.~\Cref{sec:experiments:skipped_self_regularization}). 
    Another limitation is that we do not consider \emph{fully-coupled} FS-BSDEs in our analysis
    (e.g., the Bender \& Zhang (BZ) PDE, cf.~\Cref{sec:appendix:experiments:PDEs}),
    where the forwards SDE \eqref{eq:forward_SDE} is allowed to depend on $Y_t$.
\end{enumerate}

\section{Further Discussion on Related Works}
\label{sec:appendix:related_works}

While our work focuses on the impact that the widely used EM integration scheme has on the performance of BSDE-based solvers, several complementary enhancements to the BSDE loss have been proposed in literature~\cite{wang2022modelfree,naito2020weakexpansions,takanashi2022controlvariate,pham2021nnbackward,beck2021deepsplitting}. Many of these improvements are naturally compatible with our proposed Heun-BSDE framework. For example, the Heun-BSDE method could be adapted to utilize a control variate~\cite{takanashi2022controlvariate}, applied in operator splitting settings~\cite{beck2021deepsplitting}, 
and extended to fully non-linear PDEs~\cite{pham2021nnbackward}, enabling direct comparison against their EM-BSDE baseline.

Furthermore, in addition to PINNs and EM-based BSDE methods discussed in the paper, there are various other deep-learning methods for solving PDEs such as Deep Ritz~\cite{yu2018deep}, Neural Operators~\cite{kovachki2023neuraloperator}, and
tensor trains~\cite{richter2023tensortrains},
in addition to various theoretical analyses developed for It\^o-based BSDE formulations~\cite{han2020convergencebsde}.
We leave extending these approaches and analyses to 
Stratonovich-based formulations as future research directions.

\section{Details Regarding Model-Free BSDE Formulation}
\label{sec:appendix:model_free}

Suppose that the drift term $f(x, t)$ from the forwards SDE
\eqref{eq:forward_SDE} is \emph{unknown}.
Instead, suppose that our the computational model 
takes as input a realization of $(B_t)_{t=0}^{T}$
and returns FSDE \emph{trajectories} $(X_t)_{t=0}^{T}$, or more practically a sub-sampled trajectory $\{X_k\}_{k=0}^{N}$.
Under this computational model, PINNs methods cannot be used to solve the PDE \eqref{eq:main_PDE}---even if the other terms $g$, $h$, and $\phi$ are all known---since the residual term $R[u]$ cannot be computed without knowledge of the drift term $f$. However, in this settings, BSDE methods including the proposed Heun-BSDE method can still be used.
This is because BSDE losses only require access 
to the FSDE trajectory $(X_t)$ and the Brownian motion $(B_t)$ used to generate it (cf.~\Cref{eq:BSDE}).

A specific example where the proposed computational model is realistic comes from model-free optimal control, and was first described in~\cite{wang2022modelfree}. 
Consider the following deterministic continuous-time control-affine \emph{fully-actuated} system:
\begin{align}
    \dot{X}_t = f(X_t, t) + \Phi(X_t, t) U_t, \quad X_t, U_t \in \R^d, \quad \rank(\Phi(x, t)) = d\,\,\forall\,(x,t). \label{eq:control_affine}
\end{align}
We assume that we do not know the drift term $f(x, t)$, but we are able to select control inputs $U_t$ and obtain the resulting trajectories $(X_t)$; this setting is often called the \emph{model-free setting} in optimal control and reinforcement learning.
In this framework,
by setting the input $U_t$ to 
be a nominal input $\bar{U}_t$ injected with excitation noise injected,
i.e., $U_t = \bar{U}_t + \textrm{``}\mathrm{Noise}_t\textrm{''}$,
we can select the realization $(B_t)$ of Brownian noise and
observe the trajectories of forward SDE of the forms:
\begin{align}
    \rmd X_t &= [ f(X_t, t) + \Phi(X_t, t) \bar{U}_t ]\rmd t + g(X_t, t) \diamond \rmd B_t \label{eq:control_affine_SDE},
\end{align}
where the $\diamond$ indicates the SDE is to be interpreted in terms of either It{\^{o}} or Stratonovich, depending on the context.
To rigorously define $\textrm{``}\mathrm{Noise}_t\textrm{''}$ to establish a connection between
\eqref{eq:control_affine} and \eqref{eq:control_affine_SDE} is
technical, requiring the use of e.g., rough path theory~\cite{lyons1998}.
We will take a more practical approach inspired from \cite{wang2022modelfree} and observe how injecting Gaussian noise into discretizations of \eqref{eq:control_affine}
induces stochastic discretizations of \eqref{eq:control_affine_SDE}.

Concretely, we proceed as follows.
We work with constant step-size integrators, and define
integration times $t_k = k \tau$, $k \in \N$,
for a step-size $\tau > 0$.
Given a nominal control input $\bar{U}_t$ which is an open-loop (i.e., only time-dependent) signal,\footnote{A similar argument can also be made for state-dependent policies.} we form our 
control input $U_t$ by injecting Gaussian noise as follows:
\begin{align}
    U_t = \bar{U}_t + w_{\floor{t/\tau}} / \sqrt{\tau}, \quad w_k \sim \sfN(0, I_d). \label{eq:exploration_noise}
\end{align}
For what follows we 
define $\hat{U}_k := U_{t_k}$ and $\hat{\bar{U}}_k := \bar{U}_{t_k}$ for $k \in \N$.
We will assume that our dynamics $f, \Phi$ are continuous in both $(x, t)$, in addition to the nominal signal $\bar{U}_t$
being continuous in $t$. However, since our signal $U_t$ is discontinuous in $t$ due to the addition of the Gaussian noise, 
the resulting vector field
$F(x, t) := f(x, t) + \Phi(x, t) U_t$ is
discontinuous on $t$. Hence, we will assume that
the integration strategies will, in order to generate the $(k+1)$-th iterate given the $k$-th iterate, only evaluate the vector field $F(x, t)$ on the half-open interval $[t_k, t_{k+1})$.
In particular, we will 
interpret 
$F(x, t_k)$ using the right limit
$F(x, t_k^+) = \lim_{t \to t_k^+} F(x, t)$ and
$F(x, t_{k+1})$ using the left limit 
$F(x, t_{k+1}^-) = \lim_{t \to t_{k+1}^-} F(x, t)$.

\paragraph{Euler scheme.}
Consider the standard forward Euler scheme used to discretize
\eqref{eq:control_affine} with constant step-size $\tau$:
\begin{align*}
    \hat{X}_{k+1} &= \hat{X}_k + \tau [f(\hat{X}_k, t_k) + \Phi(\hat{X}_k, t_k) \hat{U}_k] \\
    &= \hat{X}_k + \tau [ f(\hat{X}_k, t_k) + \Phi(\hat{X}_k, t_k) \hat{\bar{U}}_k] + \sqrt{\tau} \Phi(\hat{X}_k, t_k) w_k,
\end{align*}
which corresponds to the standard Euler-Maruyama discretization of the It{\^{o}} variant of \eqref{eq:control_affine_SDE}
with $g = \Phi$.

\paragraph{Heun scheme.}
Now consider the Heun scheme used to discretize \eqref{eq:control_affine}, again with constant step-size $\tau$:
\begin{align*}
    \bar{X}_{k+1} &= \hat{X}_k + \tau [ f(\hat{X}_k, t_k) + \Phi(\hat{X}_k, t_k) \hat{U}_k ], \\
    \hat{X}_{k+1} &= \hat{X}_k + \frac{\tau}{2}\left[ f(\hat{X}_k, t_k) + \Phi(\hat{X}_k, t_k) \hat{U}_k   + f(\bar{X}_{k+1}, t_{k+1}) + \Phi(\bar{X}_{k+1}, t_{k+1}) \hat{U}_{k+1}^-   \right],
\end{align*}
where $\hat{U}_{k+1}^- = \hat{\bar{U}}_{k+1} + w_k / \sqrt{\tau}$.
Using the shorthand $\bar{F}(x, t) := f(x, t) + \Phi(x, t) \bar{U}_t$, we see that:
\begin{align*}
    \bar{X}_{k+1} &= \hat{X}_k + \tau \bar{F}(\hat{X}_k, t_k) + \sqrt{\tau} \Phi(\hat{X}_k, t_k) w_k, \\
    \hat{X}_{k+1} &= \hat{X}_k + \frac{\tau}{2}\left[ \bar{F}(\hat{X}_k, t_k) + \bar{F}(\bar{X}_{k+1}, t_{k+1}) \right] + \frac{\sqrt{\tau}}{2} \left[ \Phi(\hat{X}_k, t_k) + \Phi(\bar{X}_{k+1}, t_{k+1}) \right] w_k,
\end{align*}
which corresponds to the stochastic Heun discretization of the Stratonovich variant of \eqref{eq:control_affine_SDE}, again with $g = \Phi$.

\section{Proofs of Main Results}
\label{sec:appendix:proofs}

\subsection{Auxiliary results}

We first recall a standard formula for the variance of Gaussian quadratic forms.
\begin{myprop}
\label{prop:hutchinson_variance}
Let $Q$ be a $d \times d$ symmetric matrix and $w \sim \sfN(0, I_d)$. Then,
\begin{align*}
    \E_{w} (\Tr(Q) - w^\T Q w)^2 = 2 \norm{Q}_F^2.
\end{align*}
\end{myprop}
\begin{proof}
We have that $\E_w (\Tr(Q) - w^\T Q w)^2 = \E_w (w^\T Q w)^2 - \Tr(Q)^2$.
From \cite[cf.][Lemma 2.2]{magnus1978gaussians} we obtain
the identity $\E_w(w^\T Q w)^2 = 2 \Tr(Q^2) + \Tr(Q)^2$,
which concludes the proof.
\end{proof}

\begin{myprop}
\label{prop:stochastic_riemann_sum}
Let $(X_t)_{t=a}^{b}$ for $a \leq b$ denote an $\R^d$-valued stochastic process,
and let $r : \R^d \times [a, b] \mapsto \R$ be an $L$-Lipschitz function on its domain.
Suppose that for all $t_1, t_2 \in [a, b]$ we have
$\E \norm{X_{t_1} - X_{t_2}}^2 \leq M \abs{t_1-t_2}$. Then,
\begin{align}
    \int_{a}^{b} \E[ \abs{r(X_t, t) - r(X_a, a)} ] \rmd t \leq L \left[ M^{1/2} (b-a)^{3/2} + (b-a) ^2 \right]. \label{eq:stochastic_riemann_sum_one}
\end{align}
Furthermore, suppose that for some $0 < \tau \leq 1$, we have $N := (b-a)/\tau \in \N_+$.
Define $t_n := a + \tau n$ for $n \in \{0, \dots, N\}$.
Then we have that the left-endpoint Riemann sum satisfies:
\begin{align}
    \bigabs{ \frac{1}{b-a} \int_a^b \E[ r(X_t, t) ] \rmd t - \frac{1}{N}\sum_{n=0}^{N-1}\E[ r(X_{t_n}, t_n) ] } \leq L(1 + M^{1/2}) \tau^{1/2}.  \label{eq:stochastic_riemann_sum_full}
\end{align}
\end{myprop}
\begin{proof}
First, we have:
\begin{align*}
    \int_a^b \E[\abs{ r(X_t, t) - r(X_a, a) }] \rmd t 
    &\leq L \int_a^b \E \bignorm{\cvectwo{X_t - X_a}{t - a}} \rmd t \\
    &\leq L \int_a^b \E \norm{X_t - X_a} \rmd t + L \int_a^b (t-a) \rmd t \\
    &\stackrel{(a)}{=} L \int_a^b \E \norm{X_t - X_a} \rmd t + \frac{L (b-a)^2}{2} \\
    &\leq \frac{2LM^{1/2}}{3}(b-a)^{3/2} + \frac{L (b-a)^2}{2} ,
\end{align*}
where (a) follows by Jensen's inequality and the second moment assumption on $(X_t)_t$:
\begin{align*}
    L \int_a^b \E \norm{X_t - X_a} \rmd t \leq L \int_a^b \sqrt{\E\norm{X_t - X_a}^2} \rmd t \leq LM^{1/2} \int_a^b (t - a)^{1/2} \rmd t = \frac{2LM^{1/2}}{3}(b-a)^{3/2}.
\end{align*}
This establishes \eqref{eq:stochastic_riemann_sum_one}.
We now turn to \eqref{eq:stochastic_riemann_sum_full}.
We write:
\begin{align*}
    \bigabs{\int_a^b \E[r(X_t, t)] \rmd t - \tau \sum_{n=0}^{N-1} \E[ r(X_{t_n}, t_n) ]} &= \bigabs{\sum_{n=0}^{N-1} \int_{t_n}^{t_{n+1}} \E[ r(X_t, t) ] \rmd t - \tau \sum_{n=0}^{N-1} \E[ r(X_{t_n}, t_n) ]} \\
    &= \bigabs{\sum_{n=0}^{N-1} \left( \int_{t_n}^{t_{n+1}} \E[ r(X_t, t) - r(X_{t_n}, t_n) ] \rmd t \right) } \\
    &\leq \sum_{n=0}^{N-1} \int_{t_n}^{t_{n+1}} \E[ \abs{ r(X_t, t) - r(X_{t_n}, t_n) } ] \rmd t \\
    &\stackrel{(a)}{\leq} N L \left[ M^{1/2} \tau^{3/2} + \tau^2 \right] \\
    &\stackrel{(b)}{\leq} N L (M^{1/2} + 1) \tau^{3/2} = (b-a) L (M^{1/2} + 1) \tau^{1/2},
\end{align*}
where (a) is from \eqref{eq:stochastic_riemann_sum_one}
and (b) is from the assumption that $\tau \leq 1$.
This establishes \eqref{eq:stochastic_riemann_sum_full}.
\end{proof}

\begin{myprop}
\label{prop:stochastic_riemann_squared}
Let $(X_t)_{t=a}^{b}$ for $a \leq b$ 
denote an $\R^d$-valued stochastic process,
and let $r : \R^d \times [a, b] \mapsto \R$ be a $B$-bounded and 
$L$-Lipschitz function on its domain.
Suppose that for all $t_1, t_2 \in [a, b]$ we have
$\E \norm{X_{t_1} - X_{t_2}}^2 \leq M \abs{t_1-t_2}$, with $M \geq 1$. Then, if $b-a \leq 1$, 
\begin{align}
    \bigabs{\E\left( \int_a^b r(X_t, t) \rmd t \right)^2 - (b-a)^2 \E[ r^2(X_a, a) ]} \leq \left[ L^2 M + 4 B L M^{1/2} \right] (b-a)^{5/2}.
\end{align}
\end{myprop}
\begin{proof}
We first decompose:
\begin{align*}
    \E\left(\int_{a}^{b} r(X_t, t) \rmd t\right)^2 &= \E\left( (b-a) r(X_{a}, a) + \int_{a}^{b} [ r(X_t, t) - r(X_{a}, a) ] \rmd t \right)^2 \\
    &= (b-a)^2 \E[ r^2(X_{a}, a) ] + \E\left( \int_{a}^{b} [ r(X_t, t) - r(X_{a}, a) ] \rmd t \right)^2 \\
    &\qquad + 2 (b-a) \E\left[  r(X_{a}, a) \int_{a}^{b} [ r(X_t, t) - r(X_{a}, a) ] \rmd t \right] .
\end{align*}
Next, by Jensen's inequality,
\begin{align*}
    \E\left( \int_{a}^{b} [ r(X_t, t) - r(X_{a}, a) ] \rmd t \right)^2 &\leq (b-a) \int_a^b \E[ (r(X_t, t) - r(X_{a}, a) )^2 ] \rmd t \\
    &\leq (b-a) L^2 \int_a^b (\E[ \norm{X_t - X_a}^2 ] + (t-a)^2) \rmd t \\
    &\leq (b-a) L^2 \int_a^b (M (t-a) + (t-a)^2) \rmd t 
    \stackrel{(a)}{\leq} L^2 M (b-a)^3,
\end{align*}
where in (a) we use the assumptions that $M \geq 1$ and $b-a \leq 1$.
On the other hand, by another application of Jensen's inequality,
\begin{align*}
    2(b-a)\bigabs{\E\left[  r(X_{a}, a) \int_{a}^{b} [ r(X_t, t) - r(X_{a}, a) ] \rmd t \right]} &\leq 2(b-a) B \int_a^b \E[ \abs{ r(X_t, t) - r(X_a, a) } ] \rmd t \\
    &\stackrel{(a)}{\leq} 2B L \left[ M^{1/2} (b-a)^{5/2} + (b-a)^3 \right] \\
    &\stackrel{(b)}{\leq} 4BL M^{1/2} (b-a)^{5/2},
\end{align*}
where (a) uses \Cref{prop:stochastic_riemann_sum}, specifically \eqref{eq:stochastic_riemann_sum_one},
and (b) uses the assumptions that $M \geq 1$ and $b-a \leq 1$.
The claim now follows.
\end{proof}

\begin{myprop}
\label{prop:SDE_holder_one_half}
Consider the It{\^{o}} SDE $(X_t)_{t=a}^{b}$
defined by $\rmd X_t = f(X_t, t) \rmd t + g(X_t, t) \rmd B_t$,
with 
\begin{align*}
    \sup_{(x, t) \in \R^d \times [a, b]} \max\{\norm{f(x, t)}, \norm{g(x, t)}_F \} \leq B.
\end{align*}
For any $t_0, t_1 \in [a, b]$, 
\begin{align*}
    \E\norm{X_{t_1} - X_{t_0}}^2 \leq 2B^2 \left[ (t_1 - t_0)^2 + \abs{t_1 - t_0} \right].
\end{align*}
\end{myprop}
\begin{proof}
Assume wlog that $t_1 \geq t_0$.
We first decompose:
\begin{align*}
    \E \norm{ X_{t_1} - X_{t_0} }^2 &= \E \bignorm{ \int_{t_0}^{t_1} f(X_t, t) \rmd t + \int_{t_0}^{t_1} g(X_t, t) \rmd B_t }^2 \\
    &\leq 2 \E \bignorm{ \int_{t_0}^{t_1} f(X_t, t) \rmd t }^2 + 2 \E \bignorm{ \int_{t_0}^{t_1} g(X_t, t) \rmd B_t  }^2.
\end{align*}
Next, we have:
\begin{align*}
    \bignorm{ \int_{t_0}^{t_1} f(X_t, t) \rmd t } \leq \int_{t_0}^{t_1} \norm{ f(X_t, t) } \rmd t \leq (t_1 - t_0) B \Longrightarrow \E \bignorm{ \int_{t_0}^{t_1} f(X_t, t) \rmd t }^2 \leq B^2 (t_1-t_0)^2.
\end{align*}
On the other hand, by It{\^{o}} isometry,
\begin{align*}
    \E \bignorm{ \int_{t_0}^{t_1} g(X_t, t) \rmd B_t  }^2 &= \int_{t_0}^{t_1} \E\norm{ g(X_t, t) }_F^2 \rmd t \leq B^2 (t_1 - t_0).
\end{align*}
The result now follows
\end{proof}

\begin{myprop}
\label{prop:strat_SDE_holder_one_half}
Consider the Stratonovich SDE $(\Xs_t)_{t=a}^{b}$
defined by $\rmd \Xs_t = f(\Xs_t, t) + g(\Xs_t, t) \circ \rmd B_t$.
Suppose that for $t \in [a, b]$, the map $x \mapsto g(x, t)$ is $C^1$, and that
\begin{align*}
    \sup_{(x, t) \in \R^d \times [a, b]} \max\left\{ \bignorm{ f(x, t) + \frac{1}{2}\sum_{k=1}^{m} \partial_x g^k(x, t) g(x, t)}, \norm{ g(x, t) }_F \right\} \leq B.
\end{align*}
Then for any $t_0, t_1 \in [a, b]$,
\begin{align*}
    \E\norm{\Xs_{t_1} - \Xs_{t_0}}^2 \leq 2B^2 \left[ (t_1 - t_0)^2 + \abs{t_1 - t_0} \right].
\end{align*}
\end{myprop}
\begin{proof}
Consider $\bar{f}(x, t) := f(x, t) + \frac{1}{2}\sum_{k=1}^{m} \partial_x g^k(x, t) g(x, t)$.
The It{\^{o}} SDE $\rmd X_t = \bar{f}(X_t, t) \rmd t + g(X_t, t) \rmd B_t$ is pathwise equivalent to $(\Xs_t)_t$, i.e., $(\Xs_t(\omega))_t = (X_t(\omega))_t$
for a.e.\ $\omega$, and hence the result follows from 
\Cref{prop:SDE_holder_one_half}.
\end{proof}

\subsection{Proofs of \Cref{lemma:euler_maruyama_bsde} and \Cref{thm:EM_final_form}}
\label{sec:appendix:proofs:euler_maruyama_bsde}

We first restate and prove  \Cref{lemma:euler_maruyama_bsde}.
\EMthm*
\begin{proof}
We first introduce two pieces of notation:
$O(\cdot)$ and $O^*(\cdot)$.
The former $O(\cdot)$ hides constants that depend 
arbitrarily on the dimension $d$, the bounds on $f$ and $g$,
and $\norm{u_\theta}_{C^{2,1}}$,
whereas the latter $O^*(\cdot)$ in addition also
hides constants that depend \emph{polynomially} on $\norm{w}$.
The latter polynomial dependence is important 
when we take expectations of powers of $O^*(\cdot)$ terms,
since $\E \norm{w}^p$ is finite for any finite $p \in \N$.

Setting $\bar{\Delta} := (\hat{x}_{t+\tau} - x, \tau) \in \R^d \times \calI$ and writing $u = u_\theta$, a second-order Taylor expansion yields:
\begin{align*}
    u(\hat{x}_{t+\tau}, t + \tau) - u(x, t) 
    &= Du(x, t) \bar{\Delta} + \frac{1}{2} \bar{\Delta}^\T D^2 u(x, t) \bar{\Delta} + O(\norm{\bar{\Delta}}^3), \\
    D u(x, t) \bar{\Delta} &= \ip{\nabla u(x, t)}{ \hat{x}_{t+\tau} - x } + \partial_t u(x, t) \tau, \\
    \frac{1}{2} \bar{\Delta}^\T D^2 u(x, t) \bar{\Delta} &= \frac{1}{2}\big(  (\hat{x}_{t+\tau} - x)^\T \nabla^2 u(x, t) (\hat{x}_{t+\tau} - x) \\
    &\qquad + 2 \tau \ip{\partial_t \nabla u(x, t)}{\hat{x}_{t+\tau} - x} + \tau^2 \partial^2_{t} u(x, t) \big).
\end{align*}
Plugging in $\hat{x}_{t+\tau} - x = f(x, t) \tau + \sqrt{\tau} g(x, t) w$, we obtain:
\begin{align*}
    &(u(\hat{x}_{t+\tau}, t + \tau) - u(x, t)) - (h(x, t) \tau - \sqrt{\tau} \ip{\nabla u(x, t)}{g(x, t) w}) \\
    &= \tau \left[ \ip{\nabla u(x, t)}{f(x, t)} + \partial_t u(x, t) - h(x, t) + \frac{1}{2} w^\T g(x, t)^\T \nabla^2 u(x, t) g(x, t) w  \right] + O^*(\tau^{3/2}) \\
    &= \tau \left[ R[u](x, t) - \frac{1}{2} \Tr(H(x, t) \cdot \nabla^2 u(x, t)) + \frac{1}{2} w^\T g(x, t)^\T \nabla^2 u(x, t) g(x, t) w  \right] + O^*(\tau^{3/2}),
\end{align*}
where in the last equality we used the definition
of the PDE residual from \eqref{eq:main_PDE}.
Hence,
\begin{align*}
    &\ell_{\mathrm{EM},\tau}(x, t) \\
    &= \E_w (u(\hat{x}_{t+\tau}, t + \tau) - u(x, t)) - (h(x, t) \tau - \sqrt{\tau} \ip{\nabla u(x, t)}{g(x, t) w})^2 \\
    &= \tau^2 \cdot \E_w \left(R[u](x, t) - \frac{1}{2} \Tr(H(x, t) \cdot \nabla^2 u(x, t)) + \frac{1}{2} w^\T g(x, t)^\T \nabla^2 u(x, t) g(x, t) w\right)^2 + O(\tau^{5/2}) \\
    &= \tau^2 \left( (R[u](x, t))^2 + \frac{1}{4} \E_w\left( \Tr(H(x, t) \cdot \nabla^2 u(x, t)) - w^\T g(x, t)^\T \nabla^2 u(x, t) g(x, t) w \right)^2   \right) + O(\tau^{5/2}) \\
    &= \tau^2 \left( (R[u](x, t))^2 + \frac{1}{2}  \Tr( (H(x, t) \cdot \nabla^2 u(x, t))^2 ) \right) + O(\tau^{5/2}),
\end{align*}
where the final equality follows from \Cref{prop:hutchinson_variance}.
The claim now follows.
\end{proof}

Next, we use \Cref{lemma:euler_maruyama_bsde}, along with
order $1/2$ strong convergence for EM integration (\Cref{sec:appendix:SDE_convergence})
to show the following result.
\EMfinalform*
\begin{proof}
To start, we have that:
\begin{align*}
    L_{\mathrm{EM},\tau}(\theta) &= \frac{1}{N} \sum_{n=0}^{N-1} \E_{\hat{X}_n}[ \tau^{-2} \cdot \ell_{\mathrm{EM,\tau}}(\theta, \hat{X}_n, t_n) ] \\
    &\stackrel{(a)}{=} \frac{1}{N} \sum_{n=0}^{N-1} \E_{\hat{X}_n}\Big[ \underbrace{(R[u_\theta](\hat{X}_n, t_n))^2 + \frac{1}{2} \Tr\left[ (H(\hat{X}_n, t_n) \cdot \nabla^2 u_\theta (\hat{X}_n, t_n))^2 \right]}_{=: \bar{R}_\theta(\hat{X}_n, t_n)} \Big] + O(\tau^{1/2}),
\end{align*}
where (a) comes from \Cref{lemma:euler_maruyama_bsde}
which holds since (i) $f,g \in C^{0,1}$ implies $f,g$ are bounded, and
(ii) $u_\theta \in C^{2,1}$.
Our next observation is that the map $\bar{R}_\theta$ is also Lipschitz
continuous over the domain $\R^d \times \calI$ by our assumptions $f,g,h[u_\theta] \in C^{0,1}$ and $u_\theta \in C^{2,1}$.
Let us call this Lipschitz constant $L_{\bar{R}_\theta}$, 
which depends only on 
the norms $\norm{f}_{C^{0,1}}, \norm{g}_{C^{0,1}}, \norm{h[u_\theta]}_{C^{0,1}}, \norm{u_\theta}_{C^{2,1}}$.
Continuing from above,
\begin{align*}
    \E_{\hat{X}_n}[ \bar{R}_\theta(\hat{X}_n, t_n) ] &\stackrel{(a)}{=} \E_{(B_t)_{t}}[ \bar{R}_\theta(\hat{X}_n, t_n)  ] \\
    &\stackrel{(b)}{=} \E_{(B_t)_{t}}[ \bar{R}_\theta(X_{t_n}, t_n) ] + \E_{(B_t)_{t}}[ \bar{R}_\theta(\hat{X}_n, t_n) - \bar{R}_\theta(X_{t_n}, t_n) ],
\end{align*}
where in (a) we consider the process $\{\hat{X}_n\}_n$ as being defined over
$\{\Delta W_n\}_{n} := \{B_{t_{n+1}} - B_{t_n}\}_{n}$ in place of the
process $\{\sqrt{\tau} w_n\}_{n}$ in \eqref{eq:backwards_SDE_EM}, which is distributionally equivalent,
in (b) we consider the forward SDE $(X_t)_t$ from \eqref{eq:forward_SDE}
as being \emph{coupled} with the process $\{\hat{X}_n\}_n$ via the same
realization of both Brownian motion $(B_t)_t$ and $X_0 = \hat{X}_0 = x_0$.
Hence we have:
\begin{align*}
    \bigabs{ \E_{\hat{X}_n}[ \bar{R}_\theta(\hat{X}_n, t_n) ] - \E_{(B_t)_{t}}[ \bar{R}_\theta(\hat{X}_n, t_n)  ] } &\leq \E_{(B_t)_t}[\abs{ \bar{R}_\theta(\hat{X}_n, t_n) - \bar{R}_\theta(X_{t_n}, t_n) } ] \\
    &\leq L_{\bar{R}_\theta} \E_{(B_t)_t}[ \norm{ \hat{X}_n - X_{t_n} } ] \\
    &\leq L_{\bar{R}_\theta} \sqrt{\E_{(B_t)_t}[ \norm{ \hat{X}_n - X_{t_n} }^2 ] }.
\end{align*}
Now, by definition, since the functions $f, g \in C^{0, 1}$, then the pair $(f, g)$ is EM-regular (\Cref{def:EM_regularity}).
From \Cref{thm:EM_strong_order_one_half}, we have that $\{\hat{X}_n\}_n$ is strong order $1/2$ convergent towards $(X_t)_t$, and hence
$\E_{(B_t)_t}[ \norm{ \hat{X}_n - X_{t_n} }^2 ] \leq \E_{(B_t)_t}[ \max_{n \in \{0, \dots, N\}} \norm{ \hat{X}_n - X_{t_n} }^2 ] \leq C^2 \tau$, where the constant
$C$ does not depend on $\tau$, but can depend on $d$, $T$, and the
norms $\norm{f}_{C^{0,1}}$ and $\norm{g}_{C^{0,1}}$.
Consequently, 
\begin{align}
    L_{\mathrm{EM},\tau}(\theta) &= \frac{1}{N} \sum_{n=0}^{N-1} \E_{(B_t)_t}[ \bar{R}_\theta(X_{t_n}, t_n) ] + O( \tau^{1/2} ), \label{eq:EM_BSDE_equiv_with_sum}
\end{align}
Our last step is to approximate the sum $\frac{1}{N} \sum_{n=0}^{N-1}$ in \eqref{eq:EM_BSDE_equiv_with_sum}
with the integral.
To do this, we will use \Cref{prop:stochastic_riemann_sum}.
We already have $\bar{R}_\theta$ is Lipschitz over $\R^d \times \calI$.
Furthermore, since $f, g \in C^{0,1}$, they are both bounded over the domain,
and hence \Cref{prop:SDE_holder_one_half}
shows that $\E \norm{X_{t_1} - X_{t_0}}^2 \leq O(1) \abs{t_1 - t_0}$ for any $t_0, t_1 \in \calI$.
Therefore by \Cref{prop:stochastic_riemann_sum} and the assumption $\tau \leq 1$,
\begin{align*}
    \frac{1}{T} \int_0^T \E[ \bar{R}_\theta(X_t, t) ] \rmd t 
    &= \frac{1}{N} \sum_{n=0}^{N-1} \E[ \bar{R}_\theta(X_{t_n}, t_n) ] + O(\tau^{1/2}) 
    \stackrel{(a)}{=} L_{\mathrm{EM},\tau}(\theta) + O(\tau^{1/2}),
\end{align*}
where (a) is from \eqref{eq:EM_BSDE_equiv_with_sum}.
The result now follows.
\end{proof}

\subsection{Proofs of \Cref{lemma:heun_bsde} and \Cref{thm:heun_final_form}}
\label{sec:appendix:proofs:heun_bsde}

We now restate and prove \Cref{lemma:heun_bsde}.
\Heunthm*
\begin{proof}
Similar to the proof of \Cref{lemma:euler_maruyama_bsde},
we let $O(\cdot)$
hide constants that depend on $d$ and the H{\"{o}}lder norms
$\norm{f}_{C^{1,1}}$, $\norm{g}_{C^{1,1}}$, $\norm{h_\theta}_{C^{1,1}}$, and
$\norm{u_\theta}_{C^{3,1}}$,
and $O^*(\cdot)$ additionally
hides constants that depend \emph{polynomially} on $\norm{w}$.

Our first step is to check that $\hs_\theta \in C^{1,1}$ under our assumptions.
Recalling that $\hs_\theta(x, t) = h_\theta(x, t) - \frac{1}{2} \Tr( H(x, t) \nabla^2 u_\theta(x, t))$, this is ensured if $h_\theta, H \in C^{1,1}$ and $u_\theta \in C^{3,1}$,
which holds since $g \in C^{1,1}$ implies $H \in C^{1,1}$.
Furthermore, $\norm{ \hs_\theta }_{C^{1,1}} = O(1)$.
For what follows, we drop the dependency in the notation on $\theta$.

We next Taylor expand $\hat{z}_{t+\tau} - z_t$ up to order $\tau$ terms.
To do this, we observe that
by our assumptions on $f, g, \hs, u$, the functions $F, G$ are both in $C^{1,1}$.
Hence,
\begin{align*}
    F(\bar{z}_{t+\tau}, t+\tau) &= F(z_t, t) + D_Z F(z_t, t)[\bar{z}_{t+\tau} - z_t] + \partial_t F(z_t, t) \tau +  O(\norm{\bar{z}_{t+\tau} - z_t}^2) + O(\tau^2) \\
    &= F(z_t, t) + D_Z F(z_t, t)[G(z_t, t) w] \sqrt{\tau} + O^*(\tau).
\end{align*}

By a similar argument,
\begin{align*}
    G(\bar{z}_{t+\tau}, t+\tau) &= G(z_t, t) + D_Z G(z_t, t)[G(z_t, t) w] \sqrt{\tau} + O^*(\tau).
\end{align*}
Hence,
\begin{align}
    \hat{z}_{t+\tau} - z_t &= \left[ F(z_t, t) + \frac{1}{2} D_Z G(z_t, t)[G(z_t, t)w]w \right] \tau + G(z_t, t) w \sqrt{\tau} + O^*(\tau^{3/2}). \label{eq:heun_stochastic_taylor}
\end{align}
A straightforward computation yields:
\begin{align*}
    D_Z G((x, y), t)[(\Delta_x, \Delta_y)] &= \cvectwo{ Dg(x, t)[\Delta_x] }{ \nabla u(x, t)^\T Dg(x, t)[\Delta_x] + \Delta_x^\T \nabla^2 u(x, t) g(x, t) },
\end{align*}
and therefore:
\begin{align*}
    \frac{1}{2} D_Z G(z_t, t)[G(z_t, t)w]w = \frac{1}{2} \cvectwo{  Dg(x, t)[ g(x, t)w ] w  }{ \nabla u(x, t)^\T Dg(x, t)[g(x, t) w] w + w^\T g(x, t)^\T \nabla^2 u(x, t) g(x, t) w  }.
\end{align*}
Substituting the above into expression \eqref{eq:heun_stochastic_taylor} 
for $\hat{z}_{t+\tau} - z_t$,
\begin{align}
    \hat{z}_{t+\tau} - z_t &= \cvectwo{f(x, t)}{h(x, t) - \frac{1}{2}\Tr(H(x, t) \nabla^2 u(x, t))} \tau \nonumber \\
    &\qquad+ \frac{1}{2} \cvectwo{  Dg(x, t)[ g(x, t)w ] w  }{ \nabla u(x, t)^\T Dg(x, t)[g(x, t) w] w + w^\T g(x, t)^\T \nabla^2 u(x, t) g(x, t) w  } \tau \nonumber \\
    &\qquad+ \sqrt{\tau} \cvectwo{g(x, t) w}{\nabla u(x, t)^\T g(x, t) w} + O^*(\tau^{3/2}). \label{eq:z_hat_next_minus_z}
\end{align}
Setting $\bar{\Delta} := (\hat{x}_{t+\tau} - x, \tau) \in \R^{d + 1}$ we have:
\begin{align*}
    u(\hat{x}_{t+\tau}, t + \tau) - u(x, t) 
    &= Du(x, t) \bar{\Delta} + \frac{1}{2} \bar{\Delta}^\T D^2 u(x, t) \bar{\Delta} + O(\norm{\bar{\Delta}}^3), \\
    D u(x, t) \bar{\Delta} &= \ip{\nabla u(x, t)}{ \hat{x}_{t+\tau} - x } + \partial_t u(x, t) \tau, \\
    \frac{1}{2} \bar{\Delta}^\T D^2 u(x, t) \bar{\Delta} &= \frac{1}{2}\big(  (\hat{x}_{t+\tau} - x)^\T \nabla^2 u(x, t) (\hat{x}_{t+\tau} - x) \\
    &\qquad + 2 \tau \ip{\partial_t \nabla u(x, t)}{\hat{x}_{t+\tau} - x} + \tau^2 \partial^2_{t} u(x, t) \big),
\end{align*}
from which we conclude,
\begin{align*}
    &u(\hat{x}_{t+\tau}, t+\tau) - u(x, t) \\
    &= \ip{\nabla u(x, t)}{\hat{x}_{t+\tau} - x} + \partial_t u(x,t ) \tau +  \frac{1}{2} (\hat{x}_{t+\tau}-x)^\T \nabla^2 u(x, t) (\hat{x}_{t+\tau}-x) + O(\tau^{3/2}) \\
    &= \left[ \ip{\nabla u(x, t)}{f(x, t) + \frac{1}{2} Dg(x, t)[g(x, t)w]w } + \partial_t u(x, t)  + \frac{1}{2} w^\T g(x, t)^\T \nabla^2 u(x, t) g(x, t) w \right] \tau \\
    &\qquad + \sqrt{\tau} \ip{\nabla u(x, t)}{g(x, t) w} + O^*(\tau^{3/2}).
\end{align*}
On the other hand, from \eqref{eq:z_hat_next_minus_z},
\begin{align*}
    \hat{y}_{t+\tau} - y_t &= \left[ h(x, t) - \frac{1}{2} \Tr( H(x, t) \nabla^2 u(x, t)) \right] \tau \\
    &\qquad + \frac{1}{2} \left[\nabla u(x, t)^\T D g(x, t)[g(x, t)w] w + \ w^\T g(x, t)^\T \nabla^2 u(x, t) g(x, t) w \right] \tau \\
    &\qquad + \sqrt{\tau} \nabla u(x, t)^\T g(x, t) w + O^*(\tau^{3/2}).
\end{align*}
Hence,
\begin{align*}
    &u(\hat{x}_{t+\tau}, t+\tau) - \hat{y}_{t+\tau} \\
    &= (u(\hat{x}_{t+\tau}, t+\tau) - u(x, t)) - (\hat{y}_{t+\tau} - y_t) \\ 
    &= \left[ \ip{\nabla u(x, t)}{f(x, t)} + \partial_t u(x, t) + \frac{1}{2} \Tr(H(x, t) \nabla^2 u(x, t)) - h(x, t) \right] \tau + O^*(\tau^{3/2}) \\
    &= R[u](x, t) \tau + O^*(\tau^{3/2}).
\end{align*}
To conclude,
\begin{align*}
    \E_w( u(\hat{x}_{t+\tau}, t+\tau) - \hat{y}_{t+\tau})^2 = \E_w ( R[u](x, t) \tau + O^*(\tau^{3/2}) )^2 = (R[u](x, t))^2 \tau^2 + O( \tau^{5/2} ).
\end{align*}
\end{proof}

\Heunfinalform*
\begin{proof}
The proof follows the structure of \Cref{thm:EM_final_form} closely. We start with:
\begin{align*}
    L_{\mathrm{Heun},\tau}(\theta) = \frac{1}{N}\sum_{n=0}^{N-1} \E_{\Xhats_n}[ \tau^{-2} \cdot \ell_{\mathrm{Heun},\tau}(\theta, \Xhats_n, t_n) ] \stackrel{(a)}{=} \frac{1}{N}\sum_{n=0}^{N-1} \E_{\Xhats_n}[ (R[u_\theta](\Xhats_n, t_n))^2 ] + O(\tau^{1/2})
\end{align*}
where (a) follows from \Cref{lemma:heun_bsde}.
Next, we define $\bar{R}_\theta(x, t) := (R[u_\theta](x, t))^2$, and
observe that $\bar{R}_\theta$ is Lipschitz over $\R^d \times \calI$ due to our assumptions on
$f, g, h_\theta, u_\theta$.
Hence, we have the decomposition:
\begin{align*}
    \E_{\Xhats_n}[ \bar{R}_\theta(\Xhats_n, t_n) ] &\stackrel{(a)}{=} \E_{(B_t)_t}[ \bar{R}_\theta(\Xhats_n, t_n) ] \\
    &\stackrel{(b)}{=} \E_{(B_t)_t}[ \bar{R}_\theta(\Xs_{t_n}, t_n) ] + \E_{(B_t)_t}[ \bar{R}_\theta(\Xhats_{t_n}, t_n) - \bar{R}_\theta(\Xs_n, t_n) ],
\end{align*}
where in (a) and (b) we take same steps as \Cref{thm:EM_final_form}:
(a) considers the process $\{\Xhats_n\}_n$ as being defined over $\{ \Delta W_n \}_n$ (Brownian increments) in place of $\{ \sqrt{\tau} w_n \}_n$ in \eqref{eq:heun_scheme}, 
and (b) couples the SDE $(\Xs_t)_t$ together with $\{ \Xhats_n \}_n$ via the
same Brownian motion $(B_t)_t$ and initial condition $\Xs_0 = \Xhats_0 = x_0$.
Hence, we have:
\begin{align*}
    \abs{ \E_{\Xhats_n}[ \bar{R}_\theta(\Xhats_n, t_n) ] - \E_{(B_t)_t}[ \bar{R}_\theta(\Xs_{t_n}, t_n) ] } \leq O(1) \sqrt{ \E_{(B_t)_t}[ \norm{ \Xs_{t_m} - \Xhats_n }^2 ] }.
\end{align*}
Since $f,g \in C^{1,1}$, then they are by definition Heun-regular (\Cref{def:heun_regularity}), and hence by \Cref{corollary:heun_strong_convergence},
$\E_{(B_t)_t} [\norm{ \Xs_{t_n} - \Xhats_n }^2] \leq \E_{(B_t)_t}[ \max_{n \in \{0, \dots, N\}} \norm{ \Xs_{t_n} - \Xhats_n }^2 ] \leq C^2 \tau$, where $C$ does not depend on $\tau$ 
but depends on $d$, $T$, and the H{\"{o}}lder norms on $f$, $g$.
Therefore we have:
\begin{align}
    L_{\mathrm{Heun},\tau}(\theta) = \frac{1}{N} \sum_{n=0}^{N-1} \E_{(B_t)_t}[ (R[u_\theta](\Xs_{t_n}, t_n))^2 ] + O(\tau^{1/2}). \label{eq:heun_BSDE_equiv_with_sum}
\end{align}
Now we can finish up using the same ending as \Cref{thm:EM_final_form}.
The only thing we need to do differently is to control
the second moment $\E\norm{ \Xs_{t_1} - \Xs_{t_0} }^2$.
Since $f, g \in C^{1,1}$, \Cref{prop:strat_SDE_holder_one_half} yields that 
$\E \norm{ \Xs_{t_1} - \Xs_{t_0} }^2 \leq O(1) \abs{t_1-t_0}$.
From this inequality and the Lipschitz continuity of $\bar{R}_\theta$ over the
domain $\R^d \times \calI$, by \Cref{prop:stochastic_riemann_sum}
and the assumption $\tau \leq 1$,
\begin{align*}
    \frac{1}{T} \int_0^T \E_{(B_t)_t}[ (R[u_\theta](\Xs_t, t))^2 ] \rmd t &= \frac{1}{N} \sum_{n=0}^{N-1} \E_{(B_t)_t}[ (R[u_\theta](\Xs_{t_n}, t_n))^2 ] + O(\tau^{1/2}) \\
    &\stackrel{(a)}{=} L_{\mathrm{Heun},\tau}(\theta) + O(\tau^{1/2}),
\end{align*}
where (a) is from \eqref{eq:heun_BSDE_equiv_with_sum}. The result now follows.
\end{proof}

\section{Analysis of Multi-Step BSDE Losses}
\label{sec:appendix:long_horizon_BSDE_analysis}

We now present the derivations supporting the analysis in \Cref{sec:long_horizon_BSDE_analysis}.
For what follows, we define the follow FS-PINNs loss:
\begin{align}
     L_{\FPINNs}(\theta) := \E_{\substack{x_0 \sim \mu_0, B_t}} \frac{1}{T} \int_{0}^{T} (R[u_\theta](X_t, t))^2 \rmd t. \label{eq:FS_PINNS_loss}
\end{align}

\subsection{BSDE loss and Euler-Maruyama discretization}

\begin{myprop}
\label{prop:BSDE_jensen_bound}
Suppose that $u_\theta \in C^2$. We have that:
\begin{align}
    L_{\mathrm{BSDE},T}(\theta) \leq L_{\FPINNs}(\theta). \label{eq:BSDE_jensen_bound}
\end{align}
\end{myprop}
\begin{proof}
To start, we abbreviate $R_\theta(x, t) = R[u_\theta](x, t)$.
By application of It{\^{o}}'s Lemma, we have:
\begin{align*}
    u_\theta(X_{T}, T) - u_\theta(X_{0}, 0) - \int_{0}^{T} h_\theta(X_t, t) \rmd t - \int_{0}^{T} \nabla u_\theta(X_t, t)^\T g(X_t, t) \rmd B_t = \int_{0}^{T} R_\theta(X_t, t) \rmd t,
\end{align*}
which immediately yields the following identity:
\begin{align}
    L_{\mathrm{BSDE},T}(\theta) = \E_{\substack{x_0 \sim \mu_0, B_t}} \left( \frac{1}{T} \int_{0}^{T} R_\theta(X_t, t) \rmd t  \right)^2. \label{eq:BSDE_identity}
\end{align}
Thus, the BSDE loss is equal to averaging the square of the accumulation of the residual error $R_\theta$ along the forward SDE trajectories. 
Hence by Jensen's inequality, the BSDE loss is dominated by:
\begin{align*}
    L_{\mathrm{BSDE},T}(\theta) \leq \E_{\substack{x_0 \sim \mu_0, B_t}} \frac{1}{T} \int_{0}^{T} (R_\theta(X_t, t))^2 \rmd t = L_{\FPINNs}(\theta).
\end{align*}
\end{proof}

Note that while the relationship in \eqref{eq:BSDE_jensen_bound} is pointed out in 
\cite[Section 5.2.3]{nusken2023interpolating}, 
the implications of this inequality are not discussed further in their work.

\begin{myprop}
\label{prop:BSDE_FPINNs}
Suppose that $f, g, h_\theta \in C^{0, 1}$, $u_\theta \in C^{2,1}$, and $\tau \leq 1$.
Then,
\begin{align}
    L_{\mathrm{BSDE},\tau}(\theta) &= L_{\FPINNs}(\theta) + O(\tau^{1/2}), \label{eq:BSDE_FPINNs}
\end{align}
where the $O(\cdot)$ hides constants depending on the H{\"{o}}lder norms of $f$, $g$, $h_\theta$, and $u_\theta$.
\end{myprop}
\begin{proof}
Again, we abbreviate $R_\theta(x, t) = R[u_\theta](x, t)$.
By our assumptions on $f, g, h_\theta \in C^{0, 1}$ and $u_\theta \in C^{2,1}$, we have
that $R_\theta \in C^{0, 1}$, with $\norm{R_\theta}_{C^{0, 1}} = O(1)$.
Also, since $f, g \in C^{0, 1}$, by \Cref{prop:SDE_holder_one_half}
we also have that $\E \norm{ X_{t_0} - X_{t_1} }^2 \leq O(1) \abs{t_0 - t_1}$ 
for $t_0, t_1 \in \calI$.
Hence by \Cref{prop:stochastic_riemann_squared},
\begin{align}
    \E\left(\int_{t_n}^{t_{n+1}} R_\theta(X_t, t) \rmd t\right)^2 &=  \tau^2 \E[ R_\theta^2(X_{t_n}, t_n) ] + O(\tau^{5/2}). \label{eq:one_term_sq_expansion}
\end{align}
Furthermore, since $R_\theta \in C^{0, 1}$, 
we can readily check that $R_\theta^2$ is Lipschitz on its domain as well,
and hence \Cref{prop:stochastic_riemann_sum} yield:
\begin{align}
    \frac{1}{T} \int_0^T \E[(R_\theta(X_t, t))^2] \rmd t = \frac{1}{N}\sum_{n=0}^{N-1} \E[ (R_\theta(X_{t_n}, t_n))^2 ] + O(\tau^{1/2}). \label{eq:one_term_expansion}
\end{align}

Therefore,
\begin{align*}
    L_{\mathrm{BSDE},\tau}(\theta) &= \E_{x_0 \sim \mu_0, B_t} \frac{1}{N\tau^2}\sum_{n=0}^{N-1} \left( \int_{t_n}^{t_{n+1}} R_\theta(X_t, t) \rmd t \right)^2 \nonumber \\
    &= \E_{x_0 \sim \mu_0, B_t} \frac{1}{N\tau^2}\sum_{n=0}^{N-1} \left(\tau^2 R_\theta^2(X_{t_n}, t_n) + O(\tau^{5/2})\right) &&[\text{using \eqref{eq:one_term_sq_expansion}}] \nonumber \\
    &= \E_{x_0 \sim \mu_0, B_t} \frac{1}{N}\sum_{n=0}^{N-1} R_\theta^2(X_{t_n}, t_n) + O(\tau^{1/2}) \nonumber \\
    &= \E_{x_0 \sim \mu_0, B_t} \frac{1}{T} \int_0^T (R_\theta(X_t, t))^2 \rmd t + O(\tau^{1/2}) &&[\text{using \eqref{eq:one_term_expansion}}]  \nonumber \\
    &= L_{\FPINNs}(\theta) + O(\tau^{1/2}).
\end{align*}
\end{proof}

\begin{myprop}
\label{prop:BSDE_T_vs_tau}
Suppose the assumptions of \Cref{prop:BSDE_FPINNs} hold. Then,
\begin{align}
    L_{\mathrm{BSDE},T}(\theta) \leq L_{\mathrm{BSDE},\tau}(\theta) + O(\tau^{1/2}), \label{eq:BSDE_jensen_relationship}
\end{align}
where the $O(\cdot)$ hides constants depending on the H{\^{o}}lder norms of $f$, $g$, $h_\theta$, and $u_\theta$.
\end{myprop}
\begin{proof}
Follows immediately from \Cref{prop:BSDE_jensen_bound} and \Cref{prop:BSDE_FPINNs}.
\end{proof}

\begin{myprop}
\label{prop:EM_N_tau_vs_BSDE_T}
Suppose that $f, g, h_\theta \in C^{0,1}$, $u_\theta \in C^{1,1}$, and $\tau \leq 1$.
We have that:
\begin{align}
    L_{\mathrm{EM}_N,\tau}(\theta) &= L_{\mathrm{BSDE},T}(\theta) + O(\tau^{1/2}),
\end{align}
where the $O(\cdot)$ hides constants depending on the $d$, $T$, and the H{\"{o}}lder norms of $f$, $g$, $h_\theta$, and $u_\theta$.
\end{myprop}
\begin{proof}
We consider the joint forward and backward SDE (cf.~\eqref{eq:forward_SDE} and \eqref{eq:backward_SDE})
with $Z_t^\theta := (X_t, Y_t^\theta)$:
\begin{align*}
    \rmd \cvectwo{X_t}{Y_t^\theta} = \underbrace{\cvectwo{f(X_t, t)}{h_\theta(X_t, t)}}_{=: F_\theta(Z_t^\theta, t)} \rmd t + \underbrace{\cvectwo{ g(X_t, t) }{ \nabla u_\theta(X_t, t)^\T g(X_t, t)}}_{=: G_\theta(Z_t^\theta, t)} \rmd B_t, \quad \cvectwo{X_0}{Y_0^\theta} = \cvectwo{x_0}{u_\theta(x_0, 0)}.
\end{align*}
Given our assumptions on $f, g, h_\theta, u_\theta$, we have that
both $F_\theta, G_\theta \in C^{0, 1}$. Hence, the pair $(F_\theta, G_\theta)$ is EM-regular (\Cref{def:EM_regularity}).
Therefore, by \Cref{thm:EM_strong_order_one_half},
the EM-discretization $\{(\hat{X}_n, \hat{Y}_n^\theta)\}_{n}$ (cf.~\eqref{eq:backwards_SDE_EM}),
coupled with $(Z_t^\theta)_t$ through Brownian increments $\{ \Delta W_n \}$, 
satisfies order $1/2$ strong convergence to $(Z_t^\theta)_t$:
\begin{align*}
    \left( \E\left[ \max_{n \in \{0, \dots, N\}} \max\{ \norm{\hat{X}_n - X_{t_n}}^2, \abs{\hat{Y}^\theta_n - Y^\theta_{t_n}}^2 \} \right]   \right)^{1/2} \leq C \tau^{1/2}.
\end{align*}
Now, define $\hat{\Psi}_N := u_\theta(\hat{X}_N, T) - \hat{Y}_N^\theta$ and $\Psi_T := u_\theta(X_T, T) - Y_T^\theta$
\begin{align*}
    \E[ \hat{\Psi}_N^2 ] &= \E[ (\Psi_T + (\hat{\Psi}_N - \Psi_T) )^2 ] 
    = \E[ \Psi_T^2 ] + \E[ (\hat{\Psi}_N - \Psi_T)^2 ] + 2 \E[ \Psi_T (\hat{\Psi}_N - \Psi_T) ].
\end{align*}
Hence by Cauchy-Schwarz,
\begin{align*}
    \abs{ \E[ \hat{\Psi}_N^2 ] - \E[ \Psi_T^2 ] } &\leq \E[ (\hat{\Psi}_N - \Psi_T)^2 ] + 2 \sqrt{ \E[ \Psi_T^2 ] }\sqrt{ \E[ (\hat{\Psi}_N - \Psi_T)^2 ] }.
\end{align*}
Since $u_\theta \in C^{1,1}$ the function is $\norm{u_\theta}_{C^{1,1}}$-Lipschitz and therefore:
\begin{align*}
    \E[ (\hat{\Psi}_N - \Psi_T)^2 ] &\leq 2 \E[ (u_\theta(\hat{X}_N, T) - u_\theta(X_T, T))^2 ] + 2 \E[ (\hat{Y}_N^\theta - Y_T^\theta)^2 ] \\
    &\leq 2 \norm{u_\theta}^2_{C^{1,1}} \E[ \norm{ \hat{X}_N - X_T }^2 ] + 2 \E[ (\hat{Y}_N^\theta - Y_T^\theta)^2 ] \\
    &\leq 2 C^2 (1 + \norm{u_\theta}^2_{C^{1,1}}) \tau.
\end{align*}
Furthermore $u_\theta$ is also $\norm{u_\theta}_{C^{1,1}}$-bounded and hence:
\begin{align*}
    \E[ \Psi_T^2 ] \leq 2 \norm{u_\theta}^2_{C^{1,1}} + 2 \E[ \abs{Y_T^\theta}^2 ].
\end{align*}
Next, since $F_\theta, G_\theta$ are both bounded, then by
\Cref{prop:SDE_holder_one_half}, we have that $\E[ \abs{Y_T^\theta}^2 ] = O(1)$.
Putting these bounds together yields:
\begin{align}
    \abs{ \E[ \hat{\Psi}_N^2 ] - \E[ \Psi_T^2 ] } \leq O(\tau + \sqrt{\tau}) = O(\sqrt{\tau}), \label{eq:psi_squared_gap}
\end{align}
since $\tau \leq 1$.
To finish the proof, we observe
\begin{align*}
    L_{\mathrm{EM}_N,\tau}(\theta) &= \E_{x_0 \sim \mu_0, w_n} \frac{1}{T^2}  \hat{\Psi}_N^2 \nonumber \\
    &= \E_{x_0 \sim \mu_0, B_t} \frac{1}{T^2}  \hat{\Psi}_N^2 \nonumber && \text{[coupling with Brownian increments]} \\
    &= \E_{x_0 \sim \mu_0, B_t} \frac{1}{T^2} \Psi_T^2 + O(\tau^{1/2}) && \text{[using \eqref{eq:psi_squared_gap}]} \\
    &= L_{\mathrm{BSDE},T}(\theta) + O(\tau^{1/2}). %
\end{align*}
\end{proof}

\begin{myprop}
\label{prop:EM_tau_vs_BSDE_tau}
Suppose that $f,g,h_\theta \in C^{0, 1}$, $u_\theta \in C^{2,1}$, and $\tau \leq 1$.
We have that:
\begin{align}
    L_{\mathrm{EM},\tau}(\theta) &= L_{\mathrm{BSDE},\tau}(\theta) + \mathrm{Bias}(\theta) + O(\tau^{1/2}), \label{eq:one_step_EM_BSDE_relation} \\
    \mathrm{Bias}(\theta) &:= \frac{1}{2T}\int_0^T \E[ \Tr( (H(X_t, t) \cdot \nabla^2 u_\theta(X_t, t))^2 )]  \rmd t \nonumber .
\end{align}
Here, $O(\cdot)$ hides factors depending on $d$, $T$, and the H{\"{o}}lder norms of $f$, $g$, $h_\theta$, and $u_\theta$.
\end{myprop}
\begin{proof}
We have the following:
\begin{align*}
    L_{\mathrm{EM},\tau}(\theta) &\stackrel{(a)}{=} \frac{1}{T} \int_0^T \E\left( (R[u_\theta](X_t, t))^2 + \frac{1}{2} \Tr((H(X_t, t) \cdot \nabla^2 u_\theta(X_t, t))^2) \right) \rmd t + O(\tau^{1/2}) \nonumber \\
    &= L_{\FPINNs}(\theta) + \mathrm{Bias}(\theta) + O(\tau^{1/2}) \nonumber \\
    &\stackrel{(b)}{=} L_{\mathrm{BSDE},\tau}(\theta) + \mathrm{Bias}(\theta) + O(\tau^{1/2}),
\end{align*}
where (a) holds from \Cref{thm:EM_final_form}, and
(b) holds from \Cref{prop:BSDE_FPINNs}.
\end{proof}

\subsection{Stratonovich BSDE and Heun discretization}

We first define the
Stratonovich variant of the FS-PINNs loss \eqref{eq:FS_PINNS_loss}:
\begin{align*}
    L_{\SFPINNs}(\theta) := \E_{\substack{x_0 \sim \mu_0, B_t}} \frac{1}{T} \int_{0}^{T} (R[u_\theta](\Xs_t, t))^2 \rmd t.
\end{align*}

\begin{myprop}
\label{prop:SBSDE_jensen_bound}
Suppose that $u_\theta \in C^2$. We have that:
\begin{align*}
    L_{\SBSDE,T}(\theta) \leq L_{\SFPINNs}(\theta).
\end{align*}
\end{myprop}
\begin{proof}
We mimic the arguments in \Cref{prop:BSDE_jensen_bound}.
Abbreviating $R_\theta(x, t) = R[u_\theta](x, t)$ and
using the Stratonovich chain rule, we have:
\begin{align*}
    u_\theta(\Xs_T, T) - u_\theta(\Xs_0, 0) &= \int_0^T \left[ R_\theta(\Xs_t, t) + h_\theta(\Xs_t, t) - \frac{1}{2}\Tr(H(\Xs_t, t) \nabla^2 u_\theta(\Xs_t, t)\right] \rmd t \\
    &\qquad + \int_0^T \nabla u_\theta(\Xs_t, t)^\T g(\Xs_t, t) \circ \rmd B_t,
\end{align*}
and hence the following identity which parallels \eqref{eq:BSDE_identity} holds:
\begin{align*}
    L_{\SBSDE,T}(\theta) = \E_{\substack{x_0 \sim \mu_0, B_t}} \left( \frac{1}{T} \int_{0}^{T} R_\theta(\Xs_t, t) \rmd t  \right)^2.
\end{align*}
Now we simply apply Jensen's inequality to conclude:
\begin{align*}
    L_{\SBSDE,T}(\theta) &= \E_{\substack{x_0 \sim \mu_0, B_t}} \left( \frac{1}{T} \int_{0}^{T} R_\theta(\Xs_t, t) \rmd t  \right)^2 \\
    &\leq \E_{\substack{x_0 \sim \mu_0, B_t}} \frac{1}{T} \int_0^T (R_\theta(\Xs_t, t))^2 \rmd t = L_{\SFPINNs}(\theta).
\end{align*}
\end{proof}

\begin{myprop}
\label{prop:SBSDE_SFPINNs}
Suppose that $f, h_\theta \in C^{0, 1}$, $g \in C^{1,1}$, $u_\theta \in C^{2,1}$, and $\tau \leq 1$. Then,
\begin{align}
    L_{\SBSDE,\tau}(\theta) = L_{\SFPINNs}(\theta) + O(\tau^{1/2}),
\end{align}
where the $O(\cdot)$ hides constants depending on the H{\"{o}}lder norms of $f$, $g$, $h_\theta$, and $u_\theta$.
\end{myprop}
\begin{proof}
The proof is nearly identical to 
the proof of \Cref{prop:BSDE_FPINNs}, but
with $L_{\SBSDE,\tau}(\theta)$
taking the place of $L_{\mathrm{BSDE},\tau}(\theta)$
and $L_{\SFPINNs}(\theta)$ taking the place of $L_{\FPINNs}(\theta)$.
The only notable difference is we need to establish
the condition $\E \norm{ \Xs_{t_0} - \Xs_{t_1} }^2 \leq O(1) \abs{t_0 - t_1}$
for $t_0, t_1 \in \calI$.
By our assumption that $f \in C^{0, 1}$ and $g \in C^{1, 1}$, we have that
both $f(x, t) + \frac{1}{2}\sum_{k=1}^{d} \partial_x g^k(x, t) g(x, t)$ 
and $g(x, t)$ are bounded, and hence the condition 
$\E \norm{ \Xs_{t_0} - \Xs_{t_1} }^2 \leq O(1) \abs{t_0 - t_1}$ holds
by \Cref{prop:strat_SDE_holder_one_half}.
\end{proof}

\begin{myprop}
\label{prop:SBSDE_T_vs_tau}
Suppose the assumptions of \Cref{prop:SBSDE_SFPINNs} hold. Then,
\begin{align*}
    L_{\SBSDE,T}(\theta) \leq L_{\SBSDE,\tau}(\theta) + O(\tau^{1/2}),
\end{align*}
where the $O(\cdot)$ hides constants depending on the H{\"{o}}lder norms of $f$, $g$, $h_\theta$, and $u_\theta$.
\end{myprop}
\begin{proof}
Follows immediately from \Cref{prop:SBSDE_jensen_bound} and \Cref{prop:SBSDE_SFPINNs}.
\end{proof}

\begin{myprop}
\label{prop:heun_N_tau_vs_SBSDE_T}
Suppose that $f, h_\theta \in C^{0,1}$, $g \in C^{1,1}$, $u_\theta \in C^{2,1}$, and $\tau \leq 1$.
We have that:
\begin{align}
    L_{\mathrm{Heun}_N,\tau}(\theta) = L_{\SBSDE,T}(\theta) + O(\tau^{1/2}),
\end{align}
where the $O(\cdot)$ hides constants depending on the $d$, $T$, and the H{\"{o}}lder norms of $f$, $g$, $h_\theta$, and $u_\theta$.    
\end{myprop}
\begin{proof}
We consider the joint forward/backward 
Stratonovich SDE $\Zs_t = (\Xs_t, \Ys_t)$ from \eqref{eq:forward_backwards_strat}
of the form
$\rmd \Zs_t = F_\theta(\Zs_t, t) \rmd t + G_\theta(\Zs_t, t) \circ \rmd B_t$.
We first show that the pair $(F_\theta, G_\theta)$ is Heun-regular (cf.~\Cref{def:heun_regularity}).
A sufficient condition is that (a) $F_\theta \in C^{0, 1}$ 
and (b) $G_\theta \in C^{1, 1}$.
For condition (a), it is equivalent to both $f$ and $\hs_\theta(x, t) = h_\theta(x, t) - \frac{1}{2} \Tr( H(x, t) \nabla^2 u_\theta(x, t))$ are in $C^{0, 1}$;
the former is by assumption, and the latter holds since
$h_\theta \in C^{0, 1}$, $g \in C^{1,1}$, and $u_\theta \in C^{2,1}$ by assumption.
Now for condition (b), it is equivalent to both $g$ and $\nabla u_\theta(x, t)^\T g(x, t)$ are in $C^{1, 1}$.
The former is again by assumption, and the latter holds since $u_\theta \in C^{2,1}$ and $g \in C^{1, 1}$.
Hence by \Cref{corollary:heun_strong_convergence}, we have that Heun discretization
$\{\Zhats_n\}_n$ from \eqref{eq:heun_scheme}, coupled with the SDE 
$(\Zs_t)_t$ through Brownian increments $\{\Delta W_n\}_n$,
satisfies order $1/2$ strong convergence to the SDE $(\Zs_t)_t$, i.e., 
\begin{align*}
    \left( \E\left[ \max_{n \in \{0, \dots, N\}} \max\{ \norm{\Xhats_n - \Xs_{t_n}}^2, \abs{\Yhats_n - \Ys_{t_n}}^2 \} \right] \right)^{1/2} \leq C \tau^{1/2},
\end{align*}
where $C$ does not depend on $\tau$.
The remainder of the proof proceeds nearly identically to 
\Cref{prop:EM_N_tau_vs_BSDE_T}, with the only difference being that
in order to argue $\E[ \abs{\Ys_T}^2 ] = O(1)$, we 
utilize that $F_\theta \in C^{0, 1}$
and $G_\theta \in C^{1, 1}$ to invoke \Cref{prop:strat_SDE_holder_one_half}.
\end{proof}

\paragraph{Showing that $L_{\mathrm{Heun},\tau}(\theta) = L_{\SBSDE,\tau}(\theta) + O(\tau^{1/2})$.}

We start from \eqref{eq:Heun_tau_one_step_limit} and following nearly identical
arguments as in the derivation of \eqref{eq:one_step_EM_BSDE_relation}.

\begin{myprop}
\label{prop:heun_tau_vs_SBSDE_tau}
Suppose that $f$, $g$, and $h_\theta$ are all in $C^{1,1}$, $u_\theta \in C^{3,1}$, and $\tau \leq 1$.
We have:
\begin{align*}
    L_{\mathrm{Heun},\tau}(\theta) = L_{\SBSDE,\tau}(\theta) + O(\tau^{1/2}),
\end{align*}
where the $O(\cdot)$ hides factors depending on $d$, $T$, and the H{\"{o}}lder norms of $f$, $g$, $h_\theta$, and $u_\theta$.
\end{myprop}
\begin{proof}
We proceed similarly to \Cref{prop:EM_tau_vs_BSDE_tau}:
\begin{align*}
    L_{\mathrm{Heun},\tau}(\theta) &\stackrel{(a)}{=} \frac{1}{T} \int_0^T \E (R[u_\theta](\Xs_t, t))^2 \rmd t + O(\tau^{1/2}) \\
    &= L_{\SFPINNs}(\theta) + O(\tau^{1/2}) \\
    &\stackrel{(b)}{=} L_{\SBSDE,\tau}(\theta) + O(\tau^{1/2}),
\end{align*}
where (a) holds from \Cref{thm:heun_final_form}, and (b) holds from \Cref{prop:SBSDE_SFPINNs}.
\end{proof}

\section{Strong Convergence of Euler-Maruyama and Heun Integration}
\label{sec:appendix:SDE_convergence}

Let $\calI := [0, T]$ denote a time interval, and 
consider functions $a : \R^d \times \calI \mapsto \R^d$
and $b : \R^d \times \calI \mapsto \R^{d \times m}$ which define the following SDE:
\begin{align}
    \rmd X_t = a(X_t, t) \rmd t + b(X_t, t) \diamond \rmd B_t, \quad X_0 \sim \calD_0. \label{eq:SDE_a_b_pair}
\end{align}
where $(B_t)_{t \geq 0}$ is $m$-dimensional Brownian motion and $\calD_0$ is an 
arbitrary distribution over $\R^d$ with bounded second moments, i.e., $\E \norm{X_0}^2 < \infty$.
Here, the pair $(a, b)$ is used instead of $(f, g)$ to avoid confusion
with the forward SDE \eqref{eq:forward_SDE}, and the
$\diamond$ notation denotes that the SDE \eqref{eq:SDE_a_b_pair} 
is either to be interpreted as an It{\^{o}} or Stratonovich SDE.
We write $b^k : \R^d \times \calI \mapsto \R^d$ for $k \in [m]$ so that $b = (b^1, \dots, b^m)$,
i.e., $b^k(t, x)$ is the $k$-th column of the matrix $b(t, x)$.
We consider a discretization time $\tau \in (0, T]$ such that $N := \floor{T / \tau} \in \N_+$.
We denote the timesteps $\{t_n\}_{n=0}^{N}$ and Brownian increments $\{ \Delta W_n \}_{n=0}^{N-1}$ as
$t_n := n \tau$ and
$\Delta W_n := B_{t_{n+1}} - B_{t_n}$.
We will also often utilize the shorthand notation $a_n(x) := a(x, t_n)$,
$b_n(x) := b(x, t_n)$, and $b_n^k(x) := b^k(x, t_n)$ for $n \in \{0, \dots, N\}$.

In this section, we review standard results regarding convergence
of basic stochastic integration schemes (Euler-Maruyama for It{\^{o}}, stochastic Heun for Stratonovich) for the SDE \eqref{eq:SDE_a_b_pair}.

\subsection{Euler-Maruyama Convergence}
\label{sec:appendix:EM_convergence}

The Euler-Maruyama scheme for integrating the SDE \eqref{eq:SDE_a_b_pair}
interpreted as an It{\^{o}} SDE is the following discrete-time process:
\begin{align}
    \hat{X}_{n+1} = a_n(\hat{X}_n) \tau + b_n(\hat{X}_n) \Delta W_n, \quad \hat{X}_0 = X_0. \label{eq:EM_a_b_pair}
\end{align}
 
The order $1/2$ strong convergence of the Euler-Maruyama process \eqref{eq:EM_a_b_pair} to the It{\^{o}} SDE \eqref{eq:SDE_a_b_pair}
is thoroughly documented in the literature. 
Concretely, we will state a result from \cite{kloeden1992numericalsolutions}.
First, we define the necessary regularity condition on the drift and diffusion
terms
\begin{mydef}[EM-regularity]
\label{def:EM_regularity}
The pair $(a, b)$ with $a : \R^d \times \calI \mapsto \R^d$
and $g : \R^d \times \calI \mapsto \R^{d \times m}$ 
is \emph{EM-regular} if 
there exists finite $K_1, K_2, K_3$ such that
for all $x, y \in \R^d$ and $s, t \in \calI$,
\begin{align*}    
    \norm{a(x, t)} + \opnorm{b(x, t)} &\leq K_1 (1 + \norm{x}), \\
    \norm{ a(x, t) - a(y, t) } + \opnorm{ b(x, t) - b(y, t) } &\leq K_2 \norm{ x - y }, \\
    \norm{ a(x, s) - a(x, t) } + \opnorm{ b(x, s) - b(x, t) } &\leq K_3 (1 + \norm{x})\abs{s-t}^{1/2}.
\end{align*}
\end{mydef}
By definition of the H{\"{o}}lder class $C^{0, 1}$, we have that if
the pair $(a, b)$ satisfies $a,b \in C^{0, 1}$, then $(a, b)$ is EM-regular, 
although \Cref{def:EM_regularity} is a weaker assumption.
With this notation of regularity in place, we have the following 
order $1/2$ strong convergence result.
\begin{mythm}[{\cite[Theorem 10.2.2]{kloeden1992numericalsolutions}}]
\label{thm:EM_strong_order_one_half}
Suppose the pair $(a, b)$ is EM-regular (cf.~\Cref{def:EM_regularity}).
Then the It{\^{o}} SDE $(X_t)_t$ defined in \eqref{eq:SDE_a_b_pair}
and the Euler-Maruyama discretization $\{\hat{X}_n\}_n$ defined in \eqref{eq:EM_a_b_pair}
satisfy the following bound:
\begin{align*}
    \left( \E \max_{n \in \{0, \dots, N\}} \norm{ {X}_{t_n} - \hat{{X}}_n }^2 \right)^{1/2} \leq C \sqrt{\tau},
\end{align*}
where the constant $C$ does not depend on $\tau$.
\end{mythm}

\subsection{Stochastic Heun Convergence}
\label{sec:appendix:heun_convergence}

The stochastic Heun discretization of the Stratonovich SDE $(X_t)_t$ defined in \eqref{eq:SDE_a_b_pair} is the discrete-time process with $\hat{X}_0 = X_0$ and:
\begin{subequations} \label{eq:heun_a_b_pair}
\begin{align}
  \hat{Y}_{n+1} &= \hat{X}_n + a_n(\hat{X}_t) \tau + b_n(\hat{X}_t) \Delta W_n, \label{eq:heun_a_b_predictor} \\
  \hat{X}_{n+1} &= \hat{X}_n + \frac{1}{2}\left[ a_n(\hat{X}_t) + a_{n+1}(\hat{Y}_{n+1}) \right] \tau + \frac{1}{2}\left[ b_n(\hat{X}_t) + b_{n+1}(\hat{Y}_{n+1}) \right] \Delta W_n. \label{eq:heun_a_b_trapezoid}
\end{align}
\end{subequations}
Analogous to the order $1/2$ strong convergence results in \Cref{sec:appendix:EM_convergence}
for the EM discretization of the It{\^{o}} SDE, we also have a similar result that holds
for the Heun discretization \eqref{eq:heun_a_b_pair} of the Stratonovich SDE \eqref{eq:SDE_a_b_pair}.
While we consider such a result to be a folklore result, 
we were unable to find a specific theorem statement in the literature listing out a
precise set of sufficient conditions on $(a, b)$ for strong convergence to hold.\footnote{The
closest statement we were able to find in the literature is \cite[Theorem D.12]{kidger2021neuralsdes},
which shows order $1/2$ strong convergence of the \emph{reversible} Heun method,
which is a modified version of the stochastic Heun method that is algebraically reversible.}
Thus, the rest of this sub-section provides a result and mostly self-contained proof
that builds on top of EM results stated in \Cref{sec:appendix:EM_convergence}.

We first start with a sufficient regularity condition, which 
adds a few extra assumptions to the EM-regular definition (\Cref{def:EM_regularity}).
\begin{mydef}[Heun-regularity]
\label{def:heun_regularity}
The pair $(a, b)$ with $a : \R^d \times \calI \mapsto \R^d$
and $b : \R^d \times \calI \mapsto \R^{d \times m}$ is \emph{Heun-regular} 
if for every $t \in \calI$ and $k \in [m]$, the map $x \mapsto b^k(t, x)$ is $C^1(\R^d)$,
and there exists finite $K_i$, $i \in [5]$ such that
for all $x, y \in \R^d$ and $s, t \in \calI$:
\begin{align*}
    \norm{a(x, t)} + \opnorm{b(x, t)} &\leq K_1, \\
    \norm{ a(x, t) - a(y, t) } + \opnorm{ b(x, t) - b(y, t) } &\leq K_2 \norm{ x - y }, \\
    \norm{ a(x, s) - a(x, t) } + \opnorm{ b(x, s) - b(x, t) } &\leq K_3( 1 + \norm{x})\abs{s-t}, \\
    \opnorm{ \partial_x b^k(x, t) - \partial_x b^k(y, t) } &\leq K_4 \norm{ x - y }, \\
    \opnorm{ \partial_x b^k(x, s) - \partial_x b^k(x, t) } &\leq K_5 (1 + \norm{x}) \abs{s-t}^{1/2}.
\end{align*}
\end{mydef}
We note that from the definition of the H{\"{o}}lder classes $C^{0, 1}$ and $C^{1, 1}$
that if $a \in C^{0, 1}$ and $b \in C^{1, 1}$, then the pair $(a, b)$ is Heun-regular.
The following result is the main convergence result for Heun.
\begin{mythm}
\label{corollary:heun_strong_convergence}
Suppose that the pair $(a, b$) is Heun-regular (cf.~\Cref{def:heun_regularity}).
Then the Stratonovich SDE $(X_t)_t$ defined in \eqref{eq:SDE_a_b_pair}
and the stochastic Heun discretization $\{\hat{X}_n\}_n$ defined in 
\eqref{eq:heun_a_b_pair} satisfy:
\begin{align}
    \left(\E \max_{n \in \{0, \dots, N\}} \norm{ X_{t_n} - \hat{X}_{n} }^2\right)^{1/2} \leq C \sqrt{\tau}, 
\end{align}
where the constant $C$ does not depend on $\tau$.
\end{mythm}

\subsubsection{Proof of \Cref{corollary:heun_strong_convergence}}

Our proof of \Cref{corollary:heun_strong_convergence} is based on the following reduction.
By defining:
\begin{align*}
    \bar{a}(x, t) := a(x, t) + \frac{1}{2} \sum_{k=1}^{m} \partial_x b^k(x, t) b^k(x, t),
\end{align*}
the It{\^{o}} SDE:
\begin{align}
    \rmd \bar{X}_t = \bar{a}(\bar{X}_t, t) \rmd t + b(\bar{X}_t, t) \rmd B_t, \quad \bar{X}_0 = X_0, \label{eq:ito_SDE_a_b_pair}
\end{align}
defines an identical process as the Stratonovich SDE $(X_t)_t$ from \eqref{eq:SDE_a_b_pair},
i.e., $(X_t(\omega))_t = (\bar{X}_t(\omega))_t$ for almost every $t, \omega$.
Furthermore, we can consider an Euler-Maruyama discretization of \eqref{eq:ito_SDE_a_b_pair}:
\begin{align}
    \hat{\bar{X}}_{n+1} &= \hat{\bar{X}}_{n} + \bar{a}_n(\hat{\bar{X}}_n) \tau + b_n(\hat{\bar{X}}_n) \Delta W_n, \quad \hat{\bar{X}}_0 = X_0. \label{eq:EM_abar_b_pair}
\end{align}

As the It{\^{o}} SDE \eqref{eq:ito_SDE_a_b_pair} and 
Stratonovich SDE \eqref{eq:SDE_a_b_pair} are identical, then we also have that 
if the pair $(\bar{a}, b)$ is EM-regular (cf.~\Cref{def:EM_regularity}), then the EM discretization 
\eqref{eq:EM_abar_b_pair}
is order $1/2$ strongly convergent to the Stratonovich SDE \eqref{eq:SDE_a_b_pair}.
Hence, this reduces the problem to comparing the two discrete processes
$\{ \hat{X}_n \}_n$ from \eqref{eq:heun_a_b_pair} 
and $\{ \hat{\bar{X}}_n \}_n$ from \eqref{eq:EM_abar_b_pair}.
In particular, if we can show that:
\begin{align*}
    \left( \E \max_{n \in \{0, \dots, N\}} \norm{ \hat{X}_n - \hat{\bar{X}}_n }^2 \right)^{1/2} \leq C \sqrt{\tau},
\end{align*}
where again the two processes are coupled under the same Brownian motion $(B_t)_t$ and initial condition $X_0$,
then by triangle inequality we have the desired result 
\Cref{corollary:heun_strong_convergence}.
One advantage of this proof strategy is that we only need
to study the evolution of two discrete-time processes, which we can do with purely elementary
(discrete-time) martingale techniques, avoiding the need for any stochastic calculus.
Indeed, the main tools we utilize are the following two results.
\begin{myprop}[{Doob's maximal inequality (vector-valued), cf.~\cite[Theorem 3.2.2]{hytonen2016analysis}}]
\label{prop:doob_maximal_inequality}
Let $(X_n)_{n \in \N_+}$ denote a martingale taking values in a normed vector space $X$ with norm $\norm{\cdot}_{X}$. We have that for any $p \in (0, \infty]$ and $n \in \N_+$:
\begin{align*}
    \E\left( \max_{i \in [n]} \norm{X_i}_X^p \right)^{1/p} \leq \frac{p}{p-1} \left(\E[\norm{X_n}_{X}^p]\right)^{1/p}.
\end{align*}
\end{myprop}

\begin{myprop}[{Discrete Gronwall inequality, cf.~\cite{clark1987discretegronwal}}]
\label{prop:discrete_gronwall}
Let $\{x_n\}_{n \in \N}$ and $\{ \beta_n \}_{n \in \N}$ be non-negative sequences satisfying for some $\alpha > 0$:
\begin{align*}
    x_n \leq \alpha + \sum_{k=0}^{n-1} \beta_k x_n, \quad n \in \N.
\end{align*}
Then we have:
\begin{align*}
    x_n \leq \alpha \exp\left( \sum_{k=0}^{n-1} \beta_k \right), \quad n \in \N.
\end{align*}
Here, we interpret $\sum_{k=0}^{-1}$ to indicate zero.
\end{myprop}

Our first step shows that Heun-regularity of the pair $(a, b)$ implies
EM-regularity of the pair $(\bar{a}, b)$.
\begin{myprop}
\label{prop:heun_implies_EM_regular}
If the pair $(a, b)$ is Heun-regular (cf.~\Cref{def:heun_regularity}), 
then the pair $(\bar{a}, b)$ is EM-regular (cf.~\Cref{def:EM_regularity}).
\end{myprop}
\begin{proof}
We let $K := \max_{i \in [5]} K_i$.
We first check the growth condition on $\norm{\bar{a}(x, t)}$:
\begin{align*}
    \norm{ \bar{a}(x, t) } &= \bignorm{ a(x, t) + \frac{1}{2} \sum_{k=1}^{m} \partial_x b^k(x, t) b^k(x, t) } \\
    &\stackrel{(a)}{\leq} K + \frac{1}{2} \sum_{k=1}^{m} \norm{ \partial_x b^k(x, t) b^k(x, t) } 
    \stackrel{(b)}{\leq} K + \frac{K}{2} \sum_{k=1}^{m} \opnorm{ \partial_x b^k(x, t) } 
    \stackrel{(c)}{\leq} K + \frac{K^2 m}{2},
s\end{align*}
where (a) holds since $\norm{a(x, t)} \leq K$, 
(b) holds since $\norm{ b^k(x, t) } \leq \opnorm{ b(x, t) } \leq K$, and
(c) holds since $\norm{ g^k(x, t) - g^k(y, t) } \leq K \norm{x - y}$ implies that
$\opnorm{ \partial_x g^k(x, t) } \leq K$.

Next, we check the Lipschitz condition over $x$ on $\bar{a}(x, t)$:
\begin{align*}
    \norm{ \bar{a}(x, t) - \bar{a}(y, t) } &\leq \norm{a(x, t) - a(y, t)} + \frac{1}{2}\sum_{k=1}^{m} \norm{ \partial_x b^k(x, t) b^k(x, t) - \partial_x b^k(y, t) b^k(y, t) } \\
    &\stackrel{(a)}{\leq} K\norm{x-y} + \frac{1}{2} \sum_{k=1}^{m} \norm{ \partial_x b^k(x, t) [ b^k(x, t) - b^k(y, t)] } \\
    &\qquad + \frac{1}{2}\sum_{k=1}^{m} \norm{ [ \partial_x b^k(x, t) - \partial_x b^k(y, t)] b^k(y, t) } \\
    &\stackrel{(b)}{\leq} K\norm{x-y} + \frac{K^2 m}{2} \norm{x-y} + \frac{K^2 m}{2} \norm{x-y} = (K + K^2 m ) \norm{x-y}.
\end{align*}
where (a) uses the Lipschitz condition $\norm{a(x, t) - a(y, t)} \leq K \norm{x-y}$,
and (b) uses $\opnorm{ \partial_x b^k(x, t) - \partial_x b^k(y, t) } \leq K \norm{x - y}$,
$\norm{ b^k(x, t) } \leq K$,
$\opnorm{ \partial_x b^k(x, t) } \leq K$, and
$\norm{ b^k(x, t) - b^k(y, t) } \leq K \norm{x - y}$,

Finally, we check the H{\"{o}}lder $1/2$ condition over $t$ on $\bar{a}(x, t)$:
\begin{align*}
    \norm{\bar{a}(x, s) - \bar{a}(x, t)} &\leq \norm{a(x, s) - a(x, t)} + \frac{1}{2}\sum_{k=1}^{m} \norm{ \partial_x b^k(x, s) b^k(x, s) - \partial_x b^k(x, t) b^k(x, t) } \\
    &\stackrel{(a)}{\leq} K (1+\norm{x}) \abs{s-t} + \frac{1}{2}\sum_{k=1}^{m} \norm{ \partial_x b^k(x, s) [ b^k(x, s) - b^k(x, t)] } \\
    &\qquad + \frac{1}{2}\sum_{k=1}^{m} \norm{ [ \partial_x b^k(x, s) - \partial_x b^k(x, t)] b^k(x, t) } \\
    &\stackrel{(b)}{\leq} K (1+\norm{x})\abs{s-t} + \frac{K^2 m}{2} (1 + \norm{x}) \abs{s-t} + \frac{K^2 m}{2} (1+\norm{x})\abs{s-t}^{1/2} \\
    &\stackrel{(c)}{\leq} (K + K^2 m) \sqrt{T} (1+\norm{x}) \abs{s-t}^{1/2},
\end{align*}
where in (a) we use $\norm{a(x, s) - a(x, t)} \leq K (1+\norm{x})\abs{s-t}$,
in (b) we use
$\norm{ b^k(x, s) - b^k(x, t) } \leq K (1+\norm{x}) \abs{s-t}$,
$\opnorm{ \partial_x b^k(x, t) } \leq K$,
$\norm{ \partial_x b^k(x, s) - \partial_x b^k(x, t) } \leq K(1+\norm{x})\abs{s-t}^{1/2}$, and
$\norm{ b^k(x, t) } \leq K$, and in (c) we use
$\abs{s-t} = \abs{s-t}^{1/2} \cdot \abs{s-t}^{1/2} \leq \sqrt{T} \abs{s-t}^{1/2}$.

Since the growth, Lipschitz, and H{\"{o}}lder $1/2$ conditions on $b(x, t)$ are immediate
from the Heun regularity assumptions, this concludes the claim.
\end{proof}

The following result shows that the discrete-time processes 
\eqref{eq:heun_a_b_pair} and \eqref{eq:EM_abar_b_pair} are strongly convergent.
\begin{mylemma}
\label{lemma:heun_strong_convergence_discrete}
Suppose that the pair $(a, b)$ is Heun-regular (cf.~\Cref{def:heun_regularity})
and that $\tau \leq 1$.
Then, we have that the updates \eqref{eq:heun_a_b_pair} and \eqref{eq:EM_abar_b_pair} satisfy:
\begin{align*}
    \left( \E \max_{n \in \{0, \dots, N\}} \norm{ \hat{X}_n - \hat{\bar{X}}_n }^2 \right)^{1/2} \leq C \sqrt{\tau},
\end{align*}
where the constant $C$ does not depend on $\tau$.
\end{mylemma}
\begin{proof}
To start the proof, we recall that $t_n = n \tau$, $N = \floor{T/\tau}$ which is assumed to be a positive integer,
and $\Delta W_n := B_{t_{n+1}} - B_{t_n}$.
We will equivalently write $\Delta W_n = \sqrt{\tau} w_n$, where $\{ w_n \}_{n=0}^{N-1}$ are i.i.d.\ $\sfN(0, I_m)$ random vectors.
In order to index the coordinates of both $\Delta W_n$ and $w_n$, we use the notation
$\Delta W_n^i$ and $w_n^i$, for $i \in [m]$, to refer to the $i$-th coordinate of the vector.
We let $K := 1 + \E\norm{X_0}^2 + \max_{i \in [5]} K_i$ to denote a bound on all the parameters from both Heun regularity and the second moment of the initial condition $X_0$.
As with $a_n, b_n$, we also define $\bar{a}_n(x) := \bar{a}(x, t_n)$ for $n \in [N]$.
Furthermore, we will drop the hat notation to reduce
notational clutter and
and write $X_n, Y_n = \hat{X}_n, \hat{Y}_n$,
and similarly $\bar{X}_n = \hat{\bar{X}}_n$;
since we are not dealing with the SDEs
\eqref{eq:SDE_a_b_pair} and
\eqref{eq:ito_SDE_a_b_pair},
there is no risk of confusion with this notation.

To avoid keeping track of explicit dependence on constants besides $\tau$,
for a set of parameters $Q$
we use $C_Q$ to denote a finite constant that
that depends only on the parameters listed in $Q$,
and $a \lesssim_Q b$ to denote $a \leq C_Q b$.
For example, $C_{K, m}$ is a constant
that depends only on $(K, m)$ (its dependency may be arbitrary however). 
We also let $a \lesssim b$ denote $a \leq C b$ where $C$ is a universal positive constant.

With the aforementioned notation in place, 
we have the following discrete-time update rules for \eqref{eq:EM_abar_b_pair}:
\begin{align*}
    \bar{X}_{n+1} = \bar{X}_n + \bar{a}_n(\bar{X}_n) \tau + b_n(\bar{X}_n) \Delta W_n,
\end{align*}
and for \eqref{eq:heun_a_b_pair}:
\begin{align*}
    Y_{n+1} &= X_n + a_n(X_n) \tau + b_n(X_n) \Delta W_n, \\
    X_{n+1} &= X_n + \frac{1}{2}\left[ a_n(X_n) + a_{n+1}(Y_{n+1})  \right]\tau + \frac{1}{2}\left[ b_n(X_n) + b_{n+1}(Y_{n+1})  \right] \Delta W_n.
\end{align*}
Let us define the filtration $\calF_n := \sigma(w_0, \dots, w_{n-1})$ for $n \in \N_+$
with $\calF_0$ the
trivial $\sigma$-algebra.
We observe a key property of both $\{ X_n \}_n$ and
$\{ \bar{X}_n \}_n$ is that both $X_n$
and $\bar{X}_n$ are $\calF_{n}$-measurable.
Hence by the tower property we have for expressions $A, B$,
$\E[ A(X_n) B(w_n) ] = \E[ A(X_n) \E[ B(w_n) \mid \calF_n ]]$ and
$\E[ B(w_n) \mid \calF_n] = \E_{w_n \sim \sfN(0, I_m)}[ B(w_n) ]$, 
a property we make heavy use of in our calculations.
Another simple inequality we make use of is that
for any $q \in \N_+$ and any set of vectors $v_1, \dots, v_q$,
\begin{align*}
    \bignorm{ \sum_{i=1}^{q} v_i }^2 \leq q \sum_{i=1}^{q} \norm{v_i}^2,
\end{align*}
which follows by triangle inequality and Cauchy-Schwarz.

\paragraph{Heun update decomposition.}
To relate the Heun and EM updates, we write the Heun update as:
\begin{align}
    X_{n+1} = X_n + \bar{a}_n(X_n)\tau + b_n(X_n)\Delta W_n + E_n, \label{eq:heun_EM_decomposition}
\end{align}
where $E_n$ contains the residual terms:
\begin{align*}
    E_n := \underbrace{\frac{1}{2}\left[ a_{n+1}(Y_{n+1}) - a_n(X_n) \right]\tau}_{=: E_n^a} + \underbrace{\frac{1}{2} \left[ b_{n+1}(Y_{n+1}) - b_n(X_n) \right] \Delta W_n - \frac{\tau}{2} \sum_{i=1}^{m} \partial_x b_n^k(X_n) b^k_n(X_n)}_{=: E_n^b}.
\end{align*}
We further decompose $E_n^b$ as follows.
We first write:
\begin{align*}
    b_{n+1}(Y_{n+1}) - b_n(X_n) &= (b_{n+1}(Y_{n+1}) - b_n(Y_{n+1})) + (b_n(Y_{n+1}) - b_n(X_n)).
\end{align*}
Next, we use the Heun-regularity to expand the RHS above:
\begin{align*}
    b^k_n(Y_{n+1}) - b^k_n(X_n) &= \partial_x b^k_n(X_n) (Y_{n+1} - X_n) + R^k_n \\
    &= \partial_x b^k_n(X_n) ( a_n(X_n)\tau + b_n(X_n) \Delta W_n ) + R^k_n,
\end{align*}
where the remainder term $R^k_n$ satisfies $\norm{R_n^k} \leq K \norm{Y_{n+1} - X_n}^2$.
Hence,
\begin{align*}
    &(b_n(Y_{n+1}) - b_n(X_n)) \Delta W_n \\
    &= \sum_{k=1}^{m} (b_n^k(Y_{n+1}) - b_n^k(X_n)) \Delta W_n^k \\
    &= \tau\sum_{k=1}^{m} \partial_x b_n^k(X_n) a_n(X_n) \Delta W_n^k + \sum_{k=1}^{m} \partial_x b_n^k(X_n) b_n(X_n) \Delta W_n \Delta W_n^k + \sum_{k=1}^{m} R_n^k \Delta W_n^k
\end{align*}
Now the middle term further decomposes as:
\begin{align*}
    \sum_{k=1}^{m} \partial_x b_n^k(X_n) b_n(X_n) \Delta W_n \Delta W_n^k &= \sum_{k_1,k_2=1}^{m} \partial_x b_n^{k_1}(X_n) b_n^{k_2}(X_n) \Delta W_n^{k_1} \Delta W_n^{k_2} \\
    &= \sum_{k_1,k_2=1}^{m} \partial_x b_n^{k_1}(X_n) b_n^{k_2}(X_n) \left(\Delta W_n^{k_1} \Delta W_n^{k_2} - \tau \ind_{\{ k_1 = k_2 \}} \right) \\
    &\qquad + \tau \sum_{k=1}^{m} \partial_x b_n^k(X_n) b_n^k(X_n) .
\end{align*}
Combining these decompositions together,
\begin{align*}
    E_n^b &= \frac{1}{2} \left[ b_{n+1}(Y_{n+1}) - b_n(X_n) \right] \Delta W_n - \frac{\tau}{2} \sum_{i=1}^{m} \partial_x b_n^k(X_n) b^k_n(X_n) \\
    &= \frac{1}{2} \left[ b_{n+1}(Y_{n+1}) - b_n(Y_{n+1}) \right] \Delta W_n + \frac{1}{2} \left[ b_n(Y_{n+1}) - b_n(X_n) \right] \Delta W_n - \frac{\tau}{2} \sum_{i=1}^{m} \partial_x b_n^k(X_n) b^k_n(X_n) \\
    &=  \frac{1}{2} \left[ b_{n+1}(Y_{n+1}) - b_n(Y_{n+1}) \right] \Delta W_n + \frac{\tau}{2} \sum_{k=1}^{m} \partial_x b^k_n(X_n) a_n(X_n) \Delta W_n^k + \frac{1}{2} \sum_{k=1}^{m} R_n^k \Delta W_n^k \\
    &\qquad + \frac{1}{2} \sum_{k=1}^{m} \partial_x b_n^k(X_n) b_n(X_n) \Delta W_n \Delta W_n^k - \frac{\tau}{2} \sum_{i=1}^{m} \partial_x b_n^k(X_n) b^k_n(X_n) \\ 
    &=  \underbrace{\frac{1}{2} \left[ b_{n+1}(Y_{n+1}) - b_n(Y_{n+1}) \right] \Delta W_n}_{=: E_n^{b,1}} + \underbrace{\frac{\tau}{2} \sum_{k=1}^{m} \partial_x b^k_n(X_n) a_n(X_n) \Delta W_n^k}_{=:E_n^{b,2}} + \underbrace{\frac{1}{2} \sum_{k=1}^{m} R_n^k \Delta W_n^k}_{=:E_n^{b,3}} \\
    &\qquad + \underbrace{\frac{1}{2} \sum_{k_1,k_2=1}^{m} \partial_x b_n^{k_1}(X_n) b_n^{k_2}(X_n) \left(\Delta W_n^{k_1} \Delta W_n^{k_2} - \tau \ind_{\{ k_1 = k_2 \}} \right)}_{=:E_n^{b,4}}.
\end{align*}
Thus, \eqref{eq:heun_EM_decomposition}
becomes:
\begin{align}
    X_{n+1} = X_n + \bar{a}_n(X_n)\tau + b_n(X_n)\Delta W_n + E_n^a + \sum_{\ell=1}^{4} E_n^{b,\ell}, \label{eq:heun_EM_decomposition_full}
\end{align}
which serves as the starting point for what follows.

\paragraph{Second moment bounds on error terms.}
Our next step is to bound the second moment of all the error terms separately. Here we make heavy use of the Heun-regularity conditions.

\emph{Bound on $\E\norm{E_n^a}^2$:}
We write:
\begin{align*}
    \norm{a_{n+1}(Y_{n+1}) - a_n(X_n)} &= \norm{(a_{n+1}(Y_{n+1}) - a_{n+1}(X_n)) + (a_{n+1}(X_n) - a_n(X_n))} \\
    &\leq \norm{ a_{n+1}(Y_{n+1}) - a_{n+1}(X_n) } + \norm{ a_{n+1}(X_n) - a_n(X_n) } \\
    &\leq K \norm{Y_{n+1} - X_n} + K(1 + \norm{X_n}) \sqrt{\tau} \\
    &\leq K(\norm{a_n(X_n)} \tau + \norm{b_n(X_n)} \sqrt{\tau} \norm{w_n}) + K (1 + \norm{X_n})\sqrt{\tau} \\
    &\leq K^2( \tau + \sqrt{\tau} \norm{w_n} ) + K (1 + \norm{X_n})\sqrt{\tau}.
\end{align*}
Hence,
\begin{align*}
    \E\norm{a_{n+1}(Y_{n+1}) - a_n(X_n)}^2 \lesssim  K^4 (\tau^2 + \tau m) + K^2(1 + \E\norm{X_n}^2) \tau 
    \lesssim_{K,m} (1+\E\norm{X_n}^2) \tau.
\end{align*}
which implies:
\begin{align*}
    \E\norm{E_n^a}^2 = \frac{\tau^2}{4} \E\norm{a_{n+1}(Y_{n+1}) - a_n(X_n)}^2 \lesssim_{K,m} (1+\E\norm{X_n}^2) \tau^3.
\end{align*}

\emph{Bound on $\E\norm{E_n^{b,1}}^2$:}
We have:
\begin{align*}
    \E\norm{E_n^{b,1}}^2 &= \frac{1}{4} \E \bignorm{ \left[b_{n+1}(Y_{n+1}) - b_n(Y_{n+1})\right] \Delta W_n }^2 \\
    &\lesssim K^2 \tau^3 \E[ (1+\norm{Y_{n+1}}^2) \norm{w_n}^2 ] \\
    &\lesssim K^2 \tau^3 \E[ (1 +\norm{X_n}^2 + \norm{a_n(X_n)}^2 \tau^2 + \opnorm{b_n(X_n)}^2 \norm{w_n}^2 \tau)  \norm{w_n}^2  ] \\
    &\lesssim K^4 \tau^3 \E[ (1 + \norm{X_n}^2 + \norm{w_n}^2) \norm{w_n}^2 ] \\
    &\lesssim_{K,m} (1 + \E\norm{X_n}^2) \tau^3.
\end{align*}

\emph{Bound on $\E\norm{E_n^{b,2}}^2$:}
We have:
\begin{align*}
    \E\norm{E_n^{b,2}}^2 &= \frac{\tau^2}{4}\E \bignorm{ \sum_{k=1}^{m} \partial_x b^k_n(X_n) a_n(X_n) \Delta W_n^k }^2 
    = \frac{\tau^3}{4} \sum_{k=1}^{m} \E \norm{ \partial_x b_n^k(X_n) a_n(X_n) }^2 \\
    &\leq \frac{\tau^3}{4} \sum_{k=1}^{m} \E \opnorm{ \partial_x b_n^k(X_n)}^2 \norm{a_n(X_n)}^2 
    \lesssim_{K,m} \tau^3.
\end{align*}

\emph{Bound on $\E\norm{E_n^{b,3}}^2$:}
Recall that the residual $R_n^k$ satisfies
$\norm{R_n^k} \leq K \norm{Y_{n+1} - X_n}^2$.
We have:
\begin{align*}
    \E\norm{E_n^{b,3}}^2 &= \frac{1}{4} \E\bignorm{ \sum_{k=1}^{m} R_n^k \Delta W_n^k }^2 \\
    &\lesssim \tau m \sum_{k=1}^{m} \E[ \norm{R_n^k}^2 \abs{w_n^k}^2 ] \\
    &\lesssim \tau K^2 m \sum_{k=1}^{m} \E[ \norm{Y_{n+1} - X_n}^4 \abs{w_n^k}^2 ]  \\
    &\lesssim \tau K^2 m \sum_{k=1}^{m} \E\left[ (  \norm{a_n(X_n)}\tau + \opnorm{b_n(X_n)} \norm{w_n} \sqrt{\tau} )^4   \abs{w_n^k}^2 \right] \\
    &\lesssim_{K,m} \tau^3.
\end{align*}

\emph{Bound on $\E\norm{E_n^{b,4}}^2$:}
We have:
\begin{align*}
    \E\norm{E_n^{b,4}}^2 &= \frac{\tau^2}{4} \E \bignorm{ \sum_{k_1,k_2=1}^{m} \partial_x b_n^{k_1}(X_n) b_n^{k_2}(X_n) \left( w_n^{k_1} w_n^{k_2} - \ind_{\{k_1=k_2\}} \right) }^2 
    \lesssim_{K,m} \tau^2.
\end{align*}

\paragraph{Second moment bounds on the Heun process.}

We now use \eqref{eq:heun_EM_decomposition_full}
to write for any $n \in [N]$:
\begin{align}
    \E\norm{X_n}^2 &= \E\bignorm{X_0 + \sum_{i=0}^{n-1} \bar{a}_i(X_i) \tau + \sum_{i=0}^{n-1} b_i(X_i) \Delta W_i + \sum_{i=0}^{n-1} \left( E_i^a + \sum_{\ell=1}^{4} E_i^{b,\ell} \right) }^2 \nonumber \\
    &\lesssim \E\norm{X_0}^2 + \tau^2 \E\bignorm{\sum_{i=0}^{n-1} \bar{a}_i(X_i)}^2 + \E \bignorm{\sum_{i=0}^{n-1} b_i(X_i) \Delta W_i }^2 + 
    \E \bignorm{ \sum_{i=0}^{n-1} E_i^a }^2 + \sum_{\ell=1}^{4} \E \bignorm{ \sum_{i=0}^{n-1} E_i^{b,\ell} }^2 . \label{eq:heun_EM_decomposition_unrolled}
\end{align}
We now focus on bounding these second moments,
using the fact that for $n \in [N]$, we have $\tau n \leq \tau N = \tau \floor{T/\tau} \leq T$.
First, we have
$\norm{a_i(X_i)} \leq K + \frac{K^2 m}{2}$ and hence
\begin{align*}
     \tau^2 \E\bignorm{\sum_{i=0}^{n-1} \bar{a}_i(X_i)}^2 \lesssim \tau^2 n \sum_{i=0}^{n-1} \E \norm{\bar{a}_i(X_i)}^2 \lesssim_{K,m} (n \tau)^2 \lesssim_{K,m,T} 1.
\end{align*}
Next, since $\{ b_i(X_i) \Delta W_i \}_{i}$ forms a martingale difference sequence (MDS),
\begin{align*}
    \E \bignorm{\sum_{i=0}^{n-1} b_i(X_i) \Delta W_i }^2 = \sum_{i=0}^{n-1} \E\norm{ b_i(X_i) \Delta W_i}^2  = \tau \sum_{i=0}^{n-1} \E \norm{b_i(X_i)}_F^2 \lesssim_{K,m} n\tau \lesssim_{K,m,T} 1.
\end{align*}
Next, we have the following bound using the second moment computations for the error terms:
\begin{align*}
    \E\bignorm{\sum_{i=0}^{n-1} E_i^a}^2 + \E\bignorm{\sum_{i=0}^{n-1} E_i^{b,1}}^2 &\lesssim n \sum_{i=0}^{n-1} (\E\norm{E_i^a}^2 +\E\norm{E_i^{b,1}}^2) \lesssim_{K,m} n \sum_{i=0}^{n-1} (1+\E\norm{X_i}^2)\tau^3 \\
    &\lesssim_{K,m} (n\tau)^2 \tau + n \tau^3 \sum_{i=0}^{n-1} \E \norm{X_i}^2 
    \lesssim_{K,m,T} \tau + \tau^2 \sum_{i=0}^{n-1} \E\norm{X_i}^2, \\
    \E\bignorm{\sum_{i=0}^{n-1} E_i^{b,2}}^2 + \E\bignorm{\sum_{i=0}^{n-1} E_i^{b,3}}^2 &\lesssim n \sum_{i=0}^{n-1} (\E\norm{E_i^{b,2}}^2 +\E\norm{E_i^{b,3}}^2) \lesssim_{K,m} (n\tau)^2 \tau \lesssim_{K,m,T} \tau.
\end{align*}
Furthermore, since $\{ E_i^{b,4} \}_{i}$ is an MDS,
\begin{align*}
    \E\bignorm{ \sum_{i=0}^{n-1} E_i^{b,4} }^2 = \sum_{i=0}^{n-1} \E \norm{E_i^{b,4}}^2 \lesssim_{K,m} (n\tau)\tau \lesssim_{K,m,T} \tau.
\end{align*}
Combining these bounds together in \eqref{eq:heun_EM_decomposition_unrolled}, we obtain:
\begin{align*}
    \E\norm{X_n}^2 \lesssim_{K,m,T} 1 + \tau^2 \sum_{i=0}^{n-1} \E \norm{X_i}^2.
\end{align*}
By the discrete Gronwall lemma (\Cref{prop:discrete_gronwall}), we have:
\begin{align*}
    \E \norm{X_n}^2 &\leq C_{K,m,T} \exp( n \tau^2 C'_{K,m,T} ) 
    \leq C_{K,m,T} \exp(\tau C''_{K,m,T} ) 
    \lesssim_{K,m,T} 1.
\end{align*}
Hence, we have shown that:
\begin{align*}
    \max_{n \in \{0, \dots, N\}} \E\norm{X_n}^2 \lesssim_{K,m,T} 1.
\end{align*}

\paragraph{Final result.}
Our goal now is to estimate $\E[ \Delta_n^2 ]$ for $\Delta_n := \max_{i \in [n]} \norm{\delta_i}$, where $\delta_i := \bar{X}_i - X_i$.
We start with:
\begin{align*}
    \delta_{n+1} &= \bar{X}_{n+1} - X_{n+1} \\
    &= \delta_n + \left[ \bar{a}_n(\bar{X}_n) - \bar{a}_n(X_n) \right] \tau + \left[ b_n(\bar{X}_n) - b_n(X_n) \right] \Delta W_n - E_n.
\end{align*}
Hence for $n \in [N]$,
\begin{align*}
    \norm{\delta_n}^2 &=  \bignorm{\tau \sum_{i=0}^{n-1} \left[ \bar{a}_i(\bar{X}_i) - \bar{a}_i(X_i) \right] + \sum_{i=0}^{n-1} \left[ b_i(\bar{X}_i) - b_i(X_i) \right] \Delta W_i - \sum_{i=0}^{n-1} E_i }^2 \\
    &\lesssim \tau^2 \bignorm{ \sum_{i=0}^{n-1}\left[ \bar{a}_i(\bar{X}_i) - \bar{a}_i(X_i) \right]  }^2 + \bignorm{ \sum_{i=0}^{n-1} \left[ b_i(\bar{X}_i) - b_i(X_i) \right] \Delta W_i  }^2 + \bignorm{ \sum_{i=0}^{n-1} E_i }^2 \\
    &\lesssim_{K} \tau^2 n \sum_{i=0}^{n-1} \norm{\delta_i}^2 + \bignorm{ \sum_{i=0}^{n-1} \left[ b_i(\bar{X}_i) - b_i(X_i) \right] \Delta W_i }^2 + n \sum_{i=0}^{n-1} \left( \norm{E_i^a}^2 + \sum_{\ell=1}^{3} \norm{E_i^{b,\ell}}^2 \right) + \bignorm{\sum_{i=0}^{n-1} E_i^{b,4}}^2.
\end{align*}
Therefore,
\begin{align}
    \Delta_n^2 = \max_{k \in [n]} \norm{\delta_k}^2 &\lesssim_K \tau^2 n \sum_{i=0}^{n-1} \norm{\delta_i}^2 + n \sum_{i=0}^{n-1} \left( \norm{E_i^a}^2 + \sum_{\ell=1}^{3} \norm{E_i^{b,\ell}}^2 \right) \nonumber \\
    &\qquad + \max_{k \in [n]} \bignorm{ \sum_{i=0}^{k-1} \left[ b_i(\bar{X}_i) - b_i(X_i) \right] \Delta W_i }^2 + \max_{k \in [n]} \bignorm{\sum_{i=0}^{k-1} E_i^{b,4}}^2 \nonumber \\
    &\lesssim_K \tau^2 n \sum_{i=0}^{n-1} \Delta_i^2 + n \sum_{i=0}^{n-1} \left( \norm{E_i^a}^2 + \sum_{\ell=1}^{3} \norm{E_i^{b,\ell}}^2 \right) \nonumber \\
    &\qquad + \max_{k \in [n]} \bignorm{ \sum_{i=0}^{k-1} \left[ b_i(\bar{X}_i) - b_i(X_i) \right] \Delta W_i }^2 + \max_{k \in [n]} \bignorm{\sum_{i=0}^{k-1} E_i^{b,4}}^2. \label{eq:maximal_delta_unroll}
\end{align}
Now, using nearly identical arguments as in the second moment calculation of the error terms, in addition to the uniform bound on $\E\norm{X_n}^2$ for $n \in \{0, \dots, N\}$, we have:
\begin{align*}
    \E\left[n \sum_{i=0}^{n-1} \left( \norm{E_i^a}^2 + \sum_{\ell=1}^{3} \norm{E_i^{b,\ell}}^2 \right) \right] \lesssim_{K,m,T} \tau.
\end{align*}
On the other hand, 
since both $\{ [ b_i(\bar{X}_i) - b_i(X_i) ] \Delta W_i \}_i$
and $\{ E_i^{b,4} \}_i$ are both martingale difference sequences, 
using Doob's maximal inequality (\Cref{prop:doob_maximal_inequality}),
\begin{align*}
    &\E \max_{k \in [n]} \bignorm{ \sum_{i=0}^{k-1} \left[ b_i(\bar{X}_i) - b_i(X_i) \right] \Delta W_i }^2 + \E \max_{k \in [n]} \bignorm{\sum_{i=0}^{k-1} E_i^{b,4}}^2 \\
    &\lesssim \E\bignorm{ \sum_{i=0}^{n-1} \left[ b_i(\bar{X}_i) - b_i(X_i) \right] \Delta W_i }^2 + \E\bignorm{\sum_{i=0}^{n-1} E_i^{b,4}}^2 \\
    &= \sum_{i=0}^{n-1} \E\bignorm{ \left[ b_i(\bar{X}_i) - b_i(X_i) \right] \Delta W_i }^2 + \sum_{i=0}^{n-1} \E \norm{E_i^{b,4}}^2 \\
    &\lesssim_{K,m,T}  \tau \sum_{i=0}^{n-1} \E\norm{\delta_i}^2 + \tau 
    \lesssim_{K,m,T} \tau \sum_{i=0}^{n-1} \E[ \Delta_i^2 ] + \tau.
\end{align*}
Hence, taking expectation in \eqref{eq:maximal_delta_unroll} and 
combining the previous second moment estimates, we have:
\begin{align*}
    \E[ \Delta_n^2 ] \lesssim_{K,m,T} \tau \sum_{i=0}^{n-1} \E[ \Delta_i^2 ] + \tau.
\end{align*}
From the discrete Gronwall inequality (\Cref{prop:discrete_gronwall}),
\begin{align*}
    \E[ \Delta_n^2 ] \leq C_{K,m,T} \tau \exp( n \tau C'_{K,m,T} ) \lesssim_{K,m,T} \tau,
\end{align*}
which completes the proof.
\end{proof}

We can now complete the proof of \Cref{corollary:heun_strong_convergence}.
\begin{proof}[Proof of \Cref{corollary:heun_strong_convergence}]
We have:
\begin{align*}
    X_{t_n} - \hat{X}_n &\stackrel{(a)}{=} \bar{X}_{t_n} - \hat{X}_n = (\bar{X}_{t_n} - \hat{\bar{X}}_n) + (\hat{\bar{X}}_n - \hat{X}_n),
\end{align*}
where (a) holds since the
Stratonovich SDE $(X_t)$ \eqref{eq:SDE_a_b_pair} and
the It{\^{o}} SDE $(\bar{X}_t)$ \eqref{eq:ito_SDE_a_b_pair}
are identical for a.e.\ $(\omega, t)$.
Hence we have:
\begin{align*}
    \E\max_{n \in \{0, \dots, N\}} \norm{ X_{t_n} - \hat{X}_n }^2 \lesssim \E\max_{n \in \{0, \dots, N\}} \norm{ \bar{X}_{t_n} - \hat{\bar{X}}_{n} }^2 + \E\max_{n \in \{0, \dots N \}} \norm{ \hat{\bar{X}}_n - \hat{X}_n }^2.
\end{align*}
Next, since the pair $(a, b)$ is Heun-regular by assumption,
then the pair $(\bar{a}, b)$ is EM-regular by 
\Cref{prop:heun_implies_EM_regular}.
Hence by \Cref{thm:EM_strong_order_one_half},
we have the EM discretization $\{ \hat{\bar{X}}_n \}_n$ is order $1/2$ strongly convergent to
the SDE $(\bar{X}_t)_t$, i.e., 
$\E\max_{n \in \{0, \dots, N\}} \norm{ \bar{X}_{t_n} - \hat{\bar{X}}_{n} }^2 \leq C \tau$.
Furthermore, by 
\Cref{lemma:heun_strong_convergence_discrete}
we have 
$\E\max_{n \in \{0, \dots N \}} \norm{ \hat{\bar{X}}_n - \hat{X}_n }^2 \leq C' \tau$ as well. Note in both cases, $C,C'$ do not depend on $\tau$, and hence the proof is complete.
\end{proof}

\section{Experimental Details}
\label{sec:appendix:experiments}

We use each algorithm to train an $8$-layer neural network with $64$ neurons per layer and \texttt{swish} activation~\cite{ramachandran2017searching} to model the solution $u(x, t)$ of a PDE.
The boundary condition is enforced by adding a boundary condition penalty 
$\E_{x \sim \mu'}[(u_\theta(x, T) - \phi(x)^2] + \E[ \norm{ \nabla u_\theta(x, T) - \nabla \phi(x)}^2 ]$ involving both the zero-th and first-order values of $\phi$~\cite{raissi2024forward},
where the distribution $\mu'$ is taken over
each method's approximation of
the distribution of $X_T$.
Additionally, following state-of-the-art PINN architectures practices~\cite{wang2023expert},
Fourier embeddings~\cite{tancik2020fourierfeatures} with a $256$ embedding dimension and skip connections on odd layers are used. 
We use a trajectory batch size of $64$, 
translating to $64$ realizations of the underlying Brownian motions. Additionally, we utilize a sub-sampling batch size of 1024 for the batched algorithm runs.
We use the Adam optimizer with a multi-step learning rate schedule~\cite{raissi2024forward} of $10^{-3}, 10^{-4}$, and $10^{-5}$ at 50k, 75k, and 100k iterations, respectively.
All models are trained on a single NVIDIA A100 GPU node in our internal cluster, using the \texttt{jax} library~\cite{jax2018github}.

\subsection{PDE Test Cases}
\label{sec:appendix:experiments:PDEs}
\paragraph{Hamilton-Jacobi-Bellman (HJB) Equation.}
First, we consider the following Hamilton-Jacobi-Bellman (HJB) equation studied in \cite{raissi2024forward}:
\begin{align*}
    \partial_tu(x,t)=-\mathrm{Tr}[\nabla^2u(x,t)]+\|\nabla u(x,t)\|^2, \quad x \in \R^d, \,\, t \in [0, T].
\end{align*}
For the terminal condition  $u(x,T)=\phi(x)=\ln{(.5(1+\|x\|^2))}$, the analytical solution is given as
\begin{align}
    u(x,t)=-\ln{\left(\E\left[ \exp{ \left(-g(x+\sqrt{2} B_{T-t})\right) }\right]\right)}.
    \label{eq:hjbsolution}
\end{align}
The HJB PDE is related to the forward-backward stochastic differential equation of the form:
\begin{align*}
    \rmd X_t&=\sigma \rmd B_t,\quad t \in [0,T], \\
    \rmd Y_t&=\|Z_t\|^2 \rmd t+\sigma Z_t^\T \rmd B_t,\quad t\in [0,T),
\end{align*}
where $T=1$, $\sigma=\sqrt{2}$, $X_0=0$, and $Y_T = \phi(X_T)$. Additionally, the Stratonovich SDE is given as:
\begin{align*}
        \rmd X_t&=\sigma \circ\rmd B_t,\quad t \in [0, T], \\
    \rmd Y_t&=\left[\|Z_t\|^2-\frac{1}{2} \mathrm{Tr}\left[ \sigma^2\nabla^2u(X_t,t)\right] \right]\rmd t+\sigma Z_t^\T \circ \rmd B_t,
\end{align*}
In our experiments, in order to compute the analytical solution \eqref{eq:hjbsolution}, we approximate it using $10^5$ Monte-Carlo samples.

\paragraph{Black-Scholes-Barenblatt (BSB) Equation.}
Next, we consider the 100D Black-Scholes-Barenblatt (BSB) equation from \cite{raissi2024forward} of the form
\begin{align*}
    \partial_tu(x,t)=-\frac{1}{2}\mathrm{Tr}[\sigma^2 \diag(x^2) \nabla^2u(x,t)]+r\left(u(x,t)-\nabla u(x,t)^\T x\right), \quad x \in \R^{d}, \,\, t \in [0, T],
\end{align*}
where $x^2$ is understood to be coordinate-wise, and
$\diag(v)$ is a diagonal matrix with $\diag(v)_i = v_i$.
Given the terminal condition $u(x,T)=\phi(x)=\|x\|^2$,
the explicit solution to this PDE is 
\begin{align*}
    u(x,t)=\mathrm{exp}\left((r+\sigma^2)(T-t)\right)\phi(x).
\end{align*}
The BSB PDE is related to the following FBSDE
\begin{align*}
    \rmd X_t&=\sigma \diag(X_t) \rmd B_t,\quad t \in [0,T], \\
    \rmd Y_t&=r\left(Y_t-Z_t^\T X_t\right)\rmd t+\sigma Z_t^\T \diag(X_t) \rmd B_t,\quad t\in [0,T),
\end{align*}
where $T=1$, $\sigma=.4$, $r=.05$, $X_0=(1,.5,1.,5,\dots,1,.5)$, and $Y_T=\phi(X_T)$. The equivalent Stratonovich SDE is given as:
\begin{align*}
    \rmd X_t&=\frac{\sigma^2}{2}X_t \rmd t + \sigma \diag(X_t) \circ\rmd B_t, \\
    \rmd Y_t&=\left[r\left(Y_t-Z_t^\T X_t\right)-\frac{\sigma^2}{2}\left( 
 Z_t^\T X_t+\mathrm{Tr}\left[\mathrm{diag}(X_t^2)\nabla^2_x u(X_t,t)\right]\right)\right]\rmd t+\sigma Z_t^\T \diag(X_t) \circ \rmd B_t,
\end{align*}

\paragraph{Fully-Coupled FBSDE.}
Finally, we consider a FBSDE with \emph{coupled} forward and backwarwds dynamics adapted from Bender \& Zhang (BZ) \cite{bender2008coupledfbsdes}:
\begin{align*}
    \rmd X_t&=\sigma Y_t \rmd B_t,\quad t \in [0,T], \\
    \rmd Y_t&=\left[-rY_t +\frac{1}{2}e^{-3r(T-t)}\sigma^2\left(D\sum_{j=0}^d\sin(X_{j,t})\right)^3 \right] \rmd t+Z_t^\T \rmd B_t,\quad t\in [0,T),
\end{align*}
where $X_{j,t}$ denotes the $j$-th coordinate of $X_t \in \R^d$.
Due to the dependence of the forward process $(X_t)$ on $(Y_t)$, this set of coupled FBSDE does not fit into the mathematical formulation set forth in
\eqref{eq:forward_SDE} and \eqref{eq:backward_SDE}.
Nevertheless, we can still apply the BSDE methods
described at the beginning of \Cref{sec:experiments}
by initializing $Y_0 = u_\theta(x, 0)$
and jointly integrating $(X_t, Y_t)$. 
We set $T=1$, $r=.1$, $\sigma=.3$, $D=.1$, $X_0=(\pi/2,\pi/2,\dots,\pi/2)$, and $Y_T(X_T)=\phi(X_T)=D\sum_{j=1}^d\sin(X_{j,T})$. The above FBSDE is induced from the following PDE
\begin{align*}
    \partial_tu(x,t)=-\frac{1}{2}\sigma^2u(x,t)^2\nabla^2u(x,t)+ru(x,t)-\frac{1}{2}e^{-3r(T-t)}\sigma^2\left(D\sum_{j=0}^d\sin(x_j)\right)^2
\end{align*}
with the analytical solution
$u(x,t)=e^{-r(T-t)}D\sum_{j=0}^d\sin(x_j)$.
Additionally, the equivalent Stratonovich SDE is given as:
\begin{align*}
    \rmd X_t&=\frac{\sigma^2}{2}Z_tY_t + \sigma Y_t \circ \rmd B_t, \\
    \rmd Y_t&=\left[ -rY_t +\frac{1}{2}e^{-3r(T-t)}\sigma^2\left(D\sum_{j=0}^d\sin(X_{j,t})\right)^3 - \frac{\sigma^2}{2}\left(Z_t^2Y_t+\mathrm{Tr}\left[ Y_t^2\nabla_x^2u(X_t,t)\right]\right) \right] \rmd t+Z_t^\T \circ \rmd B_t,
\end{align*}

\begin{table}[!ht]
    \centering
    \begin{tabular}{c|c|c|c|c}
         &  \multicolumn{2}{c|}{100D HJB} & \multicolumn{2}{c}{100D BSB}\\
         \textbf{Method} & \textbf{\texttt{float32}} & \textbf{\texttt{float64}} & \textbf{\texttt{float32}} & \textbf{\texttt{float64}} \\ \hline\hline
         PINNs & $0.1281 \pm .0136$ & $0.1281 \pm .0171$ & $2.9648\pm.8652$ & $1.5066\pm.2349$ \\
         FS-PINNs & $0.0838 \pm .0170$ & $0.0702 \pm .0074$ & $0.0602 \pm .0150$ & $0.0497 \pm .0031$ \\
         EM-BSDE & $0.3776 \pm .0365$ & $0.3820 \pm .0219$ & $0.2451 \pm .0160$ & $0.3735 \pm .0470$ \\ 
         EM-BSDE (NR) & $0.4459 \pm .0410$ & $0.4676 \pm .0153$ & $0.1648 \pm .0143$ & $0.1855 \pm .0078$ \\
         Heun-BSDE & $0.0675 \pm .0053$ & $0.0529 \pm .0029$ & $0.4587 \pm .0261$ & $0.0535 \pm .0113$
    \end{tabular}
    \captionsetup{skip=5pt}
    \caption{\texttt{float32} vs \texttt{float64} Performance in 100D HJB/BSB}
    \label{tab:f32vsf64}
\end{table}

\subsection{Sensitivity to Floating Point Precision}
\label{sec:appendix:experiments:floating_point}

In BSDE-based losses, floating point errors can accumulate through out integration of the SDEs, leading to poor performance on the trained model. As seen in \Cref{tab:f32vsf64}, the floating point error is especially apparent in the Heun loss on the 100D BSB case where the performance of the model is improved by a factor of 10 between \texttt{float32} and \texttt{float64}. In addition, performance improvements were observed for the PINNs and FS-PINNs models as well. It is also noted that the EM-BSDE models performed slightly worse at a \texttt{float64}, which may be attributed to the bias term present in its loss. 
Overall, floating point sensitivity is more apparent in the BSB problem than the HJB problem. We attribute this to the non-trivial forward trajectory in the BSB problem.
We leave more numerically stable implementations of the Heun solver in \texttt{float32}, such as 
PDE non-dimensionalization~\cite{wang2023expert} and
the use of the reversible Heun solver from \cite{kidger2021neuralsdes}, to future work.

\subsection{Dimensionality Study}
Additionally, we demonstrate scalability of the algorithms by re-running the HJB problem at various dimensions and plotting the RL2 error for each method. \Cref{fig:dimresults} shows EM-BSDE underperforming both FS--PINNs and Heun-BSDE across all dimensions tested. Additionally, trajectory-based methods scale more effectively to high-dimensional problems with PINNs-based methods showing a slight advantage in lower dimensions. 
\begin{figure}[!ht]
    \begin{minipage}[t]{.48\linewidth}
    \centering
    \includegraphics[width=0.98\linewidth]{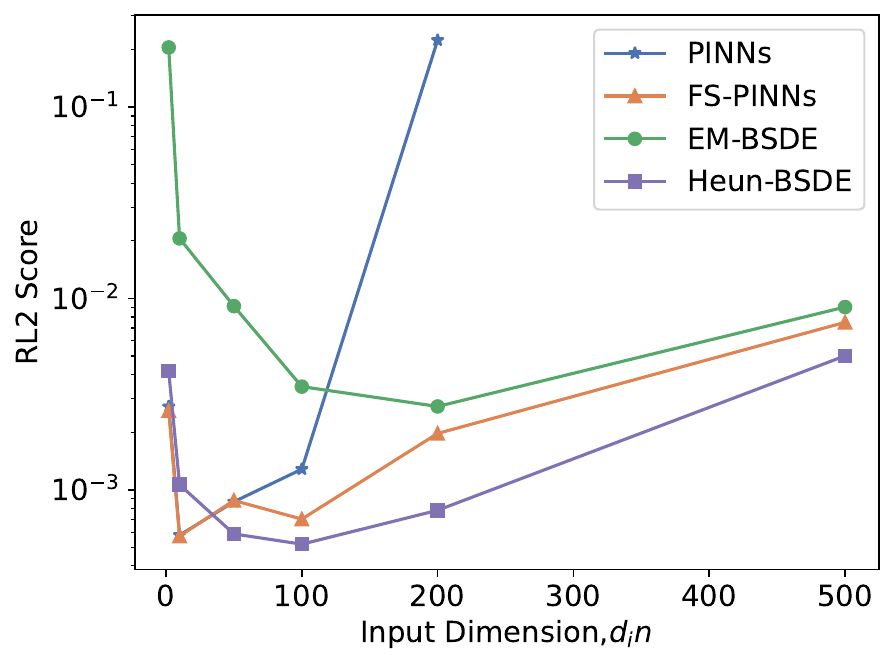}
    \caption{A plot of RL2 performance for the HJB problem at various dimensions $d_{in}=\{2,10,50,100,200,500\}$}
    \label{fig:dimresults}
    \end{minipage}
    \hfill
    \begin{minipage}[t]{.48\linewidth}
    \centering
    \includegraphics[width=\linewidth]{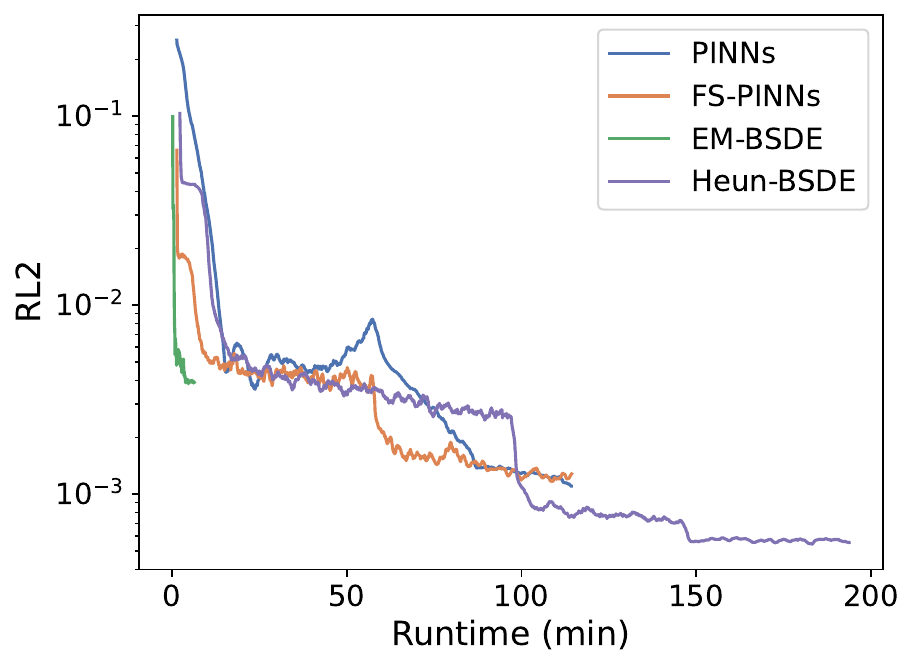}
    \caption{A plot of RL2 performance vs runtime at fixed iteration steps}
    \end{minipage}
\end{figure}

\subsection{Improved Algorithm Schemes}
\label{sec:appendix:experiments:algorithms}
We describe the algorithmic differences between the full roll-out algorithm and the batched sub-sampling variation for the general BSDE loss; a similar algorithm is used FS-PINNs, replacing self-regularization loss with the PINNs loss.

\begin{algorithm}[hbt]
    \caption{Full BSDE Loss Algorithm}
    \label{alg:bsdeoriginal}
    \begin{algorithmic}[1]
        \small
        \Input Neural network $\hat u_\theta(x,t)$, parameters $\theta$, %
        terminal function $\phi$, time step $\Delta t$, trajectory length $N$.
        \Output Self-consistency loss $\mathcal L_{\mathrm{sr}}$ and terminal loss $\mathcal L_{\phi}$
        \State Sample initial state: $(x[0],t[0])=(x_0,0)$, with $x_0 \sim \mu$
        \State Evaluate network at initial state: $(u,u_x)=(\hat u_\theta(x[0],t[0]),\nabla_x \hat u_\theta(x[0],t[0]))$
        \State Initialize self-consistency losses: $\ell_{\mathrm{step}}[0:N-1]\gets 0$
        \For{$i=0,\dots,N-1$}
            \State Sample Brownian noise: $\xi\sim\mathsf{N}(0,I_d)$
            \State \verb|/* Forward SDE rollout */|
            \State Propagate forward state: $x[i+1]=x[i]+\Delta x$ %
            \If{no resetting}
            \State \verb|/* Use either EM or Heun integration (NR) */|
            \State Propagate backward state:
            $y[i+1]=y[i]+\Delta y$ %
            \Else
            \State \verb|/* Use either EM or Heun integration */|
            \State Propagate backward state:
            $y[i+1]=u+\Delta y$ %
            \EndIf
            \State Propagate time: $t[i+1]=t[i]+\Delta t$
            \State Evaluate network at new state: $(u,u_x)=(\hat u(x[i+1],t[i+1]),\nabla_x\hat u(x[i+1],t[i+1]))$
    
            \State Record local residual loss: $\ell_{\mathrm{step}}[i]=(u-y[i+1])^2$
        \EndFor
        \State Compute self-consistency loss: $\mathcal{L}_{\mathrm{sr}}=\sum_{i=0}^{N-1}\ell_{\mathrm{step}}[i]$
        \State Compute terminal loss: $\mathcal{L}_\phi=(u-\phi)^2+\norm{u_x-\nabla_x \phi}^2$
        \State \Return $(\mathcal L_{\mathrm{sr}},\mathcal L_{\phi})$
    \end{algorithmic}
\end{algorithm}

\begin{algorithm}[hbt]
    \caption{Batched, Sub-sampling BSDE Loss Algorithm (Full description of \Cref{alg:bsdebatchedsimple})}
    \label{alg:bsdebatched}
    \begin{algorithmic}[1]
        \small
        \Input Neural etwork $\hat u_\theta(x,t)$, parameters $\theta$, %
        terminal function $\phi$, time step $\Delta t$, trajectory length $N$, 
        evaluation batch $B$.
        \Output Self-consistency loss $\mathcal L_{\mathrm{sr}}$ and terminal loss $\mathcal L_{\phi}$
        \State Sample initial state: $(x[0],t[0])=(x_0,0)$, with $x_0 \sim \mu$
        \State Sample Brownian noise: $\xi[0:N-1]\sim \mathsf{N}(0,I_d)$
        \State Evaluate network at initial state: $(u,u_x)=(\hat u_\theta(x[0],t[0]),\nabla_x\hat u_\theta(x[0],t[0]))$
        \State \verb|/* Forward SDE rollout */|
        \For{$i=0,\dots,N-1$} %
            \State \verb|/* Use either EM or Heun integration */|
            \State Propagate forward state: $x[i+1]=x[i]+\Delta x$ %
            \State Propagate time: $t[i+1]=t[i]+\Delta t$
            \If{coupled} 
            \State Evaluate network at new state: $(u,u_x)=(\hat u_\theta(x[i+1],t[i+1]),\nabla_x\hat u_\theta(x[i+1],t[i+1]))$
            \EndIf
        \EndFor
            \State Stop gradient: $x[0:N]=\mathrm{SG}(x[0:N])$
            \State Separate states: $(x_i,x_{i+1},t_i,t_{i+1})=(x[0:N-1],x[1:N],t[0:N-1],t[1:N])$  %
            \State \verb|/* Same permutations */|
            \State Random sub-sampling: $(x_i,x_{i+1},t_i,t_{i+1})=\text{perm}(x_i,x_{i+1},t_i,t_{i+1})[0:B-1]$ %
            \State Evaluate network at $i$ points: $(u_i,u_{i_x})=(\hat u_\theta(x_i,t_i),\nabla_x\hat u_\theta(x_i,t_i))$
            \State \verb|/* Use either EM or Heun Integration */|
            \State Compute backward state at batched point: $y_{i+1}=u_i+\Delta y$ %
            \State Evaluate network at $i+1$ points: $u_{i+1}=\hat u_\theta(x_{i+1},t_{i+1})$
            \State \verb|/* Use PINNs loss instead for FS-PINNs */|
            \State Compute self-consistency loss: $\mathcal{L}_{\mathrm{sr}}=\frac NB \sum_{i=0}^{B-1}\left(u_{i+1}-y_{i+1}\right)^2$ 
            \State Evaluate network at $T$: $(u,u_x)=(\hat u_\theta(x[N],t[N]),\nabla_x\hat u_\theta(x[N],t[N]))$
            \State Compute terminal loss: $\mathcal{L}_{\phi}=(u+\phi)^2+\norm{u_x-\nabla_x\phi}^2$
            \State \Return $(\mathcal L_{\mathrm{sr}},\mathcal L_{\phi})$
    \end{algorithmic}
\end{algorithm}

As shown in \Cref{alg:bsdeoriginal,alg:bsdebatched}, the new variation of the loss separates the forward and backward propagation which enables random sub-sampling in the loss evaluation. At full sampling ($B=N$), the batched algorithm recovers loss and gradient values consistent with the original algorithm with proper scaling. Additionally, the stop gradient on the forward SDE has negligible effects on model performance but can improve convergence as it fixes optimization to only the backward SDE or PDE.

Note that \Cref{alg:bsdeoriginal,alg:bsdebatched} are simplified for readability and excludes some algorithmic details such as trajectory batching and loss weighting. 

\subsection{Behavior Policy Rollouts for HJB Optimal Control}
\label{sec:appendix:behavior_policy_HJB}

Suppose our control system is a deterministic control-affine system:
$\dot{x} = f(x) + g(x) u$.
For a positive definite $R$, the HJB equation for stagewise cost $c(x) + \frac{1}{2}\norm{u}^2_{R}$
and terminal cost $c_T$ 
is:
\begin{align*}
    \partial_t V + \ip{\nabla_x V}{f} + c - \frac{1}{2} \norm{g^\T \nabla V}_{R^{-1}}^2 = 0, \quad V(x, T) = c_T(x),
\end{align*}
and the optimal control induced by $V$ is $\pi_V(x, t) := - R^{-1} g^\T(x) \nabla V(x, t)$.

Now, suppose we have any \emph{rollout} policy $\pi(x, t)$,
and we consider It{\^{o}} stochastic rollouts of the form:
\begin{align*}
    \rmd X^\pi_t = [ f(X^\pi_t) + g(X^\pi_t) \pi(X^\pi_t, t) ] \rmd t + \sigma \rmd B_t.
\end{align*}
Now, the \emph{optimal} value function $V^\star(x, t)$ satisfies the following SDE:
\begin{align*}
    \rmd V^\star(X^\pi_t, t) &= \left[ \partial_t V^\star(X^\pi_t, t) + \ip{ f(X^\pi_t) + g(X^\pi_t) \pi(X^\pi_t, t)}{\nabla V^\star(X^\pi, t)} + \frac{\sigma^2}{2} \Tr( \nabla^2 V^\star(X^\pi_t, t) ) \right] \rmd t \\
    &\qquad+ \sigma \ip{\nabla V^\star(X_t^\pi, t)}{\rmd B_t} \\
    &= \left[ \frac{1}{2} \norm{g^\T \nabla V^\star }^2_{R^{-1}} - c + \ip{g \pi}{\nabla V^\star} + \frac{\sigma^2}{2} \Tr( \nabla^2 V^\star)  \right] \rmd t + \sigma \ip{\nabla V^\star}{\rmd B_t},
\end{align*}
noting that the last expression is evaluated at $(X^\pi_t, t)$.
Hence, the forward/backward It{\^{o}} SDEs for a given value function $V$
and behavior policy $\pi$ are:
\begin{subequations}
\label{eq:ito_FBSDE_behavior_HJB}
\begin{align}
    \rmd X^\pi_t &= [ f(X^\pi_t) + g(X^\pi_t) \pi(X^\pi_t, t) ] \rmd t + \sigma \rmd B_t, \\
    \rmd Y^V_t &= \left[ \frac{1}{2} \norm{g^\T \nabla V }^2_{R^{-1}} - c + \ip{g \pi}{\nabla V} + \frac{\sigma^2}{2} \Tr( \nabla^2 V)  \right] \rmd t + \sigma \ip{\nabla V}{\rmd B_t}, \label{eq:ito_BSDE_behavior_HJB}
\end{align}
\end{subequations}
noting again that the expressions in \eqref{eq:ito_BSDE_behavior_HJB} are evaluated at
$(X_t^\pi, t)$.

Similarly, we can define the forward/backward Stratonovich
SDEs as:
\begin{subequations}
\label{eq:strat_FBSDE_behavior_HJB}
\begin{align}
    \rmd X^\pi_t &= [ f(X^\pi_t) + g(X^\pi_t) \pi(X^\pi_t, t) ] \rmd t + \sigma \rmd B_t, \\
    \rmd Y^V_t &= \left[ \frac{1}{2} \norm{g^\T \nabla V }^2_{R^{-1}} - c + \ip{g \pi}{\nabla V}  \right] \rmd t + \sigma \nabla V^\T \circ \rmd B_t. \label{eq:strat_BSDE_behavior_HJB}
\end{align}
\end{subequations}
Note that the It{\^{o}} BSDE \eqref{eq:ito_BSDE_behavior_HJB}
requires explicit Hessian computation $\nabla^2 V$ while the Stratonovich BSDE \eqref{eq:strat_BSDE_behavior_HJB} does not. This holds for all first-order PDEs, such as in deterministic HJB problems.

These forward/backward SDEs can be used in conjunction
with the induced policy $\pi_V$ from the current value function 
$V$. Some care must be taken though when setting up the BSDE losses.
In particular, since both the forward SDE trajectory $(X_t^{\pi_V})_t$
and the $\pi_V(X_t^{\pi_V}, t)$ terms which appear the backward SDEs
depend implicitly on $V$, stop-gradient operators should be placed so that gradients are not back-propagated through
these values, which can destabilize training.

\subsection{Pendulum Swing Up Experiment}
\label{sec:appendix:experiments:pendulum}
In addition to the results above, we include a simple pendulum swing-up optimal control experiment inspired by~\cite{han2024nonsmooth}. Given the pendulum equations of motion,
\begin{align*}
    x=\mkmat{\theta\\ \dot\theta},\quad f(x,u)=\dot x=\mkmat{\dot\theta \\ -\frac{1}{ml^2}\left(b\dot\theta-mgl\sin\theta-u\right)},
\end{align*}
we define a optimal control objective:
\begin{align*}
    J^\star(x_0)=\min_{u(t)}\int_0^Tc(x(t),u(t))\rmd t+\Phi(x(T)),
\end{align*}
where:
\begin{align*}
    c(x,u)=\Phi(x) + ru^2, \quad
    \Phi(x)=q_1\sin^2\theta+q_1(\cos\theta-1)^2+q_2\dot\theta^2,
\end{align*}
with $q_1,q_2,r>0$.
Observe that this problem setup exactly fits the
setup in \Cref{sec:appendix:behavior_policy_HJB}, and hence
both the forward/backward It{\^{o}} and Stratonovich SDEs
in \eqref{eq:ito_FBSDE_behavior_HJB} 
and \eqref{eq:strat_FBSDE_behavior_HJB} directly apply.

\subsubsection{Pendulum Results}

\begin{table}[hbt]
    \centering
    \begin{tabular}{c|c|c|c|c}
        Metric & PINNs & FS-PINNs & EM-BSDE & Heun-BSDE \\\hline \hline
        Cost & 53.17 & 46.59 & 46.42 & 46.43 \\
        PDE Error & 2.77 & 3.38 & 78.94 & 18.6
    \end{tabular}
    \caption{Accumulated cost and average PDE error for the pendulum swing-up problem.}
    \label{tab:pendulum}
\end{table}

The results of the pendulum swing-up case are outlined in~\Cref{tab:pendulum}. 
We use the specific constants $m=1,b=0.1,l=1,g=9.8,q_1=10,q_2=1,r=1$ 
in our experiment.
It is observed that while the accumulated cost between the three trajectory-based methods remain similar, the lower PDE error on FS-PINNs and Heun-BSDE signify better learned solutions. Furthermore, in~\Cref{fig:pendulumerror}, we observe that Heun-BSDE generally has the lowest PDE error with high errors only at the discontinuities.

\begin{figure}[hbt]
    \centering
    \includegraphics[width=.75\linewidth]{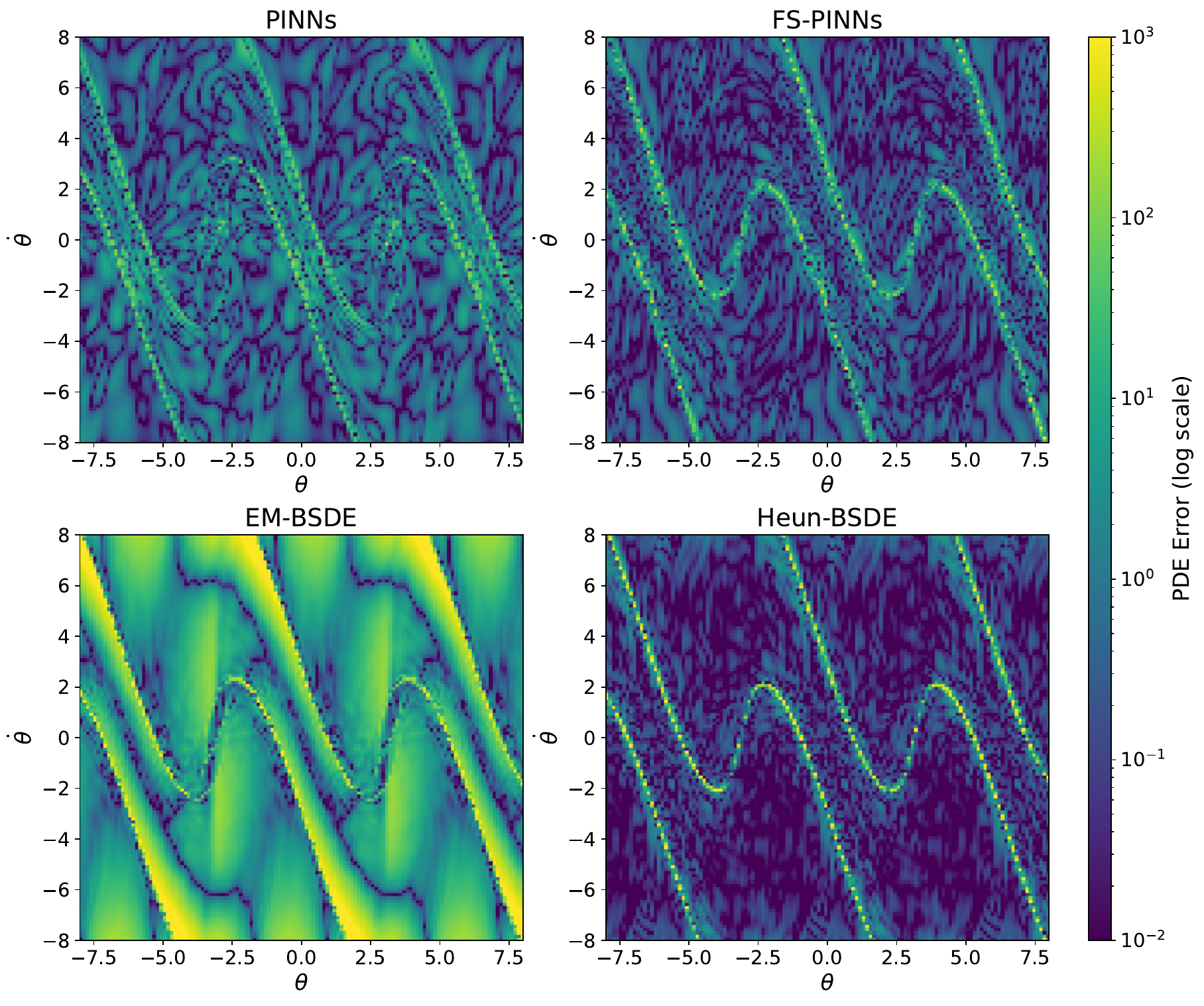}
    \caption{PDE error at $t=0$ for the pendulum swing up case.}
    \label{fig:pendulumerror}
\end{figure}

\end{document}